\documentclass{article}

\PassOptionsToPackage{numbers, compress}{natbib}

\usepackage[final]{neurips_2023}

\usepackage[utf8]{inputenc} %
\usepackage[T1]{fontenc}    %
\usepackage{hyperref}       %
\usepackage{url}            %
\usepackage{booktabs}       %
\usepackage{amsfonts}       %
\usepackage{nicefrac}       %
\usepackage{microtype}      %
\usepackage[table]{xcolor}         %

\usepackage{amsmath}
\usepackage{graphicx}
\usepackage{amsthm}
\usepackage{algorithmic}
\usepackage{algorithm}
\usepackage{wrapfig}

\hypersetup{
    colorlinks=true,
    linkcolor=red,
    citecolor=cyan,
    filecolor=magenta,      
    urlcolor=cyan,
    }

\newtheorem{theorem}{Theorem}
\newtheorem{assumption}{Assumption}
\newtheorem{remark}{Remark}
\newtheorem{lemma}{Lemma}
\newtheorem{proposition}{Proposition}

\title{PRIOR: Personalized Prior for Reactivating the Information Overlooked in Federated Learning.}

\author{%
  Mingjia Shi$^{1}$\thanks{3101ihs@gmail.com} \qquad
  Yuhao Zhou$^{1}$ \qquad
  Kai Wang$^{2}$ \qquad
  Huaizheng Zhang \qquad \\
  \textbf{
  Shudong Huang$^{1}$ \qquad
  Qing Ye$^{1\dag}$ \qquad
  Jiangcheng Lv$^{1}$\thanks{Equal Corresponding Authors \{yeqing, lvjiancheng\}@scu.edu.cn}
  }\\
  $^{1}$Sichuan University \qquad
  $^{2}$National University of Singapore
}

\begin{document}

\maketitle

\begin{abstract}

Classical federated learning (FL) enables training machine learning models without sharing data for privacy preservation, but heterogeneous data characteristic degrades the performance of the localized model. %
Personalized FL (PFL) addresses this by synthesizing personalized models from a global model via training on local data. %
Such a global model may overlook the specific information that the clients have been sampled. %
In this paper, we propose a novel scheme to inject personalized prior knowledge into the global model in each client, which attempts to mitigate the introduced incomplete information problem in PFL. %
At the heart of our proposed approach is a framework, the \textit{PFL with Bregman Divergence} (pFedBreD), decoupling the personalized prior from the local objective function regularized by Bregman divergence for greater adaptability in personalized scenarios.
We also relax the mirror descent (RMD) to extract the prior explicitly to provide optional strategies.
Additionally, our pFedBreD is backed up by a convergence analysis. %
Sufficient experiments demonstrate that our method reaches the \textit{state-of-the-art} performances on 5 datasets and outperforms other methods by up to 3.5\% across 8 benchmarks. 
Extensive analyses verify the robustness and necessity of proposed designs. %
\url{https://github.com/BDeMo/pFedBreD_public}

\end{abstract}

\section{Introduction}

Federated learning (FL)~\cite{mcmahan2017communication} has achieved significant success in many fields~\cite{10.1145/3298981, 9599369, yang_federated_2020, ZHANG2021106679, xu_federated_2021, 9415623, 9475501, 8844592, s20216230, long_federated_2020}, which include recommendation systems utilized by e-commerce platforms~\cite{yang_federated_2020}, prophylactic maintenance for industrial machinery~\cite{ZHANG2021106679}, disease prognosis employed in healthcare~\cite{xu_federated_2021}.
Data heterogeneity is a fundamental characteristic of FL, leading to challenges such as inconsistent training and testing data (data drift)~\cite{kairouz2021advances}.
An efficient solution to these challenges is to fine-tune the global model locally for adaptation on local data~\cite{arivazhagan2019federated, Zhang_2022_CVPR, jiang2019improving}.
This solution is straightforward and pioneering, but presents a fundamental limitation when dealing with highly heterogeneous data. For examples, heterogeneous data drift may introduce substantial noise~\cite{hsu_measuring_2019} and the resulted model may not generalize well to new sample~\cite{fallah2020personalized, cheng2022federated}.
Thus, heterogeneous data in FL is still challenging~\cite{tan2022towards}.

Recently, personalized FL (PFL) is proposed to mitigate the aforementioned negative impact of heterogeneous data~\cite{tan2022towards}. 
To improve the straightforward solution mentioned above, Per-FedAvg~\cite{fallah2020personalized} is introduced to train a global model that is easier to fine-tune. Another paper on the similar topic, FedProx~\cite{li_federated_2020}, aims to resolve the issue of personalized models drifting too far from the global model during training with a dynamic regularizer in the objective during local training. This issue could occur especially in post-training fine-tuning methods without regularization (\textit{e.g.}, Per-FedAvg~\cite{fallah2020personalized}).
Moreover, pFedMe~\cite{t2020personalized}, another regularization method modeling local problems using Moreau envelopes, replaces FedProx's personalized model aggregation method with an interpretable approach for aggregating local models~\cite{li_federated_2020}. It also accommodates first-order Per-FedAvg~\cite{fallah2020personalized}.

Although the existing PFL methods have achieved promising results, the prior knowledge from single global model for local training~\cite{tan2022towards} hinders the development of PFL.
Specifically, we analyze the shortcomings of current PFL methods as follows:
1) utilizing the same global model for direct local training could potentially disregard the client's sampling information.
As shown in Figure~\ref{fig_intuition}, a single global model provides global knowledge directly for local training, which overlooks the client-sampling information when the global knowledge is transferred to specific clients.
2) Explicitly extracting prior knowledge can be a challenging task. Most of the insightful works~\cite{dieuleveut2021federated, marfoq2021federated} propose assumptions for recovering this incomplete information, but these assumptions are implicit, which limits the way to use the information to develop personalized strategies.

To address the former issue above, we propose framework pFedBreD to inject personalized prior knowledge into the one provided by a global model. As shown in Figure~\ref{fig_intuition}, it is injected in the $2^{nd}$ step and the local knowledge is transferred into global model via local models instead of directly aggregating personalized models~\cite{li_federated_2020}. To address the latter, we introduce \textit{relaxed mirror descent} RMD to explicitly extract the prior for exploring personalized strategies.

Our method is backed up with direct theoretical support from Bayesian modeling in Section~\ref{sec_mth} and a convergence analysis in Section~\ref{sec_alg}, which provides a linear bound $\mathcal{O}(1/TN)$ with aggregation noise and a quadratic speedup $\mathcal{O}(1/(TNR)^{2})$ without.\footnote{$TN$ and $TNR$: total global / local epoch in FL system. See Appendix~\ref{appdx_glsry}.}
Meanwhile, the existence and validity of the injection and extraction aforementioned information is verified in Section~\ref{ssec_analysis}.
The remarkable performance of the implements of the proposed method is tested on 5 datasets and 8 benchmarks. Consistently, our method reach the \textit{state-of-the-art}. Especially, the improvement of accuracy on task DNN-FEMNIST~\cite{caldas2018leaf} is up to 3.5\%. Extensive ablation study demonstrate that parts of the hybrid strategy $\mathbf{mh}$ are complementary to each other.
Our contributions can be summarized as follows:
\begin{itemize}
    \item The problem of overlooking client-sampling information at prior knowledge being transferred is introduced in this paper, and we first investigate the possibility of explicitly expressing the prior knowledge of the information and design personalized strategies on it.
    \item To express the personalized prior, we model PFL into a Bayesian optimization problem, \textit{Global-MLE and Local-MAP}. A novel framework, pFedBreD, is proposed for computing the modeled problem, and RMD is introduced to explicitly extract prior information.
    \item Sufficient experiments demonstrate our method surpasses most baselines on public benchmarks, thereby showcasing its robustness to data heterogeneity, particularly in cases involving small aggregation ratios and non-convex local objective settings.
\end{itemize}

\begin{figure}[t]
    \centering
    \includegraphics[width=.95\textwidth]{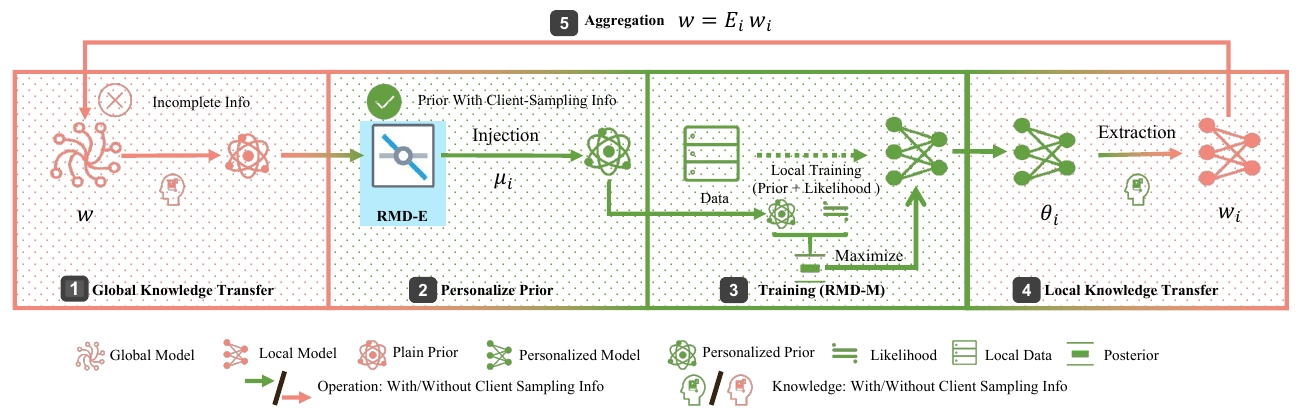}
    \caption{pFedBreD framework: Global-MLE and Local-MAP. The personalized prior knowledge is injected into the global model of the global problem (MLE) in the $2^{nd}$ step for local training. The local knowledge is extracted from the local problem (MAP) in the $4^{th}$ step for aggregation. 
    }
    \label{fig_intuition}
\end{figure}

\section{Related Works}
\label{sec_rltw}

\paragraph{Regularization}
Researchers have developed a variety of approaches based on regularization to handle the PFL challenge in recent years (\textit{e.g.}, FedU~\cite{dinh2021fedu}, pFedMe~\cite{t2020personalized}, FedAMP~\cite{huang2021personalized}, HeurFedAMP~\cite{huang2021personalized}).
All of these approaches' personalized objective functions can be expressed as $J(\theta) + R(\theta;\mu)$ where $J(\theta)$ is the loss function of the local problem and $R(\theta;\mu)$ is the regularization term used to restrict the deviation between $\theta$ and $\mu$ (\textit{e.g.}, $R(\theta;\mu)=\frac{1}{2}||\theta-\mu||^{2}$ in pFedMe).

\paragraph{Meta Learning}
One of the most representative meta-learning based single-model PFL approach is the well-known Per-FedAvg~\cite{fallah2020personalized}, aiming to find an initialization that is easy to fine-tune. That is, the global model in the FL setting is regarded as a meta model in MAML~\cite{finn_model-agnostic_2017, fallah2020convergence}, where the objective function of the local problem is $J(\theta-\eta\nabla J(\theta))$. Researchers also show the connections between FedAvg~\cite{mcmahan2017communication} and Reptile~\cite{nichol2018first}, another meta learning framework.~\cite{jiang2019improving} shows how to improve personalization of FL via Reptile. Proximal updating is also used in meta-learning based algorithms such as~\cite{zhou2019efficient}. One of our strategies $\mathbf{meg}$ is the one motivated by MAML.

\paragraph{Expectation Maximization}
Two EM-based~\cite{dempster1977maximum} methods are proposed, \textit{e.g.}, FedSparse in~\cite{louizos2021expectation} and FedEM in~\cite{dieuleveut2021federated}. Both of them focus on communication compression. The latter provides a variance reduce version and assumes complete information (or data) of the global model obeys distribution in X-family. Another FedEM~\cite{marfoq2021federated} combines Bayesian modeling, Federated Multi-task learning (FTML) and EM. Our framework pFedBreD is a expectation maximuzatioin and maximum a posteriori estimate (EM-MAP)~\cite{dieuleveut2021federated} algorithm with personalized prior specified.

\paragraph{Bayesian FL}
In recent years, studies of PFL with Bayesian learning have been proposed.
In related approaches, FLOA~\cite{liu2021bayesian} and pFedGP~\cite{achituve2021personalized} are proposed with KL divergence regularization in the loss function, which is comparable to applying specific assumption of X-family prior in pFedBreD see Appendix~\ref{appdx_ssec_bdxf} for details. Our implementation doesn’t use a Bayesian neural network (BNN) model as an inferential model as others do (\textit{e.g.}, pFedGP uses a Gaussian process tree and pFedBayes~\cite{zhang2022personalized} uses BNN). Instead, to eliminate weight sampling cost in Bayesian methods, prior knowledge is introduced through regularization term.
\section{Preliminary}

\paragraph{Overlooked Information in Prior Knowledge}
From a Bayesian and info. perspective, the global knowledge transferred in conventional method with single global model has no mutual information (MI) with client sampling $i$, i.e., formally, $w=\mathbf{E}_{i}w_{i}=\mathbf{E}_{i}w_{i}|i\Rightarrow$ MI $I(w;i)=0$, in particular when applying reg. $R(w^{(t)};...)$ or local init. $w^{(t)}_{i,0}\leftarrow w$ where $w_{i}$ is the local model on the $i^{th}$ client. This makes the specific model on each client have to re-obtain this information from scratch solely from the data during training, especially impacted on hard-to-learn representations and datasets.

\paragraph{Bregman-Moreau Envelope}
Bregman divergencee~\cite{bregman1967relaxation} is employed as a general regular term in our local objective that exactly satisfies the computational requirements and prior assumption, and is formally defined in Eq.~(\ref{equ_def_bd}).
\begin{equation}
\label{equ_def_bd}
\begin{aligned}
\mathcal{D}_{g}(x,y) :&= g(x) - g(y) - \langle \nabla g(y) , x-y \rangle %
\end{aligned}
\end{equation}
where $g$ is a convex function.  For convenience, $g$ is assumed to be strictly convex, proper and differentiable such that Bregman divergence is well-defined. To utilize the computational properties of Bregman Divergence in optimization problems, we introduce the following definition in Eq.~(\ref{equ_def_pmbme})~\cite{bauschke2006joint, bauschke2003bregman}: 
Bregman proximal mapping, Bregman-Moreau envelope, and the relationship between them.
\begin{equation}
\label{equ_def_pmbme}
\begin{aligned}
\mathcal{D}\mathbf{prox}_{g,\lambda^{-1}}f(x)&:=\arg\min_{\theta}\{f(\theta)+\lambda \mathcal{D}_{g}(\theta,x)\}, \\
\mathcal{D}\mathbf{env}_{g,\lambda^{-1}}f(x)&:=\min_{\theta}\{f(\theta)+\lambda \mathcal{D}_{g}(\theta,x)\}, \\
\nabla \mathcal{D}\mathbf{env}_{g,\lambda^{-1}}f(x)&=\lambda\nabla^{2}g(x)[x-\mathcal{D}\mathbf{prox}_{g,\lambda^{-1}}f(x)],
\end{aligned}
\end{equation}
where $\lambda > 0$ denotes the regular intensity in general and the variance of the prior in our modeling.

\paragraph{Exponential Family}
The regular exponential family (X-family) is a relatively large family that facilitates calculations.
Therefore, to yield the prior, we employ the X-family~\cite{banerjee2005clustering}defined in Eq.~(\ref{equ_def_ef}).
\begin{equation}
\label{equ_def_ef}
\begin{aligned}
\mathbf{P}_{ef}(\mathcal{V};s,g)&=h(\mathcal{V})\exp\{\langle \mathcal{V}, s \rangle -g(s)\} =h(\mathcal{V})\exp\{-\mathcal{D}_{g^{*}}(\mathcal{V},\mu)+g^{*}(\mathcal{V})\},
\end{aligned}
\end{equation}
where $g$ is assumed to be convex, $\mathcal{D}_{g}(\cdot,\cdot)$ is the Bregman divergence, and $g^{*}$ is the Fenchel Conjugate of $g$.
In Eq.~(\ref{equ_def_ef}), $s$, $h(\mathcal{V})$ and $g(s)$ are respectively the natural parameter, potential measure and logarithmic normalization factor, where we have the mean parameter $\mu=\nabla g(s)$. Additionally, to highlight the variance, the scaled exponential family (SX-family) is introduced in Eq.~(\ref{equ_def_sef})
\begin{equation}
\label{equ_def_sef}
\begin{aligned}
\mathbf{P}_{sef}(\mathcal{V};\lambda,s,g)&=h_{\mathcal{V}}(\mathcal{V})\exp\{\lambda[\langle \mathcal{V}, s \rangle - g(s)]\} 
=h_{\lambda}(\mathcal{V})\exp\{-\lambda\mathcal{D}_{g^{*}}(\mathcal{V},\mu)+\lambda g^{*}(\mathcal{V})\},
\end{aligned}
\end{equation}
where $\log h_{\lambda}(\mathcal{V})$ is the scaled potential measure, and the scale parameter $\lambda$ is employed to highlight the variance. Moreover, $\mathcal{V}$ is assumed to be the minimal sufficient statistic of the complete information for local inference, details of which can be found in Section~\ref{sec_mth}.

\section{Methodology}
\label{sec_mth}
In this section\footnote{More details of equations are in Appendix~\ref{appdx_sec_doe}.}, we introduce missing client-sampling information based on classic FL, use EM to reduce the computational cost of the information-introduced FL problem, and propose RMD, a class of prior selection strategies, based on the E-step in EM.
The general FL classification problem with KL divergence could be formulated as Eq.~(\ref{equ_gfl})~\cite{mcmahan2017communication,tan2022towards}.
\begin{equation}
\label{equ_gfl}
\begin{aligned}
\arg\min_{w}\mathbf{E}_{i}\mathbf{E}_{d_{i}}\mathbf{KL}(\mathbf{P}(y_{i}|x_{i})||\hat{\mathbf{P}}(y_{i}|x_{i},w))
=\arg\max_{w}\mathbf{E}_{i}\mathbf{E}_{d_{i}}\mathbf{E}_{y_{i}|x_{i}}\log\hat{\mathbf{P}}(y_{i}|x_{i},w),(x_{i},y_{i})\in d_{i},
\end{aligned}
\end{equation}
where we rewrite the discriminant model as an maximum likelihood estimation (MLE) problem~\cite{le_cam_maximum_1990} of $y_{i} | x_{i}$ in the right hand side (R.H.S.) of Eq.~(\ref{equ_gfl}). $(x_{i}, y_{i})$ represent the pairs of input and label respectively in dataset $d_{i}$ on the $i^{th}$ client, and $\hat{\mathbf{P}}(y_{i}| x_{i},w)$ is the inferential model parameterized by $w$. Each local data distribution is presuppose to be unique, so using the global model with local data for inference and training could overlook the fact that
the client has been sampled before transmitting the global model, and the prior knowledge transmitted directly via the global model as the local training prior knowledge (\textit{e.g.} via initial points, penalty points in dynamic regular terms, etc.) has no mutual information with the client sampling, \textit{i.e.}, the global model $w = \mathbf{E}_{i}w_{i} = \mathbf{E}_{i}w_{i} | i$.
Thus, to reduce the potential impact of the overlooked information, the complete information $\Theta_{i}$ on the $i^{th}$ client is introduced which turns Eq.~(\ref{equ_gfl}) into Eq.~(\ref{equ_gfl_em}).
\begin{equation}
\label{equ_gfl_em}
\arg\max_{w}\mathbf{E}\log\int_{\Theta_{i}}\hat{\mathbf{P}}(y_{i}|x_{i},\Theta_{i},w)\mathbf{P}(\Theta_{i}|x_{i},w)d\Theta_{i},
\end{equation}
where $\mathbf{E} = \mathbf{E}_{i}\mathbf{E}_{d_{i}}\mathbf{E}_{y_{i}|x_{i}}$ and the direct calculation of this is computationally expensive~\cite{CSILLERY2010410}.

\paragraph{Framework: Leveraging Expectation Maximization for Prior Parameter Extraction}

The integral term in Eq.~(\ref{equ_gfl_em}) makes direct computation impossible~\cite{CSILLERY2010410}, so we employ EM to approximate the likelihood with unobserved variables~\cite{dempster1977maximum} as shown in Eq.~(\ref{equ_gfl_eml}), where $\mathbf{Q}(\Theta_{i})$ is any probability measure.
\begin{equation}
\label{equ_gfl_eml}
\begin{aligned}
 \sum_{i}\log\hat{\mathbf{P}}(y_{i}|x_{i},w) 
\ge \sum_{i}\mathbf{E}_{\mathbf{Q}(\Theta_{i})}[\log\hat{\mathbf{P}}(y_{i}|x_{i},\Theta_{i},w) +\mathbf{E}_{y_{i}|x_{i},w}\log\mathbf{P}(\Theta_{i}|d_{i},w)
].
\end{aligned}
\end{equation}
Assuming that prior $\Theta_{i}|d_{i},w \sim \hat{\mathbf{P}}_{sef}(\Theta_{i};\lambda,s_{i}(w;d_{i}),g)$\footnote{A-posteriori distribution for local client whose prior knowledge is from global model. See Appendix~\ref{appdx_sec_doe}.} and the local loss function on the $i^{th}$ client $f_{i}(\Theta_{i},w)$ is $\mathbf{E}_{d_{i}}[-\log\mathbf{P}(y_{i}|x_{i},\Theta_{i},w)]$, we have the left hand side (L.H.S.) of the Eq.~(\ref{equ_gfl_sembme}) from (\ref{equ_gfl_eml}). Here is an assumption for simplification that $\theta_{i}$ contains all the information for local inference, \textit{i.e.} $\theta_{i} = \Theta_{i}$ and $\mathbf{P}(y_{i}|x_{i},\Theta_{i},w) = \mathbf{P}(y_{i}|x_{i},\Theta_{i})$. It happens when $\theta_{i}$ is all the parameters of the personalized model and we only use the personalized model for inference. Thus,  $f_{i}(\theta_{i})= \mathbf{E}_{d_{i}}[-\log\hat{\mathbf{P}}(y_{i}|x_{i},\theta_{i})]$. Thus, we can optimize an upper bound as a bi-level optimization problem as shown in the R.H.S. of the Eq.~(\ref{equ_gfl_sembme}) to solve Eq.~(\ref{equ_gfl}) approximately, where mean parameter $\mu_{i}=\nabla g \circ s_{i}$\footnote{The following $d_{i}$ is omitted with the same footnote $i$ in $\mu_{i}$  for simplification ($\mu_{i}(\cdot)\leftarrow\mu_{i}(\cdot;d_{i})$).}~\cite{ben2009robust}. And, we can derivate our framework as shown in Section~\ref{sec_alg}.
\begin{equation}
\label{equ_gfl_sembme}
\begin{aligned}
-\max_{w,\{\theta_{i}\}}\mathbf{E}_{i}&\{-f_{i}(\theta_{i})-\lambda\mathcal{D}_{g^{*}}(\theta_{i},\mu_{i}(w))\}  \le\min_{w}\mathbf{E}_{i}
\colorbox{green!20}{$\min_{\{\theta_{i}\}}\{f_{i}(\theta_{i})+\lambda\mathcal{D}_{g^{*}}(\theta_{i},\mu_{i}(w))\}$}.
\end{aligned}
\end{equation}

\paragraph{Strategies: Relaxing Mirror Descent for Prior Selection }
To extract the prior strategies and implement $\mu_{i}$ E-step of EM in close-form, we propose a method called relaxed mirror descent (RMD), where the mirror descent (MD) is EM in X-family~\cite{kunstner2021homeomorphic}. MD can be generally written as Eq.~(\ref{equ_def_md}) from the old $\hat{w}$ to the new one $\hat{w}^{+}$ in each iteration~\cite{mcmahan2017survey, kunstner2021homeomorphic}.
\begin{equation}
\label{equ_def_md}
\hat{w}^{+}\leftarrow\arg\min_{\hat{\theta}}\{f(\hat{w})+\langle\nabla f(\hat{w}),\hat{\theta}-\hat{w}\rangle+\hat{\lambda}\mathcal{D}_{\hat{g}}(\hat{\theta},\hat{w})\}.
\end{equation}
According to the Lagrangian dual, we rewrite the problem into a more general variant shown in Eq.~(\ref{equ_gfl_md}) with relaxed restrictions and superfluous parameter.
\begin{equation}
\label{equ_gfl_md}
\begin{aligned}
\arg\min_{\hat{\theta},\hat{\mu}}\{\Psi(\hat{\theta}, \hat{w})+\langle\nabla\Phi(\hat{w}),\hat{\mu}-\hat{w}\rangle  +\lambda\mathcal{D}_{g^{*}}(\hat{\theta},\hat{\mu})+(2\eta)^{-1}||\hat{\mu}-\hat{w}||^{2}\}.
\end{aligned}
\end{equation}
We can transform Eq.~(\ref{equ_gfl_md}) back into Eq.~(\ref{equ_def_md}) by setting $\Phi(\hat{w})$ to satisfy $\nabla\Phi(\hat{w})=\nabla f(\hat{w})$, and defining $\Psi(\hat{\theta},\hat{w})$ as a function with $f(\hat{w})$ and a penalty term to make $\hat{\theta}$ and $\hat{w}$ close as possible (\textit{e.g.}, $\hat{\lambda}\mathcal{D}_{\hat{g}}(\hat{\theta},\hat{w})$). This provides us a way to extract $\mu_{\Phi}$ the function to generate mean parameter of the prior, as shown in Eq.~(\ref{equ_gfl_mdlb}), which is minimizing an upper bound of the problem in Eq.~(\ref{equ_gfl_md}).
\begin{equation}
\label{equ_gfl_mdlb}
\begin{aligned}
\mathcal{D}\mathbf{env}_{g^{*},\lambda^{-1}}\Psi(\cdot,w)(\mu_{\Phi}(w)) 
 &=\min_{\theta}\{\Psi(\theta,w)+\lambda\mathcal{D}_{g^{*}}(\theta,\mu_{\Phi}(w))\}\\
\colorbox{cyan!20}{$\mu_{\Phi}(w)$}&=\arg\min_{\mu}\{\langle\nabla\Phi(w),\mu-w\rangle + (2\eta)^{-1}||\mu-w||^{2}\}.
\end{aligned}
\end{equation}

By optimality condition, we have $\mu_{\Phi}(w)= w - \eta\nabla\Phi(w)$, which can be specified by $\Phi$. The remaining part is a Bregman-Moreau envelope. Thus, we can optimize the upper bound with an EM-MAP method, alternately computing $\mu_{\Phi}(w)$ and $\mathcal{D}\mathbf{prox}_{g^{*},\lambda^{-1}}\Psi(\cdot,w)(\mu_{\Phi}(w))$.

\section{Framework Design}
\label{sec_alg}

\paragraph{Problem Formulation that Highlights Personalized Prior}
Inspired by the aforementioned motivation, the personalized models $\theta_{i}$ and mean parameters are respectively the solution of $\mathcal{D}\mathbf{env}_{g^{*},\lambda^{-1}}f_{i}(\mu_{i}(w))$ and $\mu_{i}(w)$ on the $i^{th}$ client, where $w$ is the global model. We assume that personalized model contains all the local information required for inference on the $i^{th}$ client, and satisfies 
 $\theta_{i}|d_{i},w \sim \mathbf{P}_{sef}(\theta_{i};\lambda,s_{i}(w),g)$. The global problem can be written as Eq.~(\ref{equ_main_problem_global}).
\begin{equation}
\label{equ_main_problem_global}
\begin{aligned}
    \min_{w}\mathbf{E}_{i}\{F_{i}(w):=\mathcal{D}\mathbf{env}_{g^{*},\lambda^{-1}}f_{i}(\text{\colorbox{cyan!25}{$\mu_{i}(w)$}})\}.
\end{aligned}
\end{equation}
The given $g$ is strictly convex, $\lambda>0$, $f_{i}$ is the local loss function, $s_{i}(w)$ is the natural parameter and $\mu_{i}(w)=\mathbf{E}_{\theta_{i}|x_{i},w}\theta_{i}=\nabla g(s_{i}(w))$ is the mean (or expectation) parameter in Eq.~(\ref{equ_main_problem_global}).

\paragraph{Framework: pFedBreD}
To solve the optimization problem in Eq.~(\ref{equ_main_problem_global}), we use gradient-based methods to solve the global problem using the gradient of $F_{i}$:
\begin{equation}
\label{equ_gdmp}
\begin{aligned}
\nabla F_{i}(w)=\lambda\mathbf{D}\mu_{i}(w)\nabla^{2}g^{*}(\mu_{i}(w))[\mu_{i}(w) -\mathcal{D}\mathbf{prox}_{g^{*},\lambda^{-1}}f_{i}(\mu_{i}(w))],
\end{aligned}
\end{equation}
where $\mathbf{D}$ is the gradient operator of the vector value function, and $\nabla^{2}$ is the Hessian operator.\footnote{The details of first-order methods is in Appendix~\ref{appdx_ssec_foMtd}.}
The framework is shown as Algorithm~\ref{alg_pFedBreD}, where $\mathcal{I}$ is the client selecting strategy for global model aggregation; $w_{init}$ and $\theta_{init}$ are the initialization strategies on the $i^{th}$ client; $\alpha_{m}$ is the main problem step-size; $T$, $R$, $N$ are respectively the total number of iterations, local iterations, and clients. $\beta$ is used in the same trick as~\cite{karimireddy2020scaffold, t2020personalized}. The strategies to derive the initialization points of $w_{i}$ and $\theta_{i}$ at each local epoch are $w_{i,0}^{(t)}\leftarrow w^{(t-1)}$ and $\theta_{i,0}^{(t)}\leftarrow \theta_{i,R}^{(t-1)}$.

\begin{algorithm}[t]
\caption{Algorithm for pFedBreD}  
\label{alg_pFedBreD}
\textbf{Input}: $\mathcal{I}$,\{$d_{i}$\}, $i=1...N$\\
\textbf{Parameter}: $\alpha_{m}$,$g$,$\lambda$,$T$,$R$,\{$w_{init}$,$\theta_{i}$,$\mu_{i}$,\},$i=1...N$\\
\textbf{Output}: $w^{(T)}$,\{$\theta_{i}^{(T)}$\}, $i=1...N$
\begin{algorithmic}[1] %
\STATE Initialize $w^{(0)}$, $\{\theta_{i}^{(0)}\}$, $\{\mathcal{C}_{i,R}^{0}\}$;
\FOR{t=1...T}
\STATE \colorbox{pink!40}{Server sends $w^{(t-1)}$ to clients};
\FOR{i=1...N in parallel on each clients}
\STATE Initialize $w_{i,0}^{(t)}$ and $\theta_{i,0}^{(t)}$ with $w_{init}$ and $\theta_{init}$;
\FOR{r=1...R}
\STATE \colorbox{cyan!25}{Generate $\mu_{i,r}^{(t)}\leftarrow \mu_{i}(w_{i,r-1}^{(t)},...)$};
\STATE \colorbox{green!20}{$\theta_{i,r}^{(t)}\leftarrow\mathcal{D}\mathbf{prox}_{g^{*},\lambda^{-1}}f_{i}(\mu_{i,r}^{(t)})$};
\STATE \colorbox{green!20}{$w_{i,r}^{(t)}\leftarrow w_{i,r-1}^{(t)}-\alpha_{m}\nabla F_{i}(w_{i,r-1}^{(t)})$};
\ENDFOR
\ENDFOR
\STATE \colorbox{pink!40}{Server collects $\{w_{i,R}^{(t)}\}$ and calculate $w^{(t)}\leftarrow(1-\beta) w^{(t-1)} + \beta \mathbf{E}_{\mathcal{I}} w_{i,R}^{(t)}$};
\ENDFOR
\STATE \textbf{return} $w^{T}$, \{$\theta_{i}^{T}$\}.
\end{algorithmic}
\end{algorithm}

\paragraph{Implementation: Maximum Entropy and Meta-Step}
Practically, two main parts of the pFedBreD are needed to be implemented:
\begin{itemize}
    \item $g$, the function used to derive the logarithmic normalization factor, determines the type of prior to be used;
    \item $\{s_{i}\}$ or $\{\mu_{i}\}$, the functions used to derive the natural parameter and mean parameter for the personalized local prior, determine which particular prior is used.
\end{itemize}

We propose first-order implementations based on maximum entropy rule~\cite{friedman1971jaynes, jaynes1957information}. In the SX-family, the Gaussian distribution has the maximum entropy among continuous distributions when $g$, $\mu_{i}$ (the first-order moment), and $\lambda$ (the parameter determining the second moment) are given.
Thus, we employ the scaled norm square $g=g^{*}=\frac{1}{2}||\cdot||^{2}$ to turn the prior into a spherical Gaussian, in order to maximize the entropy of the prior on a particular client. With this assumed prior, we have $\nabla g = \nabla g^{*}=I$, which means $\mu_{i}=s_{i}$. We can choose a different $\Phi_{i}$ as shown in Eq.~(\ref{equ_def_spmmdl})\footnote{A variant of $\mathbf{mh}$ is in Appendix~\ref{appdx_ssec_v}.} to generate selection strategies according to Section~\ref{sec_mth}, via $\mu_{i}(w) = w - \eta\nabla\Phi_{i}(w)$ (meta-step).
\begin{equation}
\label{equ_def_spmmdl}
    \Phi_{i}=\left\{
    \begin{aligned}
    & f_{i}\\
    & F_{i}\\
    & f_{i}+F_{i}
    \end{aligned}
    \right. \qquad
    \text{\colorbox{cyan!25}{$\mu_{i,r}^{(t)}$}} \leftarrow \left\{
    \begin{aligned}
        & w_{i,r-1}^{(t)} - \eta_{\alpha} \nabla f_{i}(w_{i,r-1}^{(t)}),  &\mathbf{lg}\\
        & w_{i,r-1}^{(t)} - \eta (w_{i,R}^{(t-1)}-\theta_{i,r-1}^{(t)}), &\mathbf{meg}\\
        & w_{i,r-1}^{(t)} - \eta_{\alpha} \nabla f_{i}(w_{i,r-1}^{(t)}) - \eta (w_{i,R}^{(t-1)}-\theta_{i,r-1}^{(t)}), & \mathbf{mh}
    \end{aligned}
    \right.
\end{equation}
where $\eta_{\alpha}$ and $\eta$ are the meta-step-size parameters. Practical parameter selection strategies with meta-step are shown as $\mu_{i,r}^{(t)}$ in Eq.~(\ref{equ_def_spmmdl}).
The three
of $\mu_{i}$, \textit{i.e.} $\mathbf{lg}$, $\mathbf{meg}$ and $\mathbf{mh}$, represent \textbf{loss gradient}, \textbf{memorized envelope gradient} and \textbf{memorized hybrid} respectively.

\paragraph{Convergence Analysis}
we analyze the convergence of pFedBreD with RMD on a uniform client sampling $\mathbf{E}_{i}=\frac{1}{N}\sum_{i=1}^{N}$ setting for simplification. Other sampling methods can be obtained with client sampling expectation $\mathbf{E}_{i}[F_{i}] = F$, by changing sampling weights.
The assumptions, proof sketch and detailed notations are in Appendix~\ref{appdx_glsry} and Appendix~\ref{appdx_sec_theo}.

\begin{theorem}[pFedBreD's global bound]
\label{theo_gm}
Under settings in Section~\ref{sec_alg} and Appendix~\ref{appdx_sec_theo}, at global epoch $T \ge \frac{2}{\hat{\mu}_{F_{\cdot}}\tilde{\alpha}}$, by properly choose $\tilde{\alpha}_{m} = \alpha_{m} \beta R$, $\exists \tilde{\alpha}_{m}\le\min\{\frac{\beta}{\sqrt{2\dot{c}}},\frac{2}{\hat{\mu}_{F_{\cdot}}},\hat{\alpha}_{m}\}$, where $A = [\frac{\hat{L}_{g^{*}}}{\hat{\mu}_{F_{\cdot}}}(\hat{u}_{m} + \eta\hat{\gamma}_{\Phi})]^{2}(\frac{\hat{\gamma}_{f}^{2}}{|\tilde{d}_{i}|}+\hat{\epsilon}^2)$, $B = [\hat{L}_{\mathcal{E}}\hat{\gamma}_{\Phi}(1+\sigma_{\Phi})(\hat{u}_{m}+\eta\hat{\gamma}_{\Phi})]^{2}$, $C = \frac{\sigma_{\Phi}^{2}\hat{L}_{\mathcal{E}}^{2}(\hat{u}_{m}+\eta\hat{\gamma}_{\Phi})^{2}}{\hat{\mu}_{F_{\cdot}}^{3}}$, $\xi^{(t)}=(1-\frac{\tilde{\alpha}\hat{\mu}_{F_{\cdot}}}{2})^{-t-1}$, $\bar{w}^{(T)} := \frac{\sum_{t=0}^{T-1}\xi^{(t)}w^{(t)}}{\sum_{t=0}^{T-1}\xi^{(t)}}$ and $\hat{\alpha}_{m}:= \frac{\hat{\mu}_{F_{\cdot}}\beta R}{e(1+\sigma_{\Phi})\hat{L}_{\mathcal{E}}(\hat{u}_{m}+\eta\hat{\gamma}_{\Phi}) 2^{R+6\frac{1}{2}}(\frac{1}{R}+2) + 18(\hat{\mu}_{F_{\cdot}}\beta R)\hat{L}_{F}}$, such that:
    $$
    \begin{aligned}
        \mathcal{O}[\mathcal{D}_{F}(\bar{w}^{(T)},w^{*})] = & \mathcal{O}(\hat{\mu}_{F_{\cdot}}e^{-\tilde{\alpha}_{m}\hat{\mu}_{F_{\cdot}}T/2} \mathbf{\Delta}^{(0)})+\mathcal{O}(\frac{A\lambda^{2}+B}{\hat{\mu}_{F_{\cdot}}}) \\
        &+ \mathcal{O}(\frac{(N/S-1){\sigma_{F,*}}^{2}}{N T\hat{\mu}_{F_{\cdot}}}) + \mathcal{O}(\frac{2^{R}C}{T^{2}\beta^{2}R^{2}}[R\sigma_{F,*}^{2} + A\lambda^{2}+B]).
    \end{aligned}
    $$
    .
\end{theorem}

\begin{theorem}[pFedBreD$_{ns}$'s  first-order personalization bound]
\label{theo_pg}
Under the same conditions as in Theorem~\ref{theo_gm}, with prior assumption of a spherical Gaussian and first-order approximation, the bound for the gap between the personalized approximate model and global model in the Euclidean space is:           
    $$
    \mathbf{E}||\tilde{\theta}_{i}(\bar{w}^{T})-w^{*}||^{2} \le \mathcal{O}(\dot{\delta}_{p}) + \mathcal{O}[\dot{c}_{p}\mathcal{D}_{F}(\bar{w}^{(T)},w^{*})]
    $$
    where $\dot{\delta}_{p}= \frac{2}{\hat{\mu}_{F_{i,\cdot}}^{2}}(\frac{\hat{\gamma}_{f}^{2}}{|\tilde{d}_{i}|}+\hat{\epsilon}^2)+\frac{2}{\lambda^{2}}\epsilon_{1}^{2} + \frac{4}{\lambda^{2}}\sigma_{F,*}^{2} +\frac{1}{2}\eta^{2}\mathcal{G}_{\Phi}^{2}$, and $\dot{c}_{p}=(\frac{32}{\lambda^{2}}\hat{L}_{F} + \frac{8}{\hat{\mu}_{F_{\cdot}}})$.
\end{theorem}

\begin{remark} Theorem~\ref{theo_gm} shows the main factors that affect the convergence of a global model are as follows: random mini-batch size, client drift error, aggregation error, heterogeneous data, dual space selection, local approximation error, and selection strategy for exponential family prior mean and variance. These can be divided into four categories based on their computational complexity. The first and second term shows that the proper fixed $\tilde{\alpha}_{m}$ can linearly reduce the influence of initial error $\mathbf{\Delta}^{(0)}$ and the global model converges to a ball near the optimal point. The radius of this ball is determined by the personalized strategy and local errors (including local data randomness and envelope approximation errors). The third term implies that a linear convergence rate $\mathcal{O}(1/(NT))$ can be obtained w.r.t. the total global epoch $NT$ in the presence of aggregation noise. Without client sampling $N=S$, according to the fourth term, the quadratic rate $\mathcal{O}(1/(TNR)^{2})$ can be obtained with $\beta=\mathcal{O}(N)$ or $\beta=\mathcal{O}(N\sqrt{R})$ (Note that the number of local epoch $R$ cannot be too large due to client drift,  according to $2^{R}$).
Theorem~\ref{theo_pg} shows that, with spherical Gaussian prior assumption and first-order methods, the radius of the neighborhood range for the minimum that includes the personalized model on $i^{th}$ client, $\mathcal{O}(C_{\Phi,F,f,d} +\frac{1}{\lambda^{2}}(\epsilon_{1}^{2} + \sigma_{F,*}^{2} + \frac{B\hat{L}_{F}}{\hat{\mu}_{F_{\cdot}}}) + \lambda^{2}\frac{A}{\hat{\mu}_{F_{\cdot}}})$, can be trade-off by $\lambda$, and is affected by the prior selection strategies and first-order approximate error besides the elements in Theorem~\ref{theo_gm}. (Note that the Euclidean space is self-dual.)
\end{remark}

\section{Experiments}
\label{sec_exp}

\subsection{General Settings}

\label{ssec_exps}
\textbf{Tricks, Datasets and Models:} our experiments include several tasks: CNN~\cite{hosang2015taking} on CIFAR-10~\cite{dinh2021fedu, krizhevsky2009learning}, LSTM~\cite{hochreiter1997long} on Sent140~\cite{caldas2018leaf} and MCLR/DNN on FEMNIST~\cite{caldas2018leaf}/FMNIST~\cite{t2020personalized, xiao2017fashion}/MNIST\cite{t2020personalized, lecun1998gradient}. The details of tricks (FT, AM), data heterogeneity and models are in Appendix~\ref{appdx_sec_exp}.

\textbf{Baselines:} we choose following algorithms as our baselines: FedAvg~\cite{mcmahan2017communication}, Per-FedAvg~\cite{fallah2020personalized}, pFedMe~\cite{t2020personalized}, FedAMP~\cite{huang2021personalized}, pFedBayes~\cite{zhang2022personalized} and FedEM~\cite{marfoq2021federated}. These baselines are respectively classical FL, MAML-based meta learning, regularization based, FTML methods, variational inference PFL and FMTL with EM.

\textbf{Global Test and Local Test:} the global and personalized model, represented by $\textbf{G}$ and $\textbf{P}$, are evaluated with global and local tests respectively. Global test means all the test data is used in the test. Local test means only the local data is used for the local test and the weight of the sum in local test is the ratio of the number of data. The results of average accuracy per client are shown in Table~\ref{tbl_res_lnlm}. Each experiment is repeated 5 times. More details are in Appendix~\ref{appdx_sec_exp}. For readability, we only give the error bar in the main Table~\ref{tbl_res_lnlm} and Table~\ref{tbl_res_ablation}, and keep one decimal except for the main Table~\ref{tbl_res_lnlm}.

\paragraph{Hyperparameter Settings}
The step-size of the main problem, $\alpha_{m}$, and the personalized step-size, $\alpha$, for all methods are 0.01. $\beta$ is 1, and the number of local epochs, $R$, is 20 for all datasets. $\lambda$ is chosen from 15.0 to 60.0. The batch sizes of Sent140 and the other datasets are 400 and 20, respectively, and the aggregation strategy, $\mathcal{I}$, is uniform sampling. The ratios of aggregated clients per global epoch are 40\%, 10\%, and 20\% for Sent140, FEMNIST, and the other datasets, respectively. The numbers of total clients, $N$, are 10, 198, 20, and 100 for Sent140, FEMNIST, CIFAR-10, and other datasets. The number of proximal iterations is 5 for all settings with proximal mapping. In our implementations, $\eta_{\alpha}$ and $\eta$ are respectively 0.01 and 0.05.

\paragraph{Summarizing the Effects of Hyper-parameters}
We test the hyper-parameter effect of $\eta$ and $\lambda$ in our implementation pFedBreD$_{ns,\mathbf{mh}}$. The details are in Appendix~\ref{appdx_sec_exp}. From the results, we find that it will degrade the test accuracy if the values of $\lambda$ or $\eta$ are too large or too small. The test accuracy of personalized model is more sensitive than the ones of global model. The test accuracy of personalized model is more sensitive to $\eta$ than to $\lambda$. Note that the hyper-parameters are roughly tuned, which shows the insensitivity of $\mathbf{mh}$, and better tuning could improve the performance in the Table~\ref{tbl_res_lnlm}.

\subsection{Analysis}
\label{ssec_analysis}
\paragraph{Comparative Analysis of Performance}
\begin{table}[t]
    \centering
    \caption{Results of average testing accuracy (\%) per client of each settings. We mark the best and second best performance by \textbf{bold} and \underline{underline}. \colorbox{orange!20}{Avg} and \colorbox{red!30}{Std}: the average results and the standard deviation of them on all tasks;\colorbox{cyan!20}{H.Avg} and \colorbox{blue!20}{H.Std}: the average results and the standard deviation of them on hard tasks (non-linear DNN with complex classification or architecture: DNN / CNN / LSTM on FEMNIST / CIFAR-10 / Sent140). The $\textbf{G}$ and $\textbf{P}$ are global and personalized model}
    \label{tbl_res_lnlm}
    \resizebox{.98\textwidth}{!}{
    \begin{tabular}{lcccccccccccc}
        \hline
        Methods / \textbf{Datasets} & \multicolumn{2}{c}{FEMNIST} & \multicolumn{2}{c}{FMNIST} & \multicolumn{2}{c}{MNIST} & CIFAR-10 & Sent140 &  \multicolumn{4}{c}{Statistics}\\
        Names - \textbf{Models} & MCLR & \cellcolor{cyan!20}DNN & MCLR & DNN & MCLR & DNN & \cellcolor{cyan!20}CNN & \cellcolor{cyan!20}LSTM & \cellcolor{orange!20}Avg & \cellcolor{red!30}Std & \cellcolor{cyan!20}H.Avg & \cellcolor{blue!20}H.Std \\      
        \hline
        FedAvg~\cite{mcmahan2017communication} -\textbf{G} & 53.38$_{\pm 0.26}$ & 57.04$_{\pm 0.08}$ & 82.75$_{\pm 0.04}$ & 80.09$_{\pm 0.06}$ & 86.59$_{\pm 0.03}$ & 88.26$_{\pm 0.05}$ & 57.51$_{\pm 0.07}$ & 70.86$_{\pm 0.01}$ & \cellcolor{orange!20}72.06 & \cellcolor{red!30}\underline{14.34} & \cellcolor{cyan!20}61.80 & \cellcolor{blue!20}7.85 \\
        FedAvg+AM  -\textbf{G} & 55.34$_{\pm 0.05}$ & 59.03$_{\pm 0.10}$ & 82.58$_{\pm 0.03}$ & 81.03$_{\pm 0.12}$ & 86.74$_{\pm 0.03}$ & 89.31$_{\pm 0.05}$ & 57.07$_{\pm 0.12}$ & 71.27$_{\pm 0.01}$ & \cellcolor{orange!20}72.80 & \cellcolor{red!30}14.01 & \cellcolor{cyan!20}62.46 & \cellcolor{blue!20}7.70 \\
        \hline
        FedEM~\cite{marfoq2021federated}  -\textbf{G} & 40.75$_{\pm 0.32}$ & 45.47$_{\pm 0.04}$ & 95.78$_{\pm 0.03}$ & 96.42$_{\pm 0.03}$ & 85.75$_{\pm 0.01}$ & 86.49$_{\pm 0.02}$ & 57.67$_{\pm 0.16}$ & 66.72$_{\pm 0.03}$ & \cellcolor{orange!20}71.88 & \cellcolor{red!30}22.28 & \cellcolor{cyan!20}56.62 & \cellcolor{blue!20}10.66 \\
        pFedBayes~\cite{zhang2022personalized} -\textbf{P} & 49.66$_{\pm 0.46}$ & - & 98.46$_{\pm 0.05}$ & 98.67$_{\pm 0.05}$ & 89.64$_{\pm 0.06}$ & 90.48$_{\pm 0.12}$ & - & - & \cellcolor{orange!20}- & \cellcolor{red!30}- & \cellcolor{cyan!20}- & \cellcolor{blue!20}- \\
        FedAMP~\cite{huang2021personalized}  -\textbf{P} & \underline{60.04}$_{\pm 0.08}$ & 66.79$_{\pm 0.04}$ & \textbf{98.63}$_{\pm 0.02}$ & 98.72$_{\pm 0.01}$ & \textbf{90.81}$_{\pm 0.02}$ & 92.21$_{\pm 0.02}$ & 77.40$_{\pm 0.04}$ & 69.83$_{\pm 0.05}$ & \cellcolor{orange!20}81.80 & \cellcolor{red!30}15.21 & \cellcolor{cyan!20}71.34 & \cellcolor{blue!20}5.46 \\
        \hline
        pFedMe~\cite{t2020personalized}  -\textbf{P} & 50.74$_{\pm 0.10}$ & 53.56$_{\pm 0.12}$ & 97.60$_{\pm 0.03}$ & 98.63$_{\pm 0.01}$ & 88.20$_{\pm 0.05}$ & 90.51$_{\pm 0.01}$ & 72.24$_{\pm 0.05}$ & 69.36$_{\pm 0.02}$ & \cellcolor{orange!20}77.61 & \cellcolor{red!30}18.96 & \cellcolor{cyan!20}65.05 & \cellcolor{blue!20}10.06 \\
        pFedMe+FT  -\textbf{P} & 58.04$_{\pm 0.11}$ & 62.93$_{\pm 0.10}$ & 97.63$_{\pm 0.01}$ & 98.39$_{\pm 0.02}$ & 88.36$_{\pm 0.02}$ & 91.71$_{\pm 0.01}$ & 68.17$_{\pm 0.11}$ & 67.82$_{\pm 0.03}$ & \cellcolor{orange!20}79.13 & \cellcolor{red!30}16.53 & \cellcolor{cyan!20}66.31 & \cellcolor{blue!20}\textbf{2.93} \\
        pFedMe+AM  -\textbf{P} & 55.56$_{\pm 0.09}$ & 60.08$_{\pm 0.05}$ & 97.57$_{\pm 0.02}$ & 98.67$_{\pm 0.00}$ & 88.46$_{\pm 0.02}$ & 91.22$_{\pm 0.00}$ & 73.35$_{\pm 0.09}$ & 70.93$_{\pm 0.05}$ & \cellcolor{orange!20}79.48 & \cellcolor{red!30}16.79 & \cellcolor{cyan!20}68.12 & \cellcolor{blue!20}7.07 \\
        \hline
        Per-FedAvg~\cite{fallah2020personalized}   -\textbf{P} & 54.34$_{\pm 0.14}$ & 62.72$_{\pm 0.03}$ & 94.28$_{\pm 0.05}$ & 97.46$_{\pm 0.04}$ & 87.09$_{\pm 0.01}$ & 90.96$_{\pm 0.02}$ & 78.87$_{\pm 0.05}$ & 70.05$_{\pm 0.03}$ & \cellcolor{orange!20}79.47 & \cellcolor{red!30}15.74 & \cellcolor{cyan!20}70.54 & \cellcolor{blue!20}8.09 \\
        Per-FedAvg+FT -\textbf{P} & 55.34$_{\pm 0.15}$ & 63.34$_{\pm 0.01}$ & 95.76$_{\pm 0.07}$ & 98.10$_{\pm 0.01}$ & 87.56$_{\pm 0.03}$ & 89.58$_{\pm 0.01}$ & 79.68$_{\pm 0.04}$ & 70.20$_{\pm 0.01}$ & \cellcolor{orange!20}79.95 & \cellcolor{red!30}15.61 & \cellcolor{cyan!20}71.07 & \cellcolor{blue!20}8.20 \\
        Per-FedAvg+AM  -\textbf{P} & 56.66$_{\pm 0.09}$ & 65.74$_{\pm 0.02}$ & 92.08$_{\pm 0.10}$ & 98.24$_{\pm 0.02}$ & 86.91$_{\pm 0.04}$ & 90.85$_{\pm 0.02}$ & 78.97$_{\pm 0.03}$ & 70.73$_{\pm 0.05}$ & \cellcolor{orange!20}80.02 &\cellcolor{red!30}14.54 &  \cellcolor{cyan!20}71.81 & \cellcolor{blue!20}6.68 \\
        \hline
        $\mathbf{mh}$ (ours) -\textbf{P} & 56.34$_{\pm 0.09}$ & 64.93$_{\pm 0.03}$ & 98.44$_{\pm 0.01}$ & 98.73$_{\pm 0.01}$ & 89.83$_{\pm 0.02}$ & 92.04$_{\pm 0.01}$ & \underline{79.44}$_{\pm 0.02}$ & \underline{72.04}$_{\pm 0.01}$ & \cellcolor{orange!20}81.47 & \cellcolor{red!30}15.88 & \cellcolor{cyan!20}72.14 & \cellcolor{blue!20}7.26 \\
        $\mathbf{mh}$ (ours)+FT -\textbf{P} & 59.81$_{\pm 0.07}$ & \underline{67.53}$_{\pm 0.02}$ & \underline{98.51}$_{\pm 0.02}$ & \textbf{98.98}$_{\pm 0.03}$ & \underline{90.10}$_{\pm 0.03}$ & \textbf{92.96}$_{\pm 0.05}$ & 79.16$_{\pm 0.03}$ & 71.87$_{\pm 0.01}$ & \cellcolor{orange!20}\underline{82.37} & \cellcolor{red!30}14.92 & \cellcolor{cyan!20}\underline{72.85} & \cellcolor{blue!20}5.88 \\
        $\mathbf{mh}$ (ours)+AM -\textbf{P} & \textbf{60.64}$_{\pm 0.02}$ & \textbf{70.34}$_{\pm 0.01}$ & 98.48$_{\pm 0.01}$ & \underline{98.75}$_{\pm 0.01}$ & 89.88$_{\pm 0.01}$ & \underline{92.32}$_{\pm 0.02}$ & \textbf{80.60}$_{\pm 0.01}$ & \textbf{73.68}$_{\pm 0.01}$ & \cellcolor{orange!20}\textbf{83.09} & \cellcolor{red!30}\textbf{14.01} & \cellcolor{cyan!20}\textbf{74.87} & \cellcolor{blue!20}\underline{5.23} \\
        \hline
    \end{tabular}
    }
\end{table}

We compare our methods and the baselines from different perspectives, including convex or non-convex problems, easy or hard tasks, and text tasks. Additionally, we briefly discuss the absence of BNN on hard tasks.

\textbf{Convex or non-convex:} on non-convex problems, especially in hard tasks, our method significantly outperforms other methods by at least 3.06\% employing some simple tricks. On convex problem, FedAMP outperforms our method somewhat on convex problems with simple data sets. One explanation is that the learning lanscape is simple in shape for these problems and FedAMP converges faster for this case. One possible reason for this is that since FedAMP uses the distance between models as a similarity in the penalty point selection, giving greater weight to the model that is most similar to the local one. In the later stages of training, since there is only one global optimum, this penalty point tends not to change, and thus the method degenerates into a non-dynamic regular term. Compounding intuition, this method will not be as advantageous for non-convex problems and harder convex problems, as penalty point tends to fall into the local optimum and lead to degradation of the dynamic regular term.

\textbf{From easy to difficult task:} from the difference between the statistics of \colorbox{orange!20}{Avg} and \colorbox{cyan!20}{H.Avg} in Table~\ref{tbl_res_lnlm}, it can be observed that meta-step methods perform most consistently, with all other methods dropping at least 10\%. This is due to the simple and effective local loss design of MAML, with its learning-to-learning design philosophy that enables the method to be more stable in complex situations~\cite{fallah2020convergence,fallah2020personalized}.

\textbf{Personalized prior on text:} text tasks, as opposed to image tasks, generally have relatively rugged learning landscape.~\cite{mikolov_distributed_2013, 5206848, collobert_natural_nodate} This understanding is manifested in specific ways, such as parameter sensitivity, slow convergence, and struggling during the process. Thus, the overlooked prior information seems to be more important, which means that each local iteration not only obtains local knowledge from the data, but also the prior itself already contains some local knowledge. Therefore, there is no need to re-obtain this knowledge from scratch solely from the data during training.

\textbf{Absence of BNN on hard tasks:} complex BNN is not in Table~\ref{tbl_res_lnlm}, such as LSTM in pFedBayes, because it is difficult to conduct comparative experiments by fixing elements, \textit{e.g.}, inferential models, tricks and optimization methods. In pFedBayes, training often crashes on hard tasks and large datasets, as mentioned in~\cite{zhang2022personalized}. Our one-step-further research shows that it may be caused by the reparameterization tricks and vanilla Gaussian sampling. If we add tricks on it, the implementation will be very different from the original pFedBayes, and it is beyond this analysis.

\paragraph{Ablation Analysis of Personalized Prior}

\begin{table}[t]
    \centering
    \caption{Average local test accuracy of personalized model (\%) in ablation experiments.(\textcolor{blue!70}{$\uparrow$}/\textcolor{red!80}{$\downarrow$}:~average accuracy is increased/reduced; AC4PP: Additional cost for personalized prior; Grad. and Add.: cost about calculate gradient and addition; Other notations are the same in Table~\ref{tbl_res_lnlm}.)}
    \label{tbl_res_ablation}
    \resizebox{.98\textwidth}{!}{
    \begin{tabular}{lccccccccccc}
        \hline
         \textbf{Methods} & \multicolumn{2}{c}{FEMNIST} & \multicolumn{2}{c}{FMNIST} & \multicolumn{2}{c}{MNIST} & CIFAR-10 & Sent140 & \multicolumn{2}{c}{Statistics} & AC4PP\\
        & MCLR & \cellcolor{cyan!20}DNN & MCLR & DNN & MCLR & DNN & \cellcolor{cyan!20}CNN & \cellcolor{cyan!20}LSTM & \cellcolor{orange!20}Avg & \cellcolor{cyan!20}H.Avg \\
       \hline
        Non-PP
       & 50.7$_{\pm 0.10}$ & 53.6$_{\pm 0.12}$ & 97.6$_{\pm 0.03}$ & 98.6$_{\pm 0.01}$ & 88.2$_{\pm 0.05}$ & 90.5$_{\pm 0.01}$ & 72.2$_{\pm 0.05}$ & 69.4$_{\pm 0.02}$ & \cellcolor{orange!20}77.6 & \cellcolor{cyan!20}65.1 & None\\
         \hline  
         $\mathbf{lg}$ (ours) & 50.8$^{\text{\textcolor{blue!70}{$\uparrow$}}}_{\pm 0.05}$ & 49.1$^{\text{\textcolor{red!80}{$\downarrow$}}}_{\pm 0.53}$ & 98.3$^{\text{\textcolor{blue!70}{$\uparrow$}}}_{\pm 0.02}$& 98.4$^{\text{\textcolor{red!80}{$\downarrow$}}}_{\pm 0.02}$ & 88.4$^{\text{\textcolor{blue!70}{$\uparrow$}}}_{\pm 0.01}$ & 91.0$^{\text{\textcolor{blue!70}{$\uparrow$}}}_{\pm 0.00}$ & 65.7$^{\text{\textcolor{red!80}{$\downarrow$}}}_{\pm 0.46}$ & 60.7$^{\text{\textcolor{red!80}{$\downarrow$}}}_{\pm 0.41}$ & \cellcolor{orange!20}75.3\textcolor{red!80}{$\downarrow$} & \cellcolor{cyan!20}58.5\textcolor{red!80}{$\downarrow$} & Grad. $\times$ R \\
         $\mathbf{meg}$ (ours) & 50.3$^{\text{\textcolor{red!80}{$\downarrow$}}}_{\pm 0.07}$ & 53.9$^{\text{\textcolor{blue!70}{$\uparrow$}}}_{\pm 0.06}$ & 97.8$^{\text{\textcolor{blue!70}{$\uparrow$}}}_{\pm 0.00}$ & 98.6$^{\text{\textcolor{red!80}{$\downarrow$}}}_{\pm 0.01}$ & 88.4$^{\text{\textcolor{blue!70}{$\uparrow$}}}_{\pm 0.01}$ & 90.6$^{\text{\textcolor{blue!70}{$\uparrow$}}}_{\pm 0.01}$ & 73.8$^{\text{\textcolor{blue!70}{$\uparrow$}}}_{\pm 0.06}$ & 69.4$^{\text{\textcolor{blue!70}{$\uparrow$}}}_{\pm 0.02}$ & \cellcolor{orange!20}\underline{77.9}\textcolor{blue!70}{$\uparrow$} & \cellcolor{cyan!20}\underline{65.7}\textcolor{blue!70}{$\uparrow$} & Add. $\times$ R \\
        $\mathbf{mh}$ (ours) & \textbf{56.3}$^{\text{\textcolor{blue!70}{$\uparrow$}}}_{\pm 0.09}$ & \textbf{64.9}$^{\text{\textcolor{blue!70}{$\uparrow$}}}_{\pm 0.03}$ & \textbf{98.4}$^{\text{\textcolor{blue!70}{$\uparrow$}}}_{\pm 0.01}$ & \textbf{98.7}$^{\text{\textcolor{blue!70}{$\uparrow$}}}_{\pm 0.01}$ & \textbf{89.8}$^{\text{\textcolor{blue!70}{$\uparrow$}}}_{\pm 0.02}$ & \textbf{92.0}$^{\text{\textcolor{blue!70}{$\uparrow$}}}_{\pm 0.01}$ & \textbf{79.4}$^{\text{\textcolor{blue!70}{$\uparrow$}}}_{\pm 0.02}$ & \textbf{72.0}$^{\text{\textcolor{blue!70}{$\uparrow$}}}_{\pm 0.01}$ & \cellcolor{orange!20}\textbf{81.5}\textcolor{blue!70}{$\uparrow$} & \cellcolor{cyan!20}\textbf{72.1}\textcolor{blue!70}{$\uparrow$} & Both above \\
        \hline
    \end{tabular}
    }
\end{table}

We conduct ablation experiments by dropping the gradient of the Bregman-Moreau envelope, the local loss function, or both, from the personalized strategy $\mathbf{mh}$ as shown in Table~\ref{tbl_res_ablation}.
The relationship among the three strategies mentioned in Eq.~(\ref{equ_def_spmmdl}) is that $\mathbf{mh}$ consists of $\mathbf{lg}$ and $\mathbf{meg}$. Moreover, pFedMe can be regarded in our framework as the one which takes the spherical Gaussian as prior and uses vanilla prior selection strategy $\mu_{i}=I$ without personalization. Thus, pFedMe and the three implementations of pFedBreD are compared.
The results reveal the instability of our implementation $\mathbf{lg}$ and the introduction of $\mathbf{meg}$ on difficult tasks is about the same as not introducing it. However, introducing both $\mathbf{lg}$ and $\mathbf{meg}$ (i.e., $\mathbf{mh}$) together shows remarkable performance. This indicates that $\mathbf{lg}$ and $\mathbf{meg}$ complement each other.
\textbf{To explain these results}, by observing the error bars, in most of the settings, $\mathbf{meg}$ is significantly more stable compared to methods that do not use personalized priors, while $\mathbf{lg}$ is relatively less stable. Based on this observation, we have reason to believe that $\mathbf{meg}$ weakens the influence of potential noise, while $\mathbf{lg}$ introduces new noise. Therefore, we can infer that while the mean parameters are steadily biased towards the personalized model, the introduction of new noise finds a path that is more likely to escape from local optima or saddle points, based on implicit regularization~\cite{pmlr-v139-razin21a, neyshabur2017geometry, neyshabur2017implicit}.

\paragraph{Generalized Coherence Analysis of Information Injection and Extraction}

\begin{figure}[t]
    \centering
    \includegraphics[width=0.40\textwidth, height=0.20\textwidth]{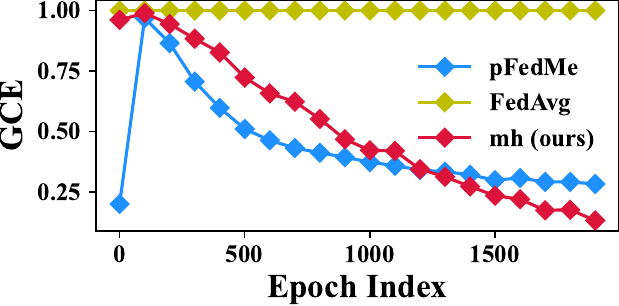}
    \includegraphics[width=0.40\textwidth, height=0.20\textwidth]{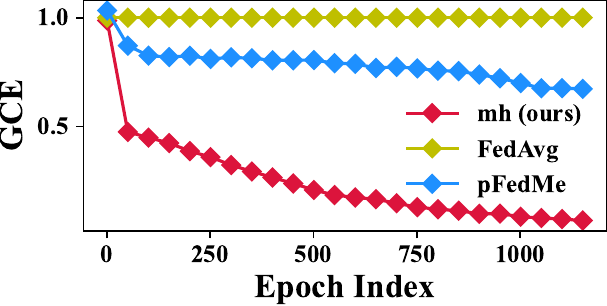}
    \begin{center}
        \footnotesize 
        CIFAR-10 with CNN \qquad\qquad\qquad\qquad\qquad\qquad MNIST with DNN
    \end{center}
    \caption{The results of GCE(\{$\nabla F_{i}(w_{i}^{(t)})$\}) at each global epoch $t$ after Savitzky-Golay filtering~\cite{savitzky_smoothing_1964}.}
    \label{fig_gce}
\end{figure}

The generalized coherence estimate (GCE)~\cite{197218} of vectors from personalized to local model (\textit{i.e.}, the envelope gradients in pFedMe and ours) among clients on each global epoch are shown in Figure~\ref{fig_gce}.
The smaller the GCE, the less coherent the envelope gradient between individual nodes and the greater the diversity of information in the global model update.
As shown in Figure~\ref{fig_gce}, we can observe that during the convergence phase, using a personalized prior method has significantly greater information diversity than not using a personalized prior method, which proves the success of injecting personalized prior knowledge into the global model and extracting local knowledge from the local training.

\paragraph{Variable-Control Analysis of Robustness}

\begin{table}[t]
\begin{minipage}[t]{0.41\textwidth}
    \centering
    \caption{The global test accuracy (\%) of the global model with different numbers of clients for aggregation $S\in\{10,20,50,100\}$.($\spadesuit$:FEMNIST, $\diamondsuit$:FMNIST)}
    \resizebox{.99\textwidth}{!}{
    \begin{tabular}{lccccc}
        \hline
        \textbf{Numbers}   & small & \multicolumn{2}{c}{$\longrightarrow$}& large & \cellcolor{red!30}Std\\
        \hline
        $\spadesuit$-DNN & 59.0 & 60.1 & 60.1 & 59.8 & \cellcolor{red!30}0.5 \\
        $\spadesuit$-MCLR & 54.4 & 55.4 & 55.4 & 55.5 & \cellcolor{red!30}0.5 \\
        $\diamondsuit$-DNN & 75.1 & 79.6 & 79.4 & 79.3 & \cellcolor{red!30}2.2 \\
        $\diamondsuit$-MCLR & 80.0 & 82.6 & 81.8 & 82.7 & \cellcolor{red!30}1.3 \\
        \hline
    \end{tabular}
    }
    \label{tbl_aggregation}
\end{minipage}
\,
\begin{minipage}[t]{0.55\textwidth}
    \centering
    \caption{The local test accuracy (\%) of the personalized model on FMNIST-DNN setting with different data heterogeneity (Non-IID) settings $\alpha \in \{0.01,0.1,1,10,100,1000\}$($\alpha \downarrow$, Non-IID$\uparrow$)~\cite{hsu_measuring_2019}. The \textbf{Bolded} means the best.}
    \resizebox{.99\textwidth}{!}{
    \begin{tabular}{lcccccccc}
        \hline
        Non-IID  & small && \multicolumn{2}{c}{$\longrightarrow$} & & large & \cellcolor{orange!20}Avg \\
        \hline
        FedAvg-\textbf{G} & 18.2 & 14.8 & 14.5 & 11.9 & 11.3  & 11.2 & \cellcolor{orange!20}13.7 \\
        pFedMe-\textbf{P} & 89.5 & 58.2 & 24.2 & 12.3 & 11.8 & 10.6 & \cellcolor{orange!20}34.4 \\
        pFedMe-\textbf{G} & 17.0 & 14.3 & 14.1 & 12.3 & 10.8 & 10.9 & \cellcolor{orange!20}13.2 \\
        $\mathbf{mh}$(ours)-\textbf{P} & \textbf{89.6} & \textbf{58.7} & \textbf{25.2} & \textbf{13.1} & 11.1 & 11.0 & \cellcolor{orange!20}\textbf{34.8} \\
        $\mathbf{mh}$(ours)-\textbf{G} & 17.1 & 14.6 & 14.6 & 12.4 & \textbf{11.9} & \textbf{11.9} & \cellcolor{orange!20}13.8 \\
        \hline
    \end{tabular}
    }
    \label{tbl_noniid}
\end{minipage}
\end{table}

We analyze the impact of aggregation noise and data heterogeneity~\cite{hsu_measuring_2019} on our method, mainly $\mathbf{mh}$, by controlling variables. Results are in Table~\ref{tbl_aggregation} and Table~\ref{tbl_noniid}. (Details are in Appendix~\ref{appdx_ssec_instability}.)
We test the performance of global model on different aggregation ratios, where all hyper-parameters except for the aggregation ratios are fixed. Meanwhile, we test the performance of both global and the personalized model on different data heterogeneity settings, where full aggregation (sample client equals total number of clients, $S=N$) and one-step local update (local epoch $R=1$) are employed to get rid of the effects of aggregation noise and client drift.
The experiments demonstrate the instability of the global model in $\mathbf{mh}$ at small aggregation ratios, which most of the other PFL methods have, by comparing their performances on different aggregation numbers. Comparing to the baselines, the experiments also demonstrate the relative robustness of our method to extreme data heterogeneity.

\paragraph{Deviation Analysis of Personalization}
\label{ssec_personalization}
\begin{figure}[t]
    \centering
    \includegraphics[width=0.24\textwidth, height=0.213\textwidth]{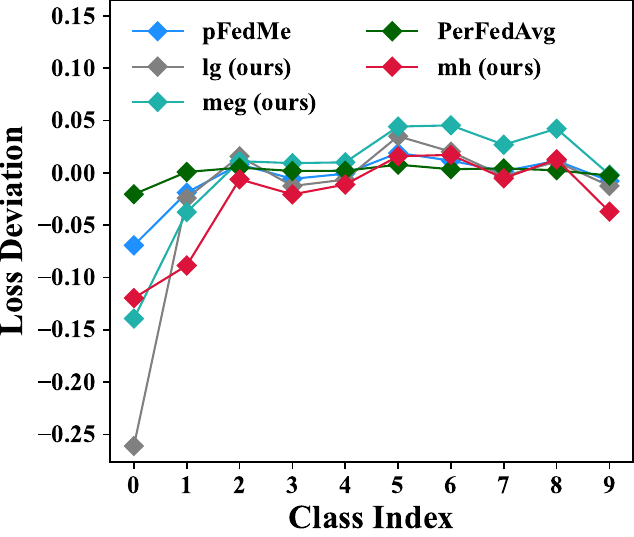}
    \includegraphics[width=0.24\textwidth, height=0.213\textwidth]{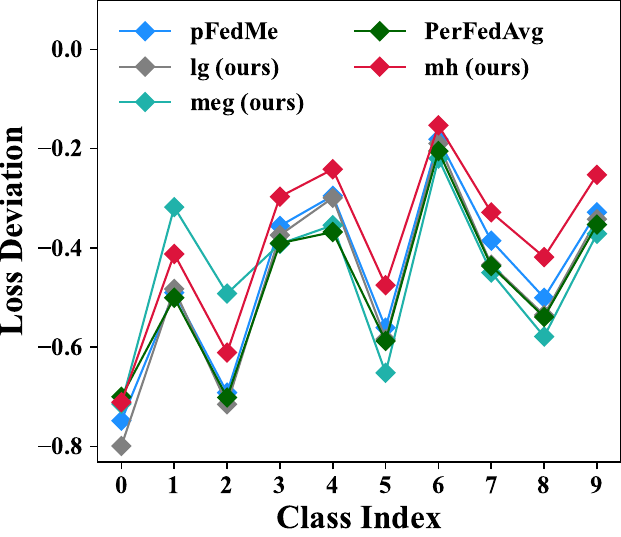}
    \includegraphics[width=0.24\textwidth, height=0.213\textwidth]{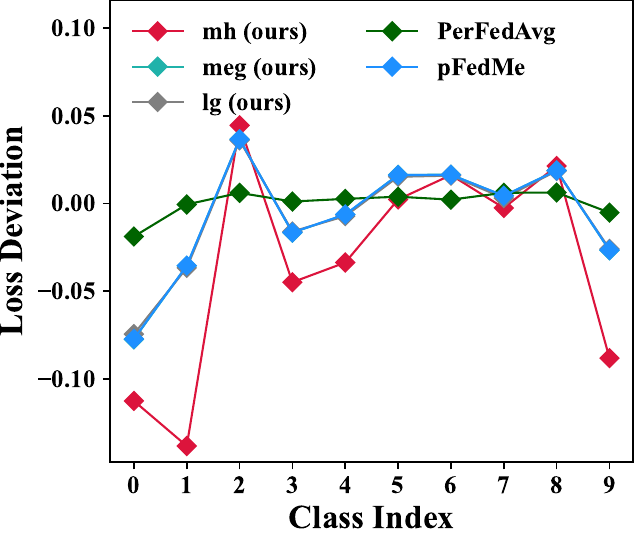}
    \includegraphics[width=0.24\textwidth, height=0.213\textwidth]{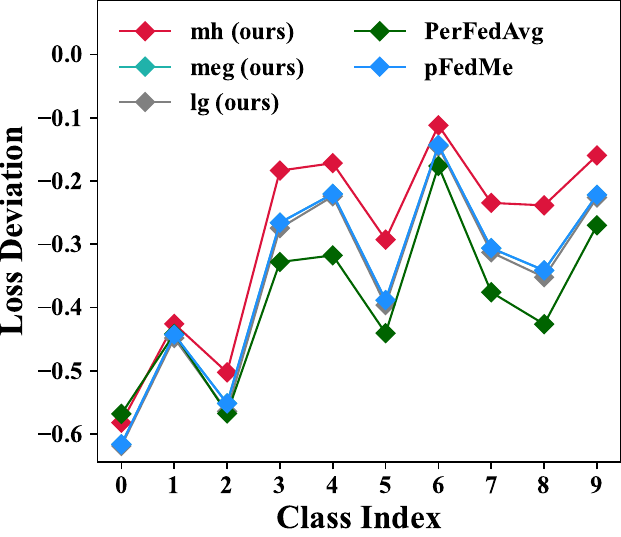}
    \begin{center}
        \footnotesize 
        MNIST-MCLR-G \quad\qquad MNIST-MCLR-L \quad\qquad
        MNIST-DNN-G \quad\qquad MNIST-DNN-L
    \end{center}
    \caption{The loss deviation of experiments in Section~\ref{sec_exp} on the first client, whose major data are on $0^{th}$ classes. The lower deviation of the available class on global tests and the higher deviation of the unavailable class on local tests demonstrate the superior personalization ability of our methods.}
    \label{fig_deviation}
\end{figure}

Deviation represents the difference between an individual and the mean value. We use the deviation of the loss function to reflect the personalization.
\textbf{On global test}, the lower the deviation, the better the personalized model performance on the corresponding local data.
\textbf{On local test}, the model is only tested on its own dataset, and because of multiple local iterations, the local test deviation converges to almost the same value, as shown in MNIST-MCLR-L and MNIST-DNN-L in Figure~\ref{fig_deviation}. Furthermore, since the local test has a loss of 0 on missing classes, a higher deviation on missing classes reflects a lower mean on these classes. Thus, the lower loss in local testing and better performance are reflected from both of the almost equal deviation in local testing and the higher deviation on missing class.
\textbf{Summary:} based on Figure~\ref{fig_deviation}, we can see that our method has higher deviation on missing classes in local testing and lower deviation in global testing. This means that our method has better personalized performance.

\section{Conclusion and Discussion}
\label{sec_cncls}

\paragraph{Conclusion}
To address the issue of neglecting client-sampling information while providing prior knowledge to local training via direct use of a global model, we propose a general concept: the personalized prior. In this paper, we propose a general framework, pFedBreD, for exploring PFL strategies under the SX-family prior assumption and computation, the RMD to explicitly extract the prior information, and three optional meta-step strategies to personalize the prior. We analyze our proposal both theoretically and empirically. Our strategy $\mathbf{mh}$ shows remarkable improvement in personalization and robustness to data heterogeneity on non-i.i.d. datasets
and the LEAF benchmark
~\cite{caldas2018leaf} with MCLR / DNN / CNN / LSTM as inferential model, which conduct convex / non-convex problems, and image / language benchmarks.

\paragraph{Limitations and Future Work}
Although $\mathbf{mh}$ shows remarkable performance and robustness, there is still instability in the global model with aggregation noise. 
Furthermore,
it should be noted that the superficial reason for the improvement of $\mathbf{mh}$ seems to be that $\eta_{\alpha}$ and $\eta$ and (which are similar to each other) are used simultaneously,  resulting in a magnitude in $\mathbf{mh}$ that is twice as large as the ones in the other two implementations and leading to better performance. 
However, empirically, simply doubling $\eta_{\alpha}$ in $\mathbf{lg}$ or $\eta$ in $\mathbf{meg}$ does not improve performance, and using one more $\mathbf{meg}$ step used in $\mathbf{lg}$ significant improvement. Our theoretical analysis cannot explain this phenomenon, and more detailed modeling is needed.

\newpage

\section*{Acknowledge}
This work is supported by the Key Program of National Science Foundation of China (Grant No. 61836006) from College of Computer Science (Sichuan University) and Engineering Research Center of Machine Learning and Industry Intelligence (Ministry of Education), Chengdu 610065, P. R. China.
This research is also supported by the National Research Foundation, Singapore under its AI Singapore Programme (AISG Award No: AISG2-PhD-2021-08-008).

\bibliography{pFedBreD.bib}
\bibliographystyle{plain}

\newpage
\appendix
\onecolumn

\section{Glossary, Some Basic Knowledge and Details about Implementations}

\subsection{Glossary}
\label{appdx_glsry}

The main notations in this paper are shown in Table~\ref{appdx_tbl_glsry}.

\begin{table}[ht]
    \centering
    \caption{The glossary of notations mentioned in this paper}
    \resizebox{.96\textwidth}{!}{
    \begin{tabular}{cc}
        \hline
        Notation & Implication \\
        \hline
        $\cdot_{i}$ & $\cdot$ on $i^{th}$ client \\
        $f_{i}$ & local loss function \\
        $F_{i}$ & local objective function \\
        $F$ & global objective function \\
        $\mathbf{E}_{\cdot}$ & expectation on $\cdot$ \\
        $\mathbf{H}$ & entropy \\
        $\mathbf{P}$ & probability measure \\
        $\mathbf{P}_{ef}$ & probability in exponential family\\
        $\mathbf{P}_{sef}$ & probability in scaled exponential family \\
        $\Omega_{\cdot}$ & complete set of $\cdot$ \\
        $\mathcal{F}$ & generic $\sigma$-algebra\\
        $\sigma(\cdot)$ & $\sigma$-algebra derived from $\cdot$ \\
        $\{\Omega,\mathcal{F}\}$ & measurable space \\
        $\{\Omega,\sigma(\Omega),\mathbf{P}\}$ & probability measurable space \\
        $\hat{\mathbf{P}}$ & estimated probability \\
        $x_{i}$,$y_{i}, d_{i}$ & input data, label data, the pairs of them \\
        $w$ & global model parameter \\
        $\Theta_{i}$ & local information \\
        $\theta_{i}$ & personalized parameters\\
        $w_{init},\theta_{init}$ & function to initialize parameters \\
        $\mu_{i}$ & the function to generate mean parameter \\
        $s_{i}$ & the function to generate natural parameter \\
        $x$,$y$,$\hat{w}$,$\theta$,$\mu$,$s$ & generic point notations \\
        T,N,R,S & number of total global epochs, clients, local epochs, number of sampling clients \\
        t,r & global epochs, local epochs \\
        $\beta$,$\eta$,$\eta_{\alpha}$,$\lambda$,$\hat{\lambda}$ & scalar notations \\
        $g,h,h_{\lambda}$ & generic function notations \\
        $\mathcal{D}_{g}$ & Bregman divergence derived from g\\
        $\mathcal{D}\mathbf{prox}$ & Bregman divergence proximal mapping \\
        $\mathcal{D}\mathbf{env}$ & Bregman Moreau envelope \\
        $\nabla, \mathbf{D}, \nabla^{2}$ & gradient, Jocobian and Hessian operator \\
        $\Delta$ & deviation from mean \\
        $\cdot^{*}$ & the Fenchel conjugate of $\cdot$ \\
        $\mathcal{L}$ & averaged local test loss \\
        $\mathcal{G}$ & averaged global test loss \\
        $\bar{\cdot}$ & mean of $\cdot$ over clients \\   
        $I$,$I_{m}$ & identity mapping, identity matrix\\
        \hline
    \end{tabular}}
    \label{appdx_tbl_glsry}
\end{table}

\subsection{Bregman Divergence}

Bregman divergence is a general distance satisfying that its first-order moment estimation is the point that minimizes the expectation of the distance to all points for all measurable functions on $\mathbf{R}^{d}$. In other words, the given distance $\mathcal{D}$ satisfies Condition (~\ref{bd_prpt}):
\begin{equation}
\label{bd_prpt}
 \forall X \in \{\mathbf{R}^{d}, \mathcal{F}, \mathbf{P}\}, \mathbf{E}[X]={\arg\min}_{y}\mathbf{E}[\mathcal{D}(X,y)   ]
\end{equation}

Eq.~(\ref{def_bd}) is the definition of Bregman divergence:
\begin{equation}
\label{def_bd}
\begin{aligned}
    \mathcal{D}_{g}(x,y)&:=g(x)-g(y)-\langle\nabla g(y),x-y\rangle \\ 
    &=\int_{y}^{x} \nabla g(t)-\nabla g(y) dt
\end{aligned}
\end{equation}
where $g$ is a convex function. For convenience, in this paper, $g$ is assumed to be strictly convex, proper and differentiable, so that the equation above Eq.~(\ref{def_bd}) are well-defined. In the perspective of Taylor expansion, Bregman divergence is the first-order residual of $g$ expanded at point $y$ valued at point $x$, which is the natural connection between Bregman divergence and Legendre transformation. The Bregman divergence does not satisfy the distance axiom, but it provides some of the properties we need, such as non-negative distance. Hence, the selected function $g$ should be convex. Furthermore, if one wants the distance to have a good property that $x=y\leftrightarrow\mathcal{D}_{g}(x,y)=0$, one needs $g$ to be strictly convex.

\subsection{Non-Maximum Entropy}
Besides, the non-maximum entropy rule approach is also worth considering, but we focus on maximum entropy prior in this section. See~\cite{seidenfeld1979not, genest1986combining, kass1996selection, gelman2013philosophy} for additional information of non-maximum entropy assumptions.

\subsection{Future PFL}
Besides the FTML, Bayesian learning, EM, and transfer learning mentioned in the main paper, neural-collapse-motivated methods and life-long learning are also promising methods to handle PFL problem~\cite{li2023fear, huang2023neural, zhang2022spatialtemporal}.

PFL could also fucos on personalizing other characteristics about FL system, e.g., communications, resource-constrained device. For example, this paper ~\cite{zhou2023communication} gives a data distillation (compression)~\cite{zhang2022expanding, yang2023does} method to reduce communication cost, and the compressed data itself contains personalize posterior information.

\subsection{Personalized Prior and MAML}

Based on previous derivations, to obtain a deployable algorithm, our remaining task is to determine $\Phi$. In this section, inspired by MAML, we briefly introduce a meta-step-based implementation method.
The mean parameter is used to represent the prior under SX-family prior assumption given any $\lambda$ and $g$ in this paper. The mean of the SX-family prior in Eq.~(\ref{equ_gfl_sembme}) is used in regular term, which can be personalized in each client $i$ as $\mu_{i}$, corresponding to $\mu_{\Phi}$ in Eq.~(\ref{equ_gfl_mdlb}), as shown in Figure~\ref{fig_intuition}. Motivated by this, we use MAML to learn the personalized regularization (or personalized prior in Bayesian learning) in Section~\ref{sec_alg}. For example, $\mathbf{meg}$ in Eq.~(\ref{equ_def_spmmdl}) uses MAML on the Bregman-Moreau envelope $\mathcal{D}\mathbf{env}_{g^{*},\lambda^{-1}}f_{i}$ by substituting it into $J$ in Section~\ref{sec_rltw} and $\Phi$ in Eq.~(\ref{equ_gfl_mdlb}).

\subsection{Sampling Method in Bayesian Learning}

Bayesian methods are a elegant solution to the complex issue of heterogeneous data, as they operate on a principle whereby the model allocates increasing attention to local data as available, and derives insight from prior information when local information is scarce. Furthermore, Bayesian modeling brings fresh probabilistic insights to PFL regularization techniques, while simultaneously providing a flexible framework for exploring novel strategies. Bayesian modeling, as well as the expectation maximuzatioin and maximum a posteriori estimate (EM-MAP)~\cite{dempster1977maximum}, provide our personalized prior approach with straightforward theoretical support, as well as more general perspectives for analysis. Meanwhile, it addresses the cost of additional sampling in the classic and approximate Bayesian learning paradigm with MAP, the regularization method.

In Bayesian modeling, the EP global loss provides more information that we want to use for local training due to its zero-avoiding property.~\cite{Minka2001ExpectationPF}

The sampling methods used to calculate the solution of Bayesian Model mentioned in this paper can be importance sampling, MCMC or others. In this work, we use the approximation Bayesian methods. See more details in~\cite{andrieu2003introduction}. The local training process based on regular terms differs from Bayesian learning based on sampling, \textit{i.e.}, each time a model needs to be obtained by sampling the model distribution under the current parameters. We choose to use Bayesian MAP as a point estimation as our estimation method, thus eliminating steps such as sampling and reparameterization to improve inference efficiency. The personalized model sampled from local training can be seen as the results from random data sampling using SGD or the mean parameter directly.

\subsection{First-order Methods}
\label{appdx_ssec_foMtd}
There are three parts in Eq.~(\ref{equ_gdmp}) we need to deal with, and the first-order methods are as shown below:

\textbf{Jacobian Matrix of Mean:} specifically, utilizing the prior selection strategy discussed in Section~\ref{sec_mth}, we have $\mathbf{D}\mu_{i}(w)=I-\eta\nabla^{2} \Phi(w)$. Using different $\Phi$  functions yields varying results. For instance, with first-order methods and the last term removed, we get the approximation $\mathbf{D}\mu_{i}\leftarrow I$.

\textbf{Hessian Matrix:} with first-order methods, we let $\nabla^{2}g^{*}(\cdot)=I_{m}$. It happens when assuming $\theta_{i}$ obeys the spherical Gaussian by letting $g=\frac{1}{2}||\cdot||^{2}$. Moreover, we can assume $\theta_{i}$ obeys the general multivariable Gaussian by letting $g=\langle\cdot,\Sigma^{-1}\cdot\rangle$ and $\nabla^{2}g(\cdot) = \Sigma^{-1} \succeq 0$.

\textbf{Proximal Mapping:} given $\mu_{i}(w)$,  the proximal mapping part $\mathcal{D}\mathbf{prox}_{g^{*},\lambda^{-1}}f_{i}(\mu_{i}(w))$ can be approximately solved with numerical methods, \textit{e.g.}, gradient descent methods. In other words, we can alternately calculate $\mu_{i}(w)$ on each client and then fix $\mu_{i}(w)$ in each local epoch with EM.

\subsection{Complexity}

Since the general process of our implementations, FedAMP and pFedMe are the same as shown in pFedBreD framework, these methods share the same complexity of memory/calculation, $O(N)/O(NTRK)$ as shown in Table~\ref{tbl_complexity}.
The complexities of both FedAvg and Per-FedAvg are $O(N)/O(NTR)$ since the original methods of them do not need a approximate proximal mapping solution, and therefore are free on $K$, the number of iterations to calculate the solution.
The complexity of FedEM is $O(NM)/O(NTRM)$, where $M$ is the components of the distributions we assume, due to the calculation of $M$ components in each global epoch.

\begin{table}[t]
    \centering
    \caption{Complexity Comparison}
    \label{tbl_complexity}
    \resizebox{.96\textwidth}{!}{
    \begin{tabular}{lcccccc}
    \hline
    Complexity/Methods &	FedEM &	FedAvg	& pFedMe &	Per-FedAvg &	FedAMP &	pFedBreD (ours) \\
    \hline
Sys. Memory & $\mathbf{O}(NM)$ & $\mathbf{O}(N)$	& $\mathbf{O}(N)$	& $\mathbf{O}(N)$	& $\mathbf{O}(N)$	& $\mathbf{O}(N)$	\\
Sys. Time	& $\mathbf{O}(NTRM)$	& $\mathbf{O}(NTR)$	& $\mathbf{O}
(NTRK)$	& $\mathbf{O}(NTR)$	& $\mathbf{O}(NTRK)$	& $\mathbf{O}(NTRK)$ \\
\hline
    \end{tabular}
    }
\end{table}

\subsection{Broader Impacts}

In recent years, PFL has found use not only in predictive tasks like mobile device input methods but also in areas where privacy is paramount, such as healthcare and finance. However, before its widespread deployment, several critical factors must be taken into consideration.

One of the primary concerns regarding PFL is its deployment cost. It involves significant computational resources, making it a costly affair. Additionally, client transparency is an important issue that needs attention. Clients have the right to know what data is being collected and how it is used.

Another factor that complicates PFL's deployment is the differences in user behavior and hardware and software configurations between clients. These differences can affect the performance of the algorithm and require bespoke solutions for each client.

In addition, PFL's robustness is another essential aspect to consider. Real-world environments are often unpredictable and can interfere with the algorithm's performance, leading to erroneous results. Therefore, it is necessary to ensure that the algorithm is sufficiently robust before deploying it.

Lastly, even though PFL offers significant benefits, potential drawbacks should not be overlooked. All stakeholders involved in its deployment need to approach this technology with caution and forethought. By considering these factors, we can harness the power of PFL while minimizing its limitations and risks.

\section{Details of Equations}
\label{appdx_sec_doe}
\subsection{Hidden Information}
\label{appdx_sec_hi}

From the definition of KL divergence, we have 
\begin{equation}
\begin{aligned}
&\arg\min_{w}\mathbf{E}_{i}\mathbf{E}_{d_{i}}\mathbf{KL}(\mathbf{P}(y_{i}|x_{i})||\hat{\mathbf{P}}(y_{i}|x_{i},w)) \\
=&\arg\min_{w}\mathbf{E}\log \mathbf{P}(y_{i}|x_{i}) - \log \hat{\mathbf{P}}(y_{i}|x_{i},w)) \\
=&\arg\min_{w}\mathbf{E} - \log \hat{\mathbf{P}}(y_{i}|x_{i},w)) \\
=&\arg\max_{w}\mathbf{E} \log \hat{\mathbf{P}}(y_{i}|x_{i},w)) \\
\end{aligned}
\end{equation}
This is used in Eq.~(\ref{equ_gfl}) in the main paper.

\subsection{Bregman Divergence and X-Family}
\label{appdx_ssec_bdxf}
We use the SX-family due to its computational advantages. While other families of distributions may be able to handle special cases, they may not be as computationally efficient.

If proper and strictly convex function $g$ is differentiable, with $g^{*}$ the Fenchel conjugate function of $g$, $\mathcal{D}_{g}(x,y)$ the Bregman divergence, $\mu$ dual point of $s$, we have:
\begin{equation}
    \begin{aligned}
        \mathcal{D}_{g^{*}}(\mathcal{V},\mu)=g^{*}(\mathcal{V})+g(s)- \langle \mathcal{V}, s \rangle=\mathcal{D}_{g}[s, \nabla g^{*}(\mathcal{V})]
    \end{aligned}
\end{equation}

From the definition of Bregman divergence , $\nabla g(s) = \mu$ and definition of $g^{*}$ Fenchel conjugate on convex function $g$ ,we have:
\begin{equation}
    \begin{aligned}
        \mathcal{D}_{g^{*}} (\mathcal{V}, \mu)
        &= g^{*}(\mathcal{V}) - g^{*}(\mu) - \langle \nabla g^{*}(\mu), \mathcal{V} - \mu \rangle \\
        &= g^{*}(\mathcal{V}) - g^{*}(\mu) - \langle s, \mathcal{V} - \mu \rangle \\
        &= g^{*}(\mathcal{V}) - \langle s, \mathcal{V} \rangle - g^{*}(\mu) + \langle \mu, s \rangle \\
        &= g^{*}(\mathcal{V}) - \langle s, \mathcal{V} \rangle + g(s) \\
    \end{aligned}
\end{equation}
Similarly, we have $\mathcal{D}_{g}[s, \nabla g^{*}(\mathcal{V})] = g
^{*}(\mathcal{V}) - \langle s, \mathcal{V} \rangle + g(s)$. This property is used in Eq.~(\ref{equ_def_ef}) and (~\ref{equ_def_sef}) in the main paper.

\begin{table}[ht]
    \centering
    \caption{Bregman divergence and exponential family. (note $\xi=\langle\cdot ,\ln\cdot\rangle$)}
    \resizebox{.96\textwidth}{!}{
    \begin{tabular}{ccccc}
        \toprule
         Name & Gaussian & Bernoulli & Possion & Exponential \\
         \midrule
          Domain& $\mathbf{R}^{d}$ & $\{0,1\}$ & $\mathbf{N}$ & $\mathbf{R}_{++}$\\
          $g(y)$ & $\frac{1}{2}||y||_{\Sigma^{-1}}^{2}$ & $\ln(1+e^{y})$ & $e^{y}$ & $-\ln(-y)$\\
          $\nabla g(y)$ & $y$ & $\frac{\exp\{y\}}{1 + \exp\{y\}}$ & $e^{y}$ & $-y^{-1}$\\
          $g^{*}(x)$ & $\frac{1}{2}||x||_{\Sigma^{-1}}^{2}$ & $\xi(x)+\xi(1-x)$ & $x\ln(x)-x$ & $-\ln(x)-1$ \\
          $\nabla g^{*}(x)$ & $x$ & $\ln(\frac{x}{1-x})$ & $\ln(x)$ & $-x^{-1}$ \\
          $\mathcal{D}_{g^{*}}(x,y)$ & $\frac{1}{2}||x-y||_{\Sigma^{-1}}^{2}$& $\ln(1+e^{(1-2x)y})$ & $e^{y}+\xi(x)-x(y+1)$ & $\frac{x}{y}-\ln\frac{x}{y}-1$\\
         \bottomrule
    \end{tabular}
    }
    \label{tbl_bd_ef}
\end{table}

Table~\ref{tbl_bd_ef} shows parts of the relationship between specific $g$ and related member in exponential family. See~\cite{banerjee2005clustering} for more about the relationships between $g$ that derives Bregman divergence $\mathcal{D}_{g}$ and related derived divergence (\textit{e.g.}, $\cdot\Sigma^{-1}\cdot$ \& Mahalanobis distance, $\sum_{\cdot} \cdot \log \cdot$ \& KL divergence / generalized I-divergence and etc.).

\subsection{Expectation Maximization}
\label{appdx_ssec_em}
The details of Eq.~(\ref{equ_gfl_eml}) in the main paper is shown in Eq.~(\ref{appdx_equ_gfl_eml}).
\begin{equation}
\label{appdx_equ_gfl_eml}
\begin{aligned}
& \sum_{i}\log\mathbf{P}(y_{i}|x_{i},w) = \sum_{i}\log\int\mathbf{P}(y_{i},\Theta_{i}|x_{i},w)d\Theta_{i}
= \sum_{i}\int\mathbf{Q}(\Theta_{i})\log\frac{\mathbf{P}(y_{i},\Theta_{i}|x_{i},w)}{\mathbf{Q}(\Theta_{i})}d\Theta_{i} \\
\ge& \sum_{i}\int\log\mathbf{Q}(\Theta_{i})\frac{\mathbf{P}(y_{i},\Theta_{i}|x_{i},w)}{\mathbf{Q}(\Theta_{i})}d\Theta_{i}
= \sum_{i}\mathbf{E}_{\mathbf{Q}(\Theta_{i})}\log\frac{\mathbf{P}(y_{i},\Theta_{i}|x_{i},w)}{\mathbf{Q}(\Theta_{i})} \\
=& \sum_{i}\mathbf{E}_{\mathbf{Q}(\Theta_{i})}\log\mathbf{P}(y_{i},\Theta_{i}|x_{i},w)-\log\mathbf{Q}(\Theta_{i}) \\
\ge& \sum_{i}\mathbf{E}_{\mathbf{Q}(\Theta_{i})}\log\mathbf{P}(y_{i},\Theta_{i}|x_{i},w)
= \sum_{i}\mathbf{E}_{\mathbf{Q}(\Theta_{i})}[\log\hat{\mathbf{P}}(y_{i}|x_{i},\Theta_{i},w)+\log\mathbf{P}(\Theta_{i}|x_{i},w)
] \\
=& \sum_{i}\mathbf{E}_{\mathbf{Q}(\Theta_{i})}[\log\hat{\mathbf{P}}(y_{i}|x_{i},\Theta_{i},w)+\log\int_{y_{i}}\mathbf{P}(\Theta_{i}|d_{i},w)\mathbf{P}(y_{i}|x_{i},w)
] \\
\ge& \sum_{i}\mathbf{E}_{\mathbf{Q}(\Theta_{i})}[\log\hat{\mathbf{P}}(y_{i}|x_{i},\Theta_{i},w)+\mathbf{E}_{y_{i}|x_{i},w}\log\mathbf{P}(\Theta_{i}|d_{i},w)
]
\end{aligned}
\end{equation}
In Eq.~(\ref{appdx_equ_gfl_eml}), we use the concavity of logarithmic function for the first inequality and entropy $\mathbf{H}(\mathbf{Q}(\Theta_{i}))=\mathbf{E}_{\mathbf{Q}(\Theta_{i})}-\log\mathbf{Q}(\Theta_{i}) \ge 0$ the for the second. (probability $\mathbf{Q}(\Theta_{i})\in[0,1]$; The first equal sign holds, when $\mathbf{Q}(\Theta_{i})=\mathbf{P}(\Theta_{i}|d_{i},w)$.) The last inequality is derived from the concavity of the logarithmic function.

\textbf{Why is a-posteriori distribution a prior in this modeling and problem formulation? What about $\hat{\lambda}$?} We assume $\Theta_{i}|d_{i},w \sim \hat{\mathbf{P}}_{sef}(\Theta_{i};\lambda,s_{i}(w;d_{i}),g)$, and have:
\begin{equation}
    \begin{aligned}
        &\mathbf{E}_{\mathbf{Q}(\Theta_{i})}[\log\hat{\mathbf{P}}(y_{i}|x_{i},\Theta_{i},w)+\mathbf{E}_{y_{i}|x_{i},w}\log \mathbf{P}(\Theta_{i}|d_{i},w)] \\
        =& \mathbf{E}_{\mathbf{Q}(\Theta_{i})}\log\hat{\mathbf{P}}(y_{i}|x_{i},\Theta_{i},w) \\ 
        & +\mathbf{E}_{\mathbf{Q} (\Theta_{i})}\mathbf{E}_{y_{i}|x_{i},w}[\log \mathbf{P}(\Theta_{i}|x_{i},w)+\log \mathbf{P}(y_{i}|\Theta_{i},x_{i},w)-\log \mathbf{P}(y_{i}|x_{i},w)] \\
    \end{aligned}
\end{equation}
Optimization local problem taken on both side in any $\mathbf{Q}$ sampling, we have:
\begin{equation}
    \begin{aligned}
    &\arg\min_{\Theta_{i}}\{\log\hat{\mathbf{P}}(y_{i}|x_{i},\Theta_{i},w)+\mathbf{E}_{y_{i}|x_{i},w}\log \mathbf{P}(\Theta_{i}|d_{i},w) \}\\
        =& \arg\min_{\Theta_{i}}\{\log\hat{\mathbf{P}}(y_{i}|x_{i},\Theta_{i},w) \\ 
        & +\mathbf{E}_{y_{i}|x_{i},w}[\log \mathbf{P}(\Theta_{i}|x_{i},w)+\log \mathbf{P}(y_{i}|\Theta_{i},x_{i},w)-\log \mathbf{P}(y_{i}|x_{i},w)]\} \\
        =& \arg\min_{\Theta_{i}}\underbrace{\{\log\hat{\mathbf{P}}(y_{i}|x_{i},\Theta_{i},w)}_{\textbf{Predicted Likelihood}} +\mathbf{E}_{y_{i}|x_{i},w}[\underbrace{\log \mathbf{P}(\Theta_{i}|x_{i},w)}_{\textbf{Prior Distribution}}+\underbrace{\log \mathbf{P}(y_{i}|\Theta_{i},x_{i},w)}_{\textbf{Assumed Likelihood}}]\}\\
    \end{aligned}
\end{equation}
Thus, we do maximum a-posteriori estimation alongside added predicted likelihood, which is virtually doing assumptions on prior distribution and take mixed likelihood. Moreover, taking assumption on a-posteriori distribution leads calculation efficiency. Note that the hyperparameters should be carefully discussed.

\textbf{Bi-level optimization trick:}
\begin{equation}
\label{appdx_max_prpty}
\begin{aligned}
    \max_{x,y} f(x,y) &\ge \max_{x}\max_{y} f(x,y)\\
    \sum_{i}a_{i}\max f(x,y_{i}) &= \max\sum_{i}a_{i} f(x,y_{i})
\end{aligned}
\end{equation}

In Eq.~(\ref{equ_gfl_sembme}), we use the two properties of $\max$ shown in Eq.~(\ref{appdx_max_prpty}). Moreover, these properties are also used to build the upper bound of Eq.~(\ref{equ_gfl_md}) as Eq.~(\ref{equ_gfl_mdlb}).

\subsection{Notations of Deviations}
\label{appdx_ssec_dev}
The notations are shown as follows:

$\mathcal{L}_{i,c}$: The averaged local test loss of the $i^{th}$ personalized model over its own local test with label $c$. The value equals zero on the clients without $c$-labeled data.

$\bar{\mathcal{L}}_{c}$: The mean of the averaged local test loss over all personalized models. Each $\mathcal{L}_{i,c}$ is weighted by the ratio of the number of own test data with label $c$.

$\mathcal{G}_{i,c}$: The averaged global test loss of the $i^{th}$ personalized model over the global test with label $c$.

$\bar{\mathcal{G}}_{c}$: The mean of the averaged global test loss over all personalized models.

The deviations of the averaged global and local test loss of the $i^{th}$ personalized model on class $c$: $\Delta \mathcal{G}_{i,c} = \mathcal{G}_{i,c} - \bar{\mathcal{G}_{c}}$ and $\Delta \mathcal{L}_{i,c} = \mathcal{L}_{i,c} - \bar{\mathcal{L}_{c}}$.

\section{More About Experiments}
\label{appdx_sec_exp}

The access of all data and code is available
\footnote{\url{https://github.com/BDeMo/pFedBreD_public}}
.

\subsection{More about implementations}
\label{appdx_ssec_maI}
The three implementations of $\mu_{i}$, \textit{i.e.} $\mathbf{lg}$, $\mathbf{meg}$ and $\mathbf{mh}$, represent \textit{loss gradient}, \textit{memorized envelope gradient} and \textit{memorized hybrid} respectively. \textit{Memorized} means that we choose the gradient of Bregman-Moreau envelope  $\nabla F_{i}(w_{i,r-1}^{(t)})$ as $\eta[w_{i,R}^{(t-1)} - \theta_{i,r-1}^{(t)}]$, where $\eta \ge 0$ is a step-size-like hyper-parameter. Each local client memorizes their own local part of the latest global model $w^{(t)}$ at the last global epochs $w_{i,R}^{(t-1)}$, instead of $w_{i,r-1}^{(t)}$ in practice.

\subsection{Variant}
\label{appdx_ssec_v}

Based on the facts, the results in Table~\ref{tbl_res_lnlm} shows the instability of our personalized models. Here we propose a variant of $\mathbf{mh}$, shown in Eq.~(\ref{appdx_equ_v}), trying to improve the robustness of personalized model on the original $\mathbf{mh}$, which use $\Phi_{i}\leftarrow f_{i}+F_{i}$.
\begin{equation}
\label{appdx_equ_v}
\begin{aligned}
    \Phi_{i} & \leftarrow \tilde{F}_{i,\tilde{\eta}_{\alpha},\tilde{\eta}}:=\tilde{\eta}_{\alpha}f_{i}\circ(\cdot-\tilde{\eta}\nabla f_{i})+F_{i} \\
    \mu_{i,r} & \leftarrow w_{i,r-1}^{(t)} - \eta \nabla \tilde{F}_{i,\tilde{\eta}_{\alpha},\tilde{\eta}}(w_{i,r-1}^{(t)}) \\
    & = w_{i,r-1}^{(t)} - \eta\{ \tilde{\eta}_{\alpha} \nabla f_{i}[w_{i,r-1}^{(t)} - \tilde{\eta} \nabla f_{i}(w_{i,r-1}^{(t)})]\} - \eta \{w_{i,R}^{(t-1)}-\theta_{i,r-1}^{(t)}\}\\
\end{aligned}
\end{equation}

This method in Eq.~(\ref{appdx_equ_v}) performance almost the same as the orginal $\mathbf{mh}$ when $\eta_{\alpha}$ is small, but it provides flexibility to tune the hyper-parameter and decide whether to focus more on the current gradient step or the meta-gradient step by tuning $\tilde{\eta}_{\alpha}$ and $\tilde{\eta}$. $\tilde{\eta}_{\alpha}\leftarrow \eta_{\alpha}/\eta$ and $\tilde{\eta}\leftarrow \eta_{\alpha}$ are used in practice.

\subsection{Implementations of Per-FedAvg}

We implement Per-FedAvg with the first-order method~\cite{fallah2020personalized} and fine-tune the personalized model twice, with each learning step of the global and personalized step sizes.

\subsection{Details of Tricks, Datasets and Models}

Tricks are shown as follows:

\textbf{FT:} fine-tuning single personalized model one more step for local test.

\textbf{AM:} aggregate momentum, the same trick used in $12^{th}$ line of Algorithm~\ref{alg_pFedBreD}.(To compare more fairly between methods with single global model; $\beta=2$ for methods and employing \text{AM})

Datasets settings are shown as follows:

\textbf{CIFAR-10:} the whole dataset is separated into 20 clients, and each client has data of 3 classes of label.~\cite{dinh2021fedu, krizhevsky2009learning}

\textbf{FEMNIST:} we use non-i.i.d. FEMNIST from LEAF benchmark with fraction of data to sample of 5\% and fraction of training data of 90\%.~\cite{caldas2018leaf}

\textbf{FMNIST:} the whole fashion-MNIST dataset is separated into 100 clients, and each client has data of 2 classes of label.~\cite{t2020personalized, xiao2017fashion}

\textbf{MNIST:} the whole MNIST dataset is separated into 100 clients, and each client has data of 3 classes of label.~\cite{t2020personalized, lecun1998gradient}

\textbf{Sent140:} we use non-i.i.d text dataset Sent140 from LEAF benchmark with fraction of data to sample of 5\%, fraction of training data of 90\% and minimum number of samples per user of 3. Then we re-separate Sent140 into 10 clients with at least 10 samples.~\cite{caldas2018leaf}

Model settings are shown as follows:

\textbf{CNN:} for the image data, we use convolutional neural network of CifarNet~\cite{hosang2015taking}.

\textbf{DNN:} the non-linear model is 2 layers deep neural network with 100-dimension hidden layer and activation of leaky ReLU~\cite{maas2013rectifier} and output of softmax.

\textbf{MCLR:} the linear model, multi-class linear regression, is 1 layer of linear mapping with bias, and then output with softmax.

\textbf{LSTM:} text data model consists of 2 LSTM layers~\cite{hochreiter1997long} as feature extraction layer of 50-dimension embeding and hidden layer and 2 layers deep neural network as classifier with 100-dimension of hidden layer.

\subsection{Non-I.I.D Distribution}

Figure~\ref{appdx_fig_niid} shows the non-i.i.d. distribution of MNIST, CIFAR-10, FMNIST, FEMNIST and Sent140. Sent140 is a bi-level classification so each client has two class of label data and we directly use the LEAF benchmark~\cite{caldas2018leaf} and Dirichlet distribution of $\alpha=0.5$ to separate users into 10 groups (See the code for more details). 

\begin{figure*}[ht]%
    \centering
    \includegraphics[width=0.24\textwidth, height=0.16\textwidth]{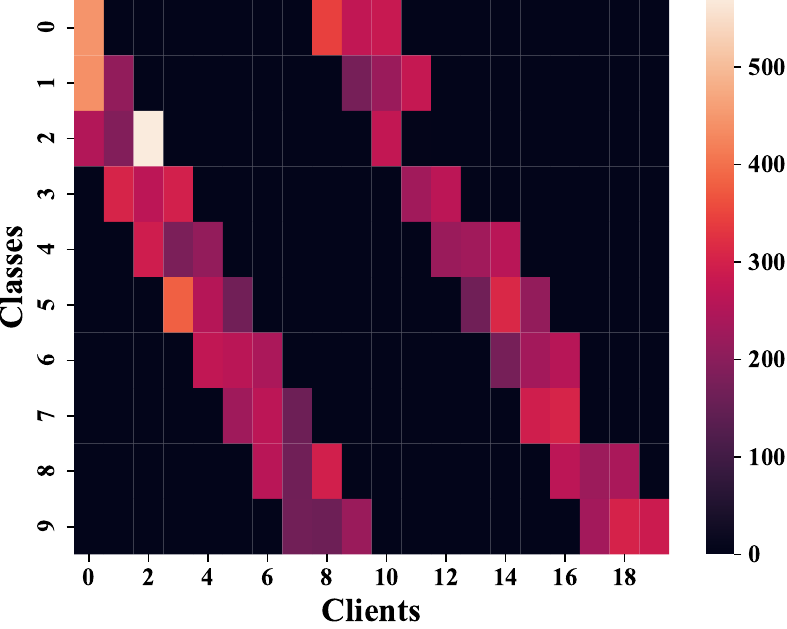}
    \includegraphics[width=0.24\textwidth, height=0.16\textwidth]{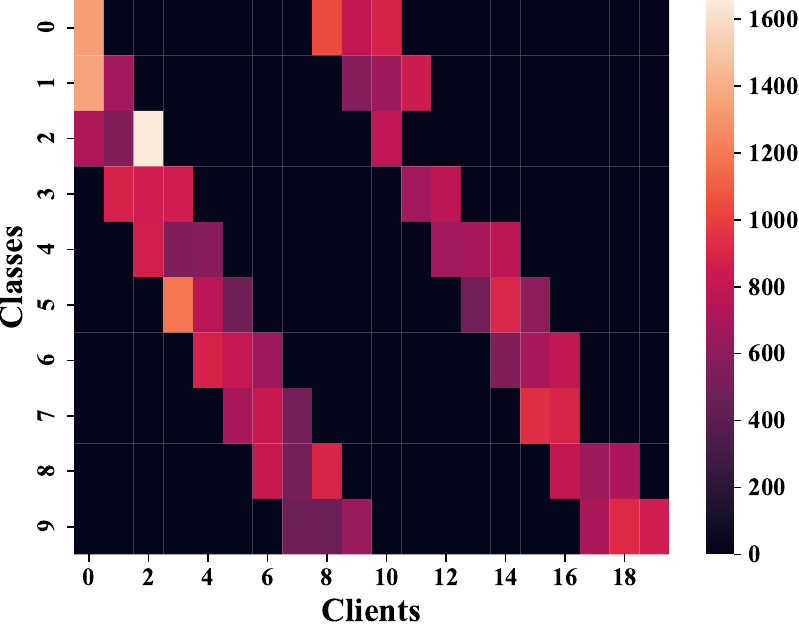}  
    \includegraphics[width=0.24\textwidth, height=0.16\textwidth]{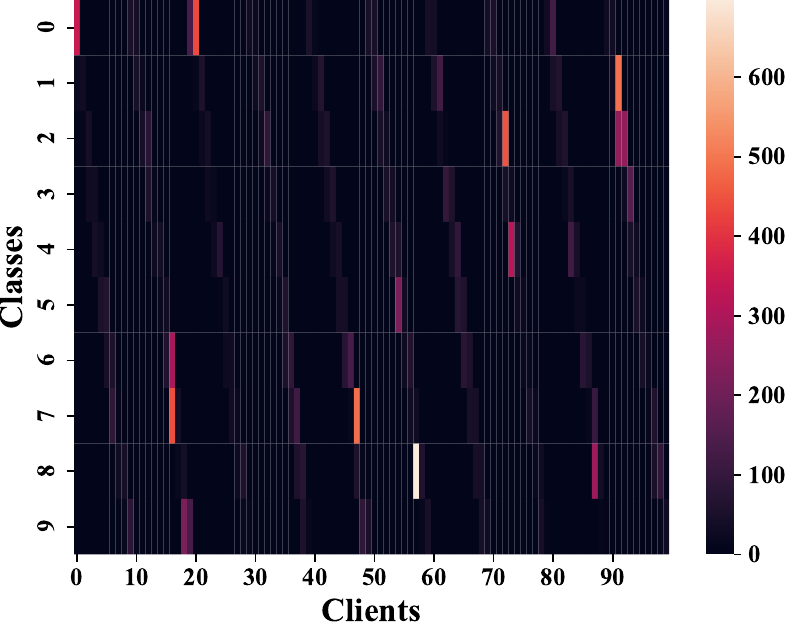} 
    \includegraphics[width=0.24\textwidth, height=0.16\textwidth]{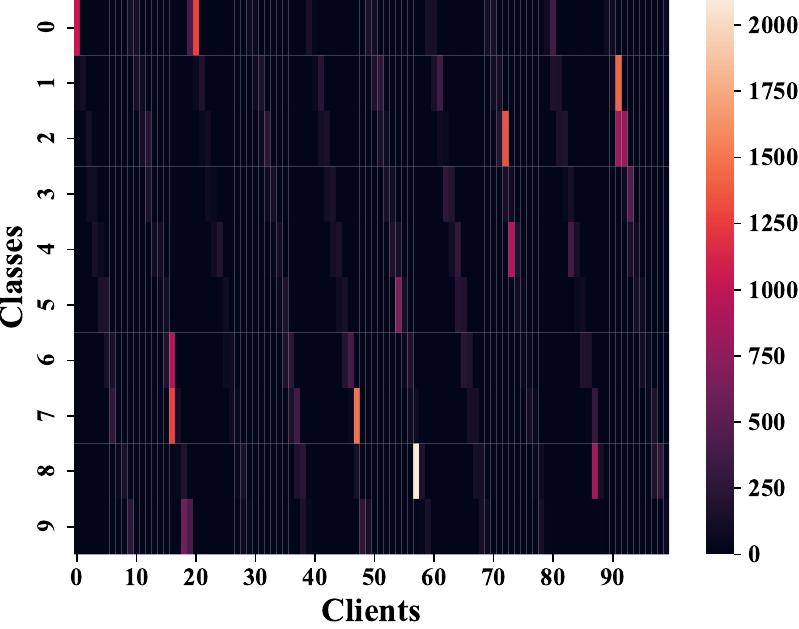}
    \begin{center}
        \footnotesize 
        CIFAR-10-Test/Train \qquad\qquad\qquad\qquad
        FMNIST-Test/Train
    \end{center}
    \includegraphics[width=0.24\textwidth, height=0.16\textwidth]{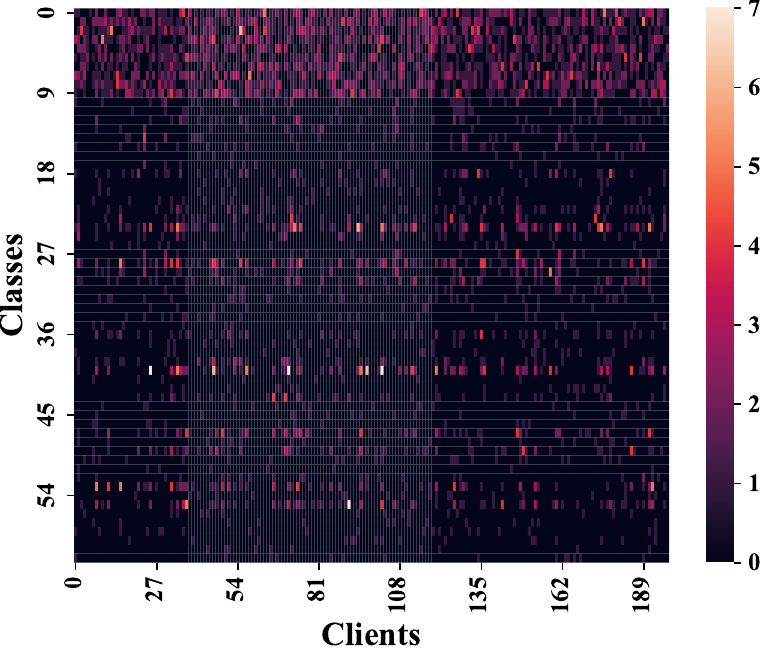}    
    \includegraphics[width=0.24\textwidth, height=0.16\textwidth]{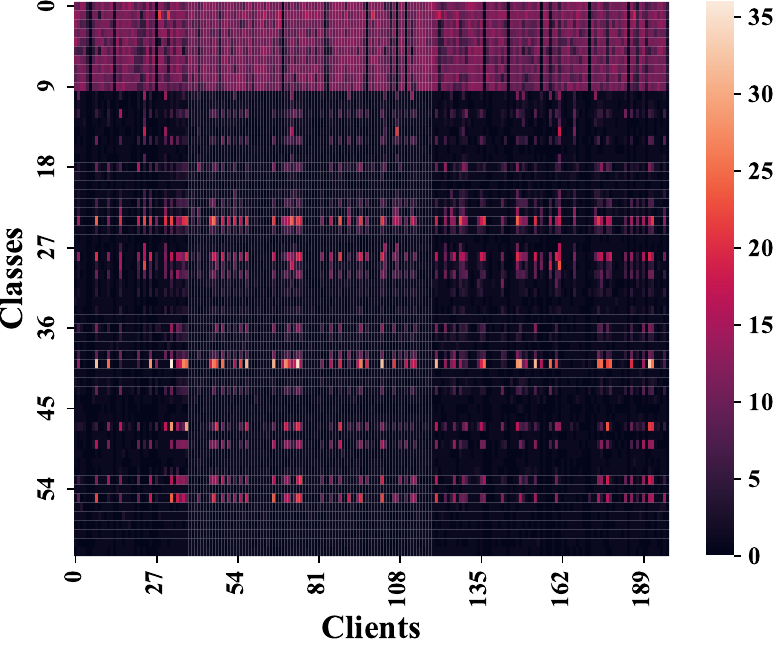}
    \includegraphics[width=0.24\textwidth, height=0.16\textwidth]{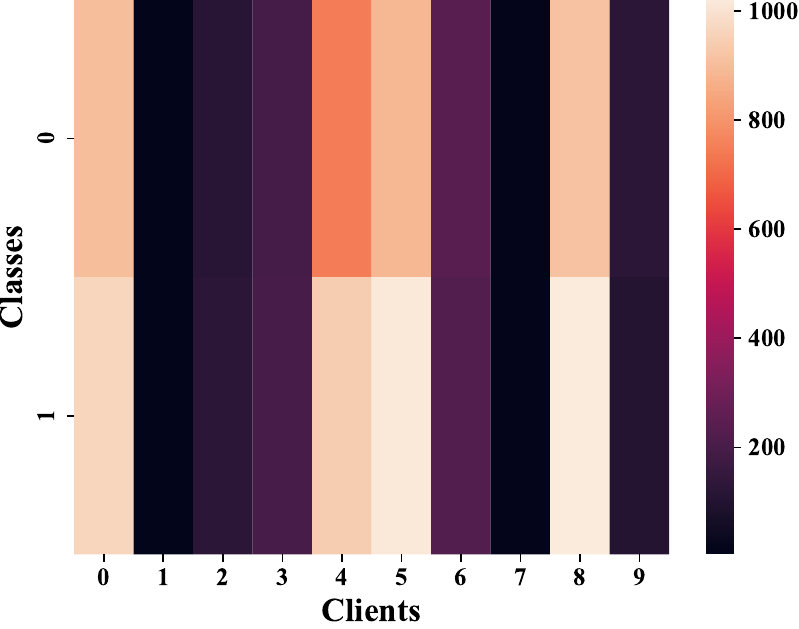}
    \includegraphics[width=0.24\textwidth, height=0.16\textwidth]{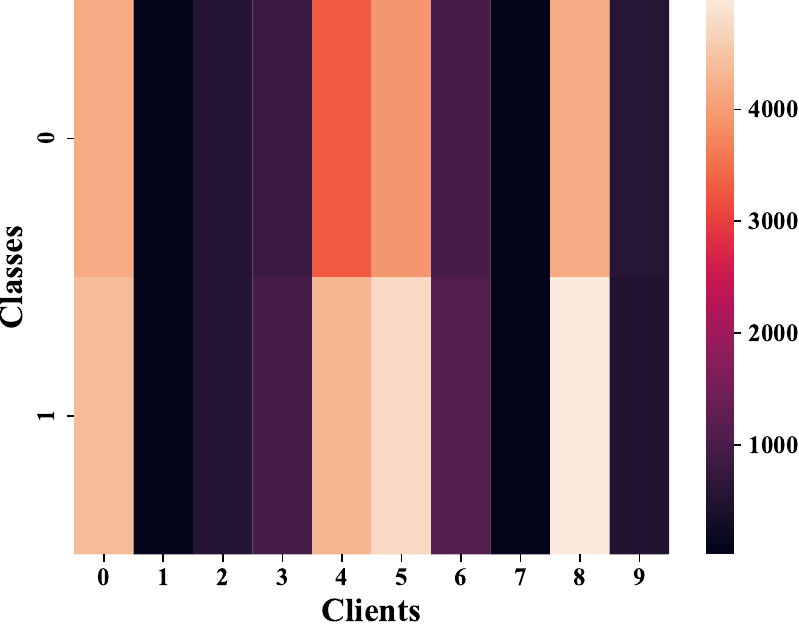}  
    \begin{center}
        \footnotesize
        FEMNIST-Test/Train \qquad\qquad\qquad\qquad
        Sent140-10-Test/Train
    \end{center}

    \includegraphics[width=0.24\textwidth, height=0.16\textwidth]{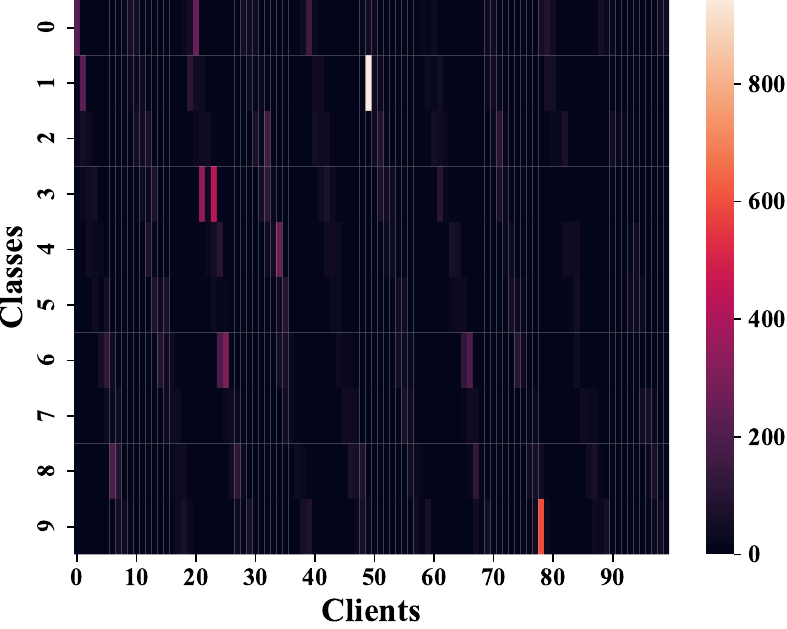}
    \includegraphics[width=0.24\textwidth, height=0.16\textwidth]{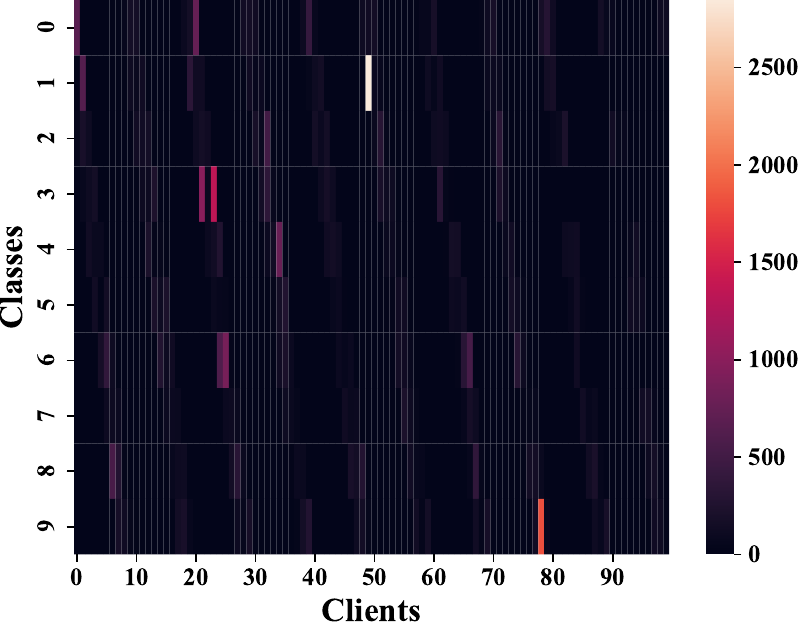}  
    \begin{center}
        \footnotesize 
        MNIST-Test/Train
    \end{center}
    
    \caption{The visualization of the non-i.i.d. data distributions of MNIST, CIFAR-10, FMNIST, FEMNIST and Sent140.}
    \label{appdx_fig_niid}
\end{figure*}

\subsection{More About Hyper-Parameter Effect}

We post the hyper-parameter effects of $\eta$ and $\lambda$ on FEMNIST, FMNIST, MNIST and Sent140 and of $\eta$ on CIFAR-10 in Figure~\ref{appdx_fig_hpe_femnist}-~\ref{appdx_fig_hpe_sent140_cifar-10}. We haven't put the effects of $\lambda$ on CIFAR-10 for better visualization of the effects of more sensitive eta, as well as our equipment limitations, and the fact that other non-linear models for image classification are already demonstrated on FEMNIST, FMNIST and MNIST. The results of these figures are in the same hyper-parameter settings as mentioned in Section~\ref{ssec_exps} except the varying hyper-parameters.

\begin{figure*}[ht]
    \centering
    \includegraphics[width=0.24\textwidth, height=0.16\textwidth]
    {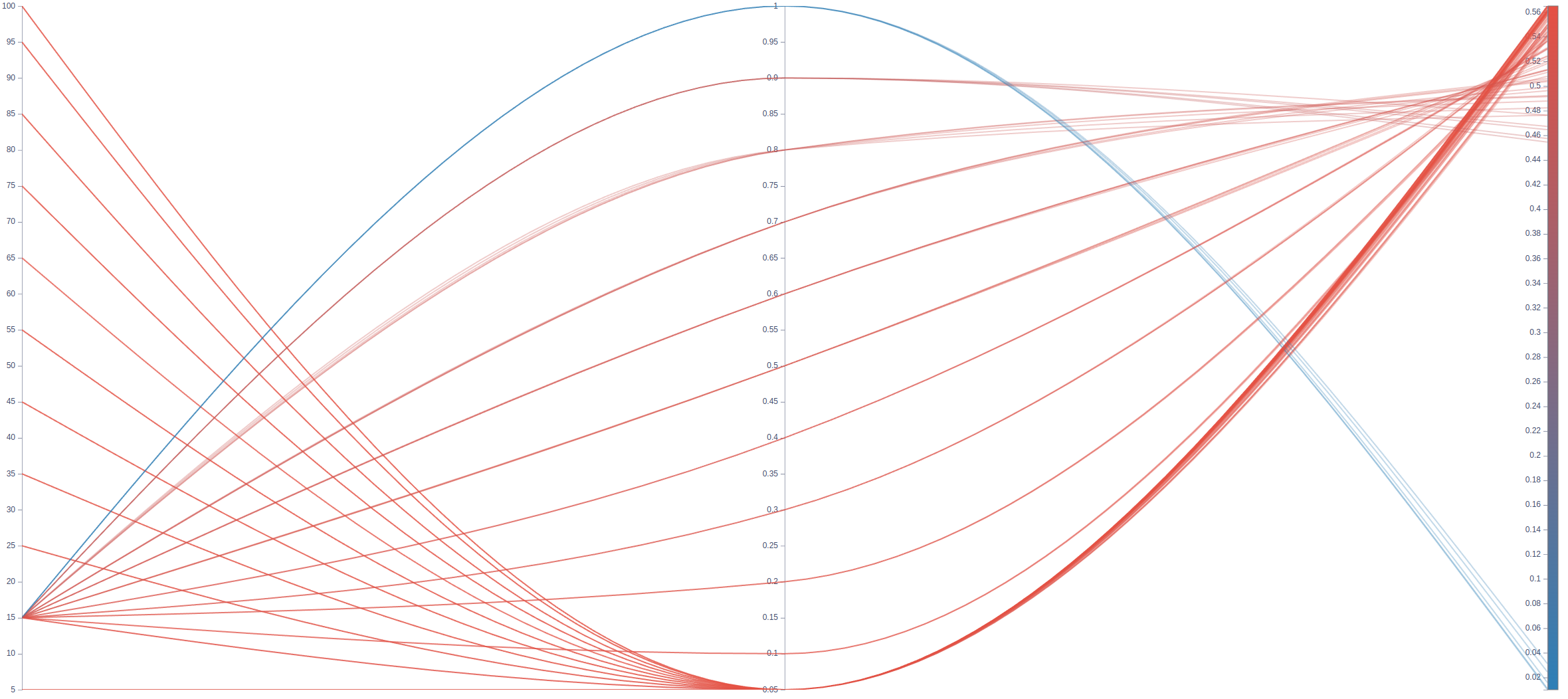}
    \includegraphics[width=0.24\textwidth, height=0.16\textwidth]
    {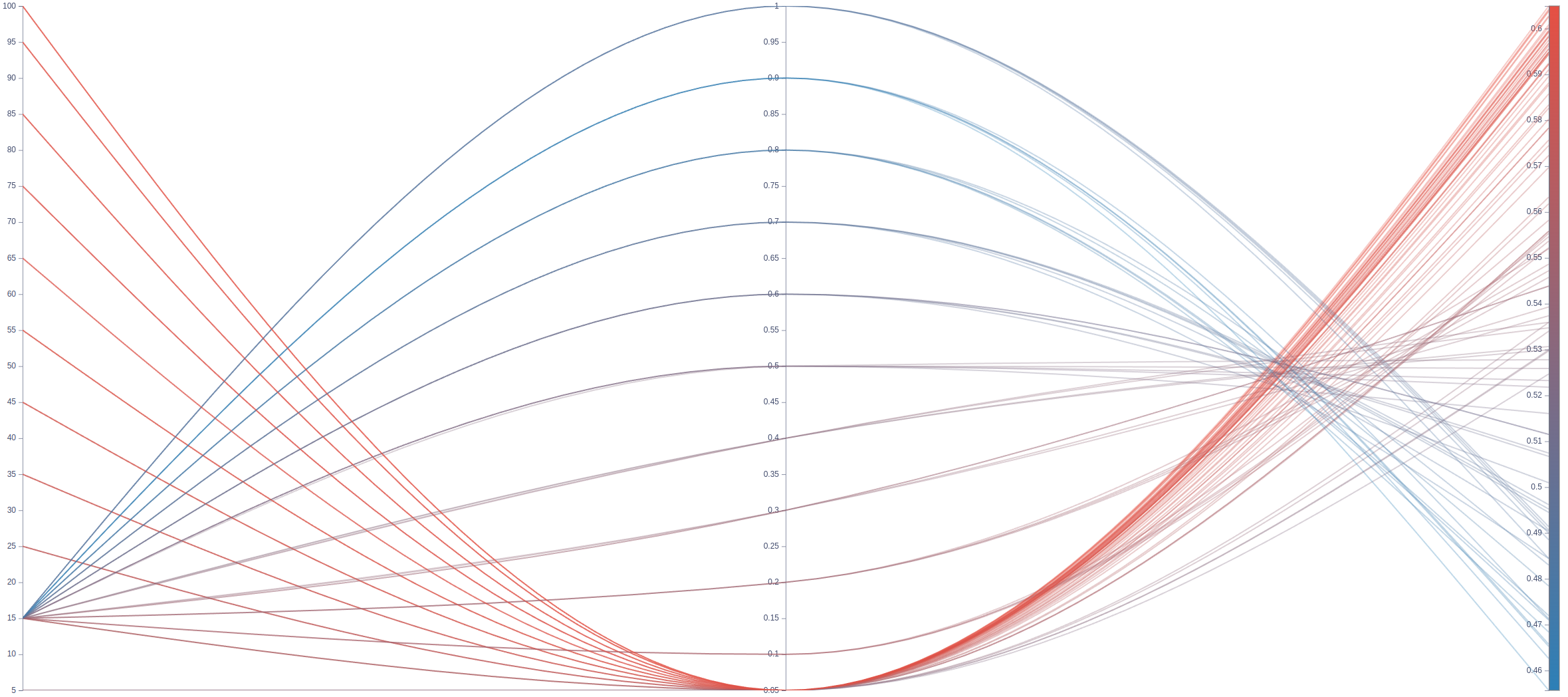}
    \includegraphics[width=0.24\textwidth, height=0.16\textwidth]{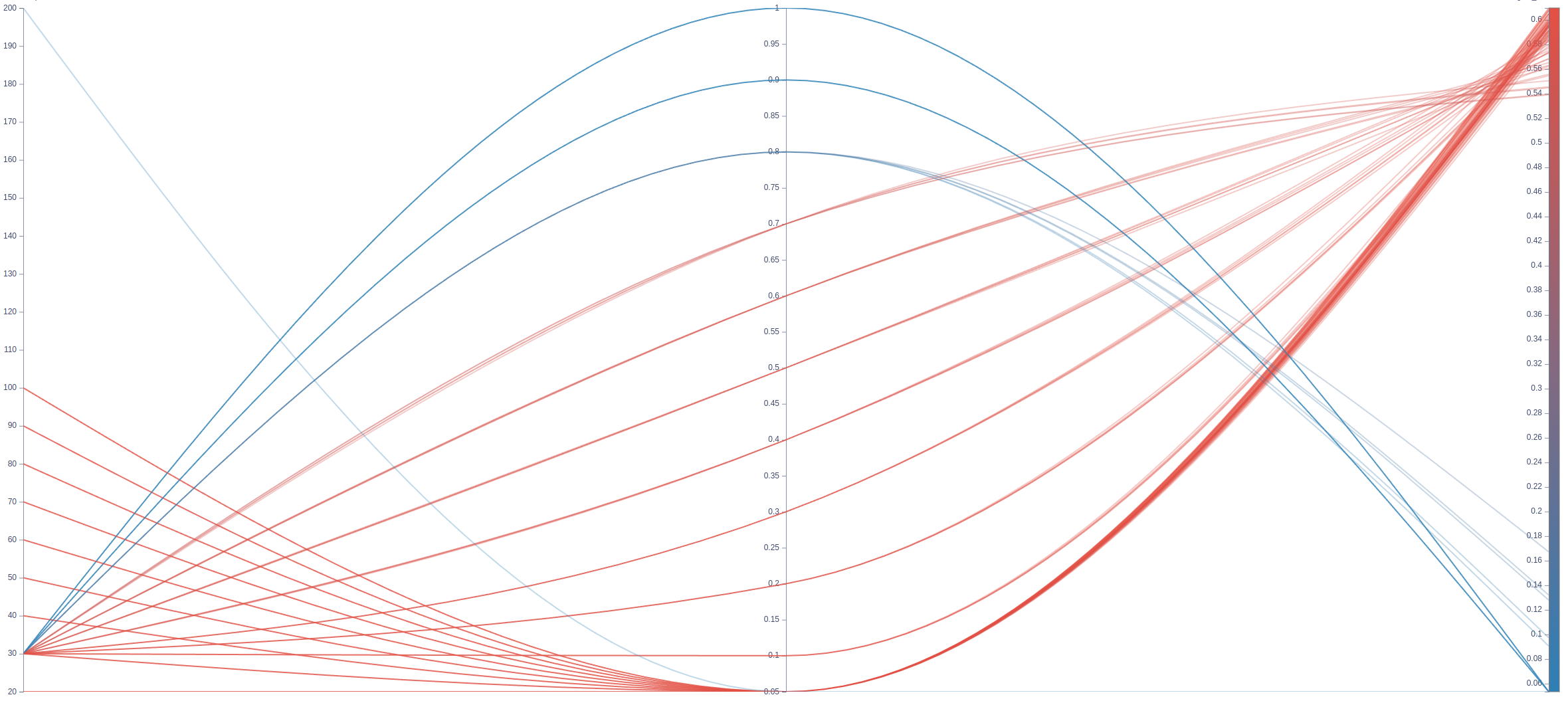}
    \includegraphics[width=0.24\textwidth, height=0.16\textwidth]{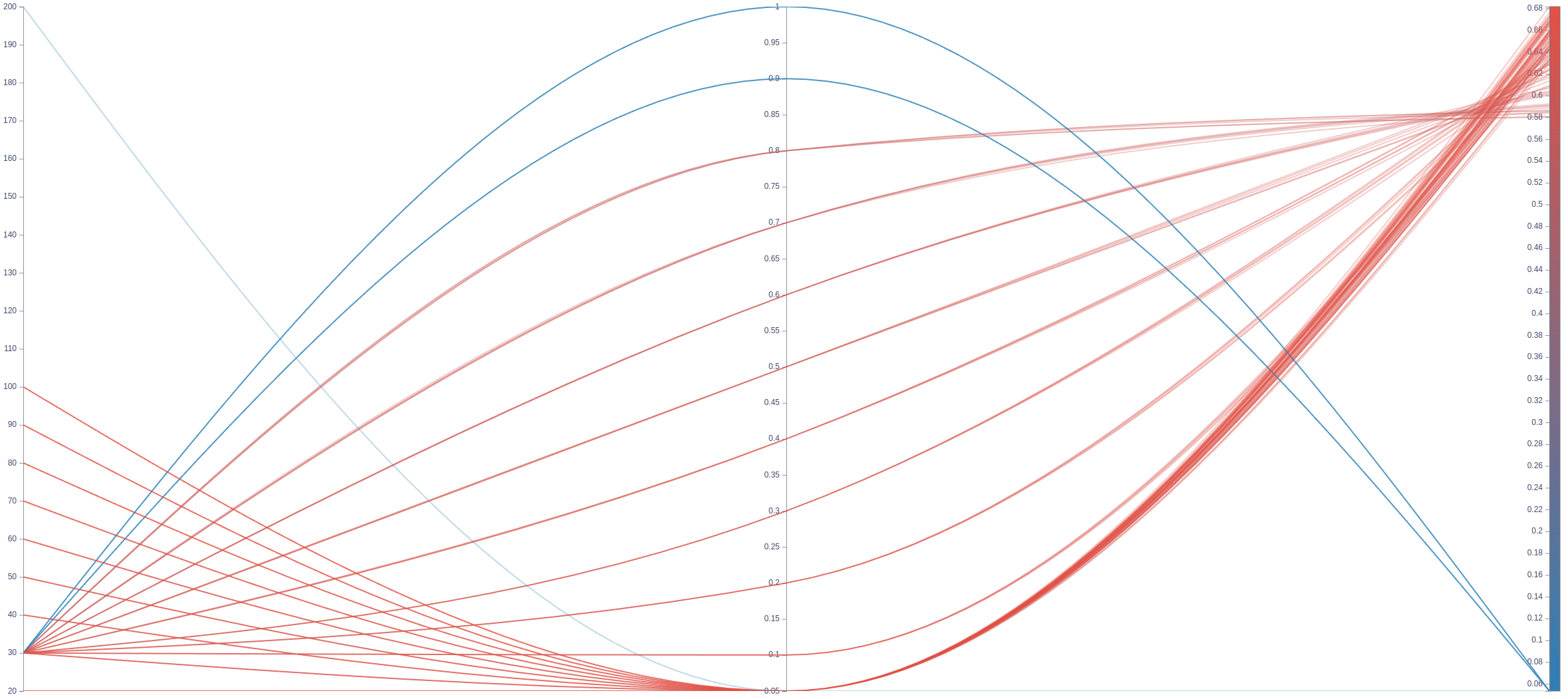}
    \begin{center}
        \footnotesize 
        FEMNIST-MCLR-G \qquad\qquad FEMNIST-MCLR-P
        \qquad\qquad
        FEMNIST-DNN-G \qquad\qquad FEMNIST-DNN-P
    \end{center}
    
    \caption{Hyper-parameter effect: The left, middle and right bars in each figure respectively represent $\lambda$, $\eta$ and test accuracy, ranges of which are respectively [0,100], [0,1] and [0,1] increasing from bottom to top (color from blue to red refers to the accuracy from 0 to 1).}
    \label{appdx_fig_hpe_femnist}
\end{figure*}

\begin{figure*}[ht]%
    \centering
    \includegraphics[width=0.24\textwidth, height=0.16\textwidth]{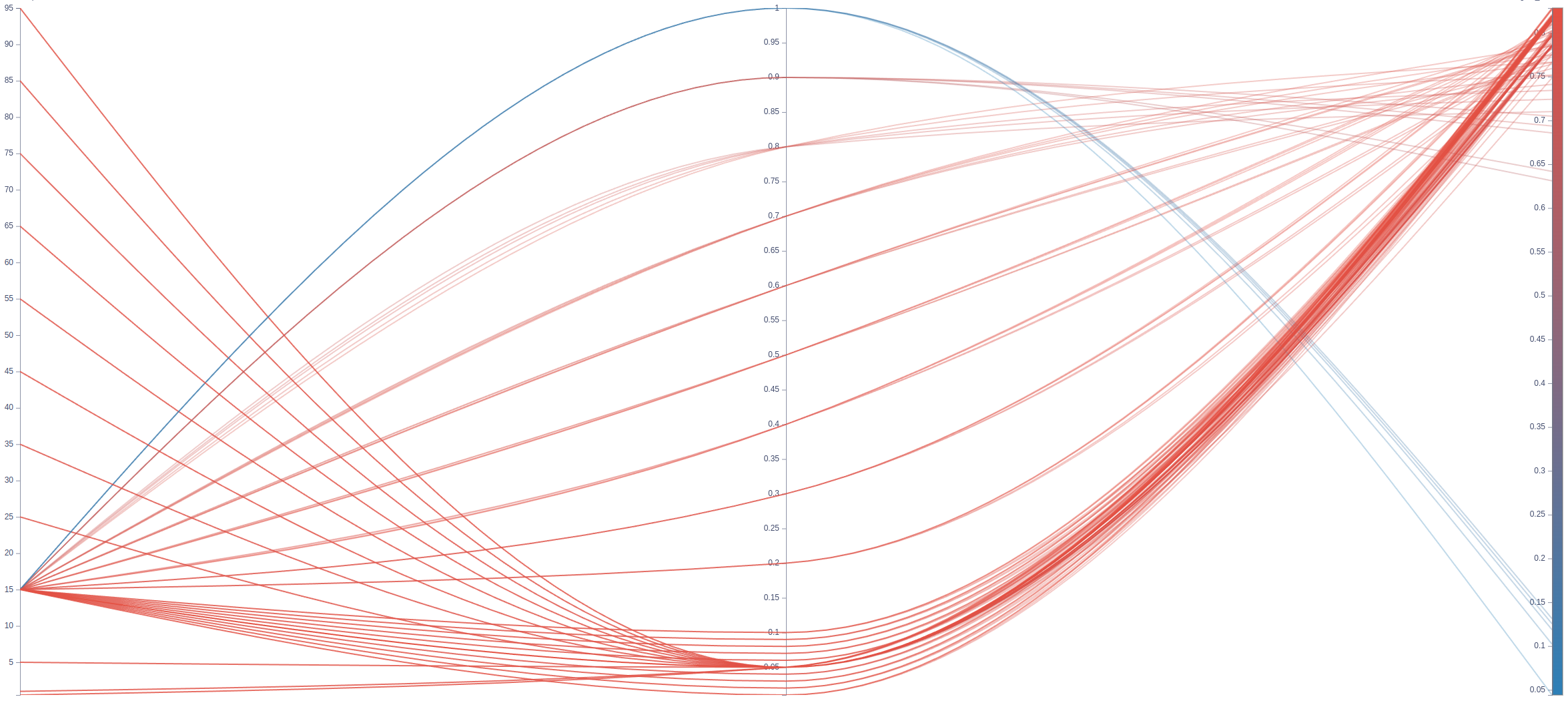}
    \includegraphics[width=0.24\textwidth, height=0.16\textwidth]{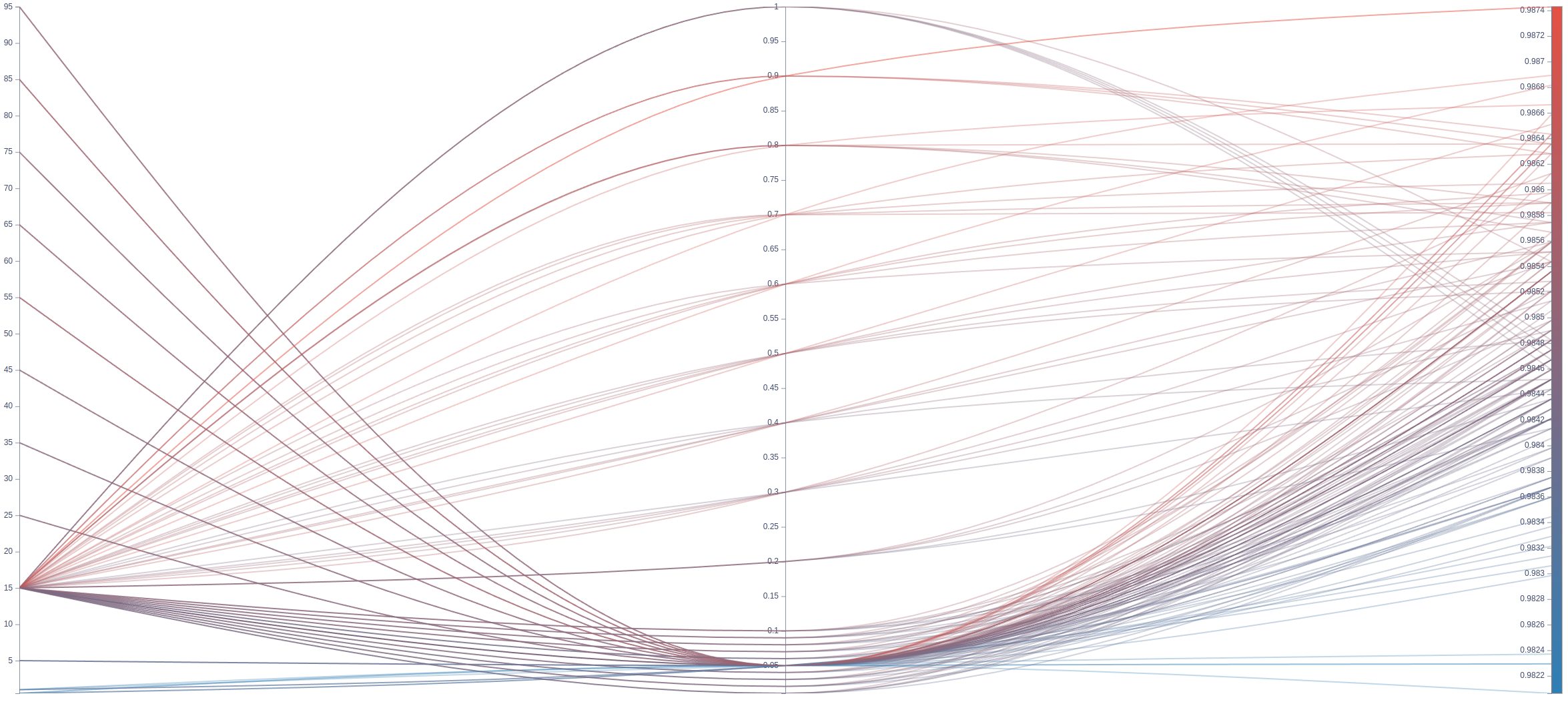}
    \includegraphics[width=0.24\textwidth, height=0.16\textwidth]{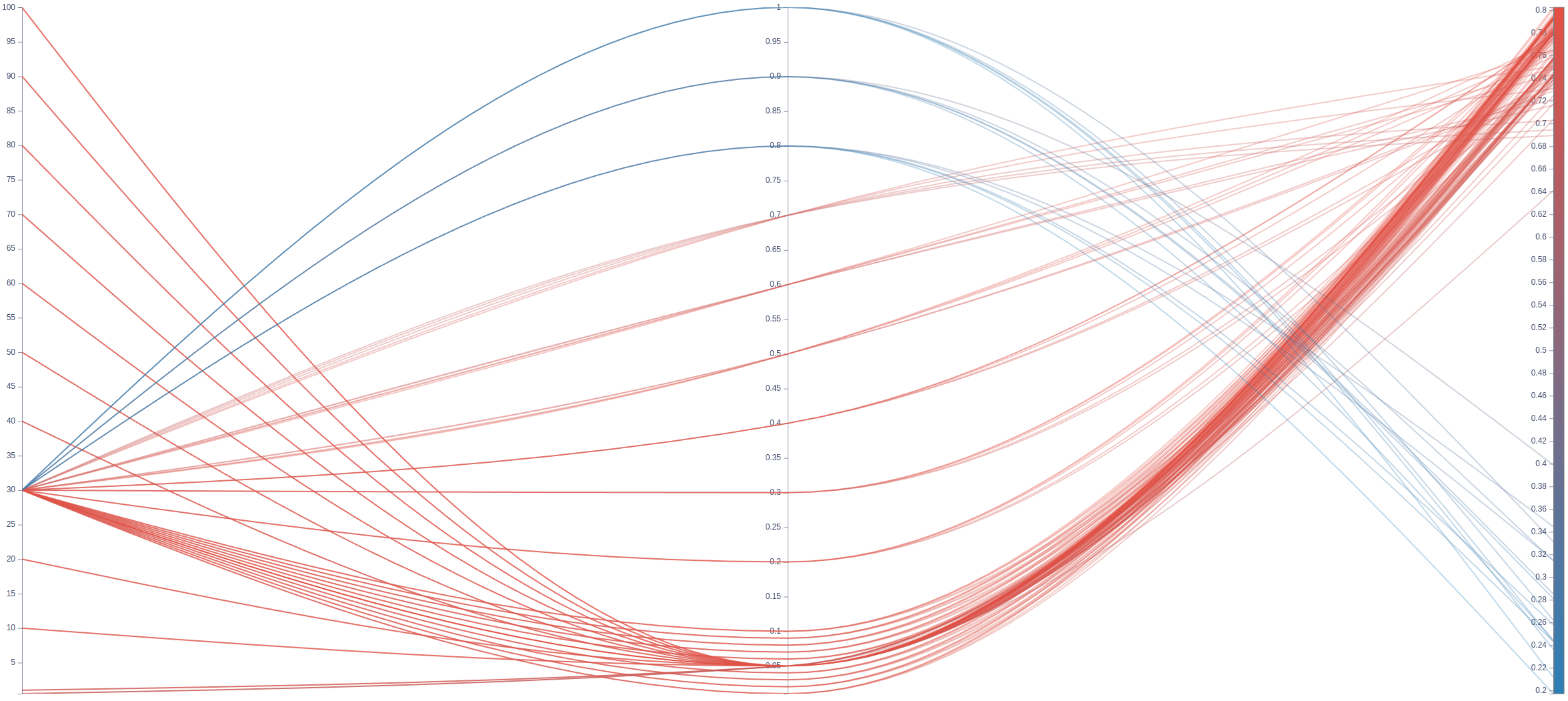}
    \includegraphics[width=0.24\textwidth, height=0.16\textwidth]{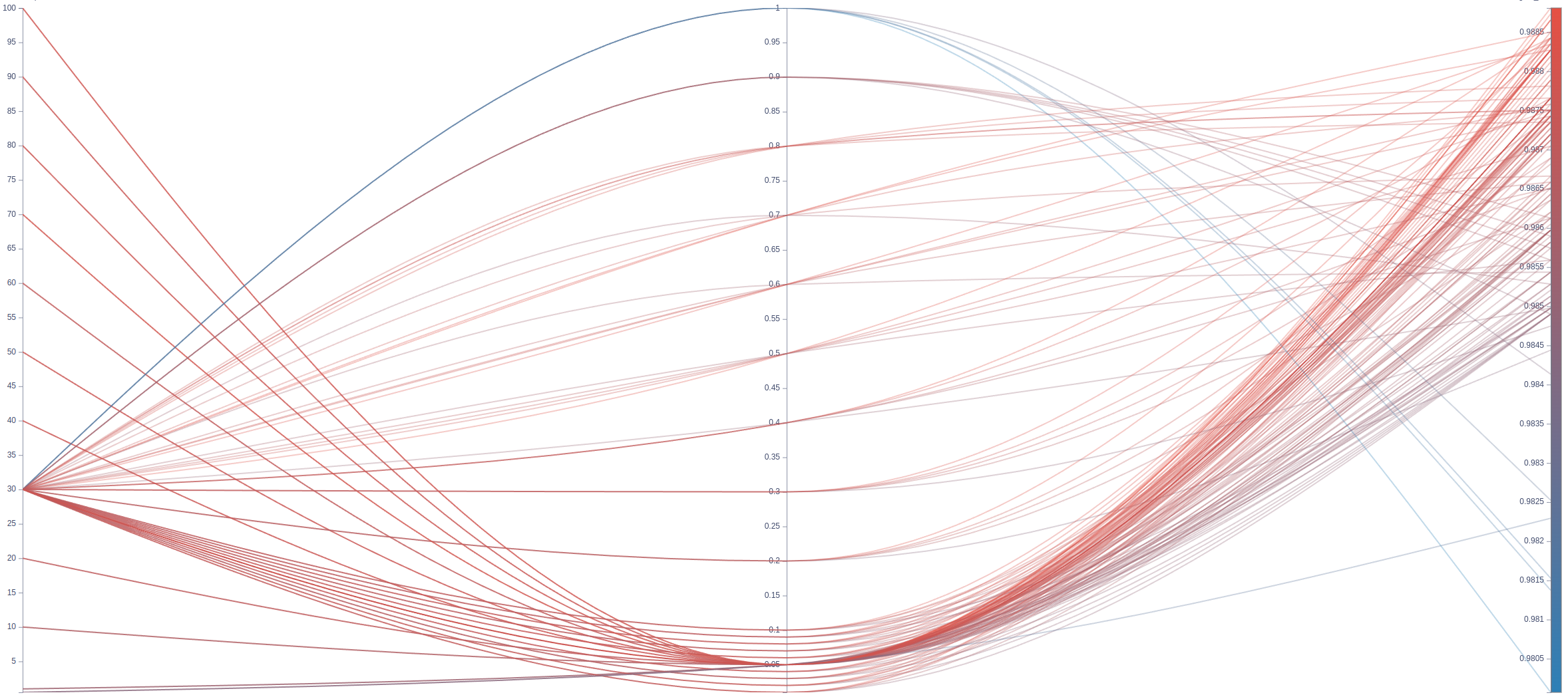}
    \begin{center}
        \footnotesize 
        FMNIST-MCLR-G \qquad\qquad FMNIST-MCLR-P
        \qquad\qquad
        FMNIST-DNN-G \qquad\qquad FMNIST-DNN-P
    \end{center}
    \caption{The left, middle and right bars in each figure respectively represent $\lambda$, $\eta$ and test accuracy, ranges of which are respectively [0,100], [0,1] and [0,1] increasing from bottom to top (color from blue to red).}
    \label{appdx_fig_hpe_fmnist}
\end{figure*}

\begin{figure*}[ht]%
    \centering
    \includegraphics[width=0.24\textwidth, height=0.16\textwidth]{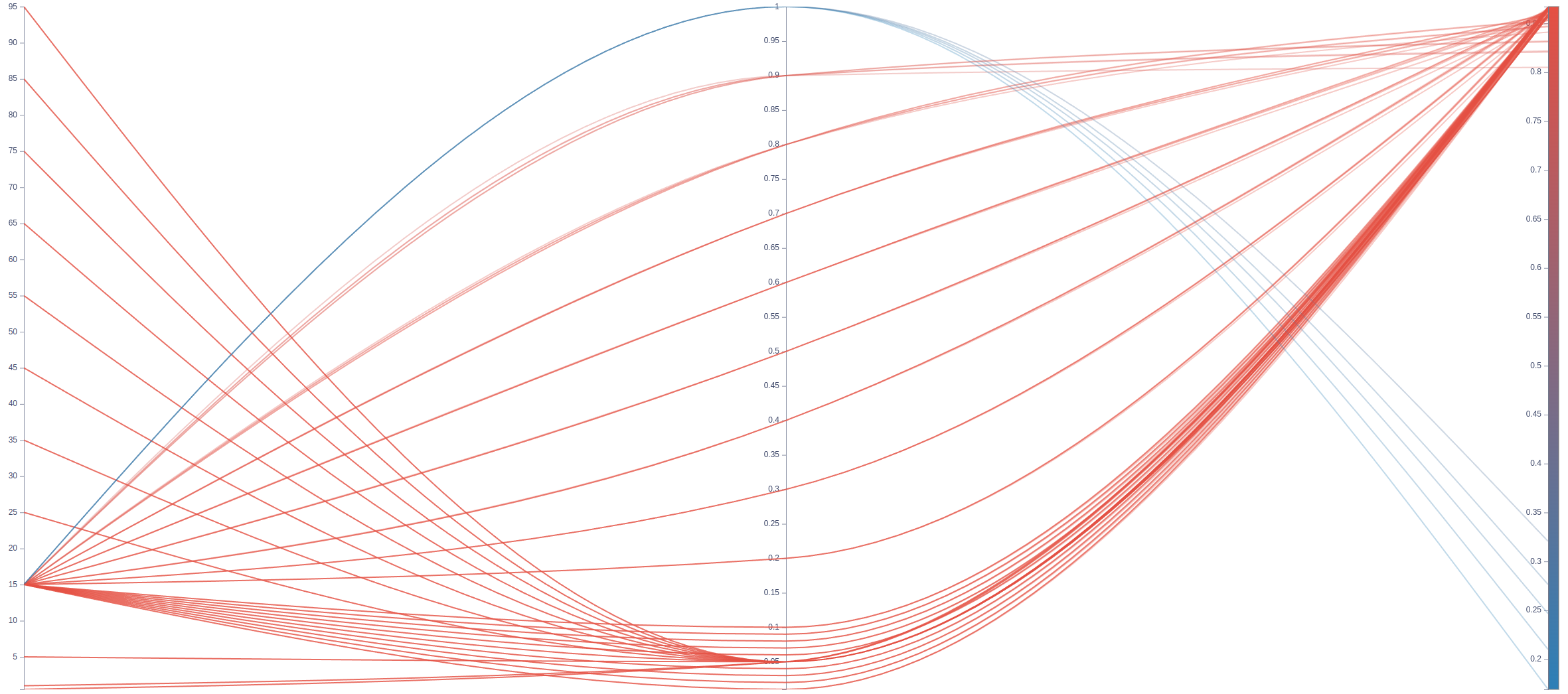}
    \includegraphics[width=0.24\textwidth, height=0.16\textwidth]{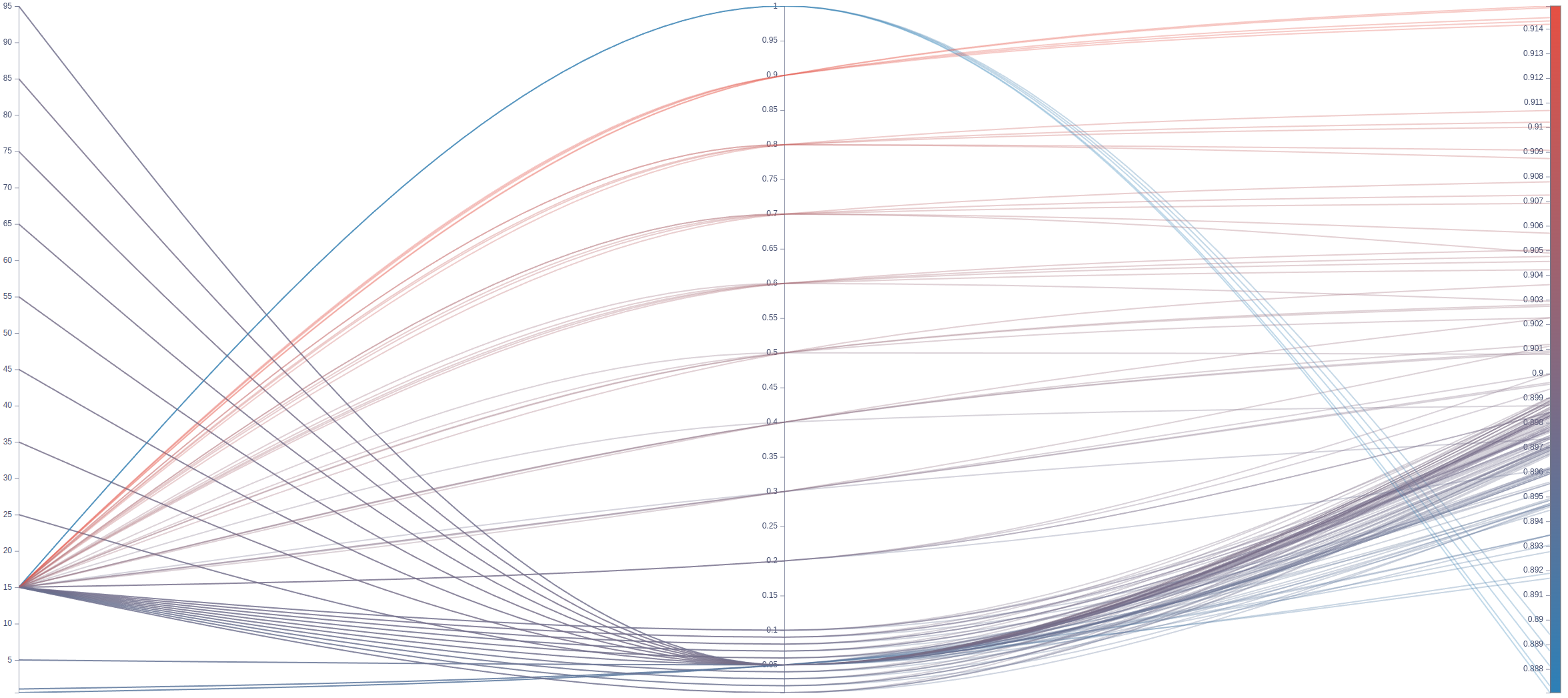}
    \includegraphics[width=0.24\textwidth, height=0.16\textwidth]{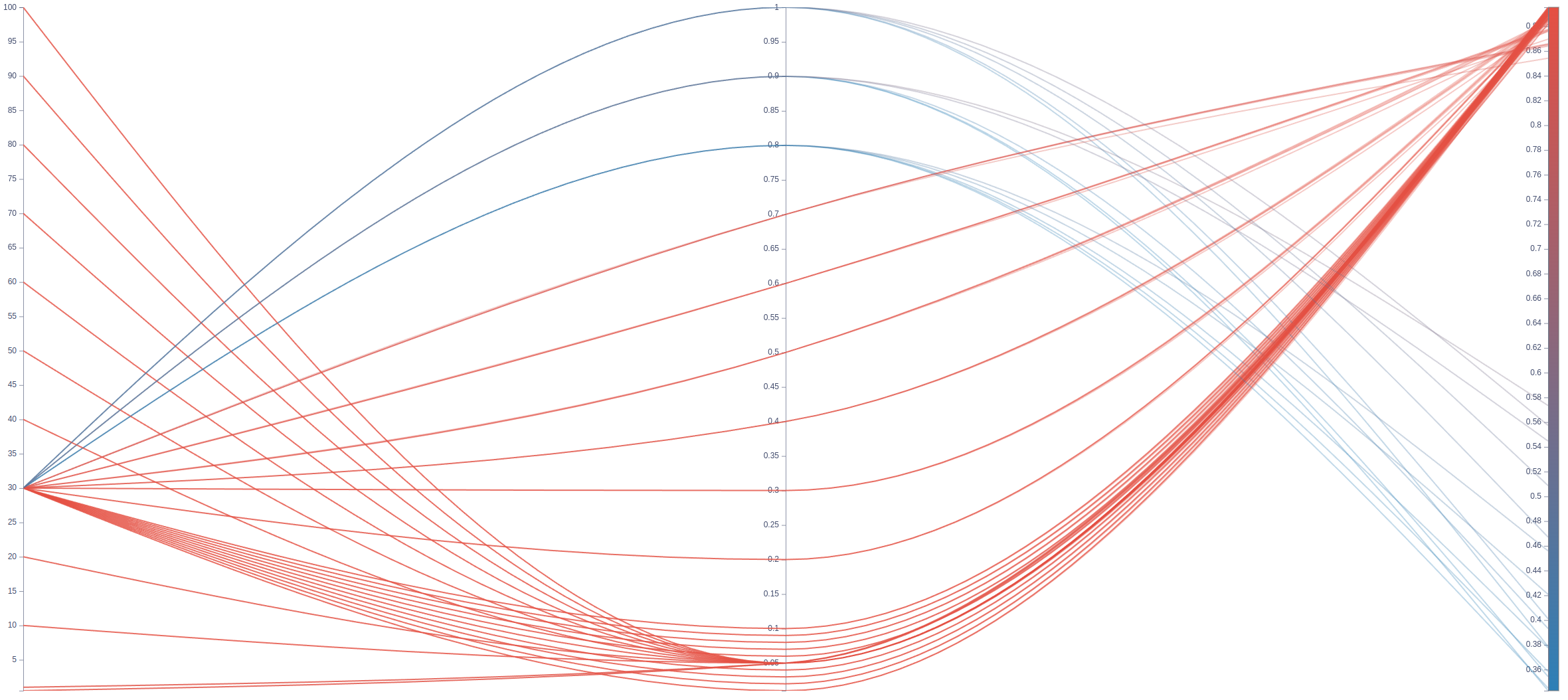}
    \includegraphics[width=0.24\textwidth, height=0.16\textwidth]{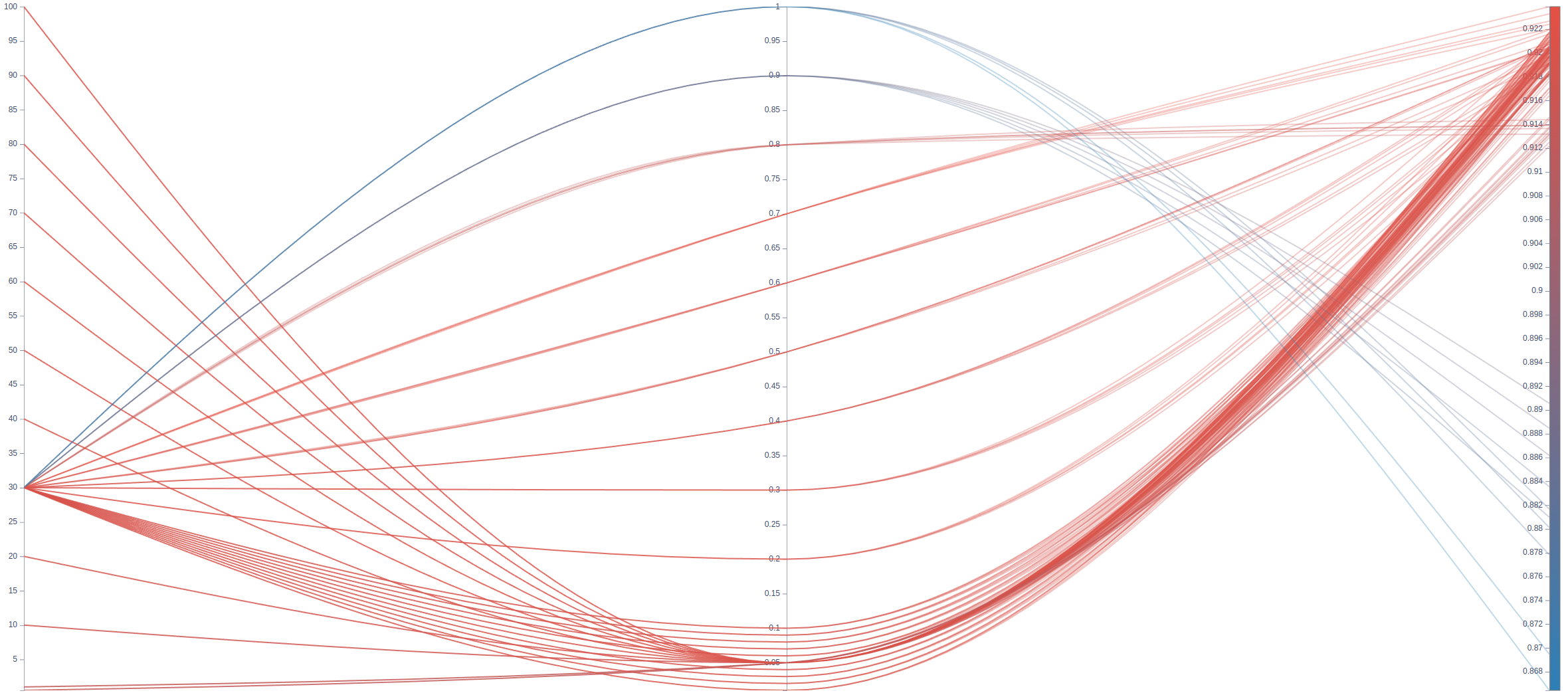}
    \begin{center}
        \footnotesize 
        MNIST-MCLR-G \qquad\qquad MNIST-MCLR-P
        \qquad\qquad
        MNIST-DNN-G \qquad\qquad MNIST-DNN-P
    \end{center}
    \caption{The left, middle and right bars in each figure respectively represent $\lambda$, $\eta$ and test accuracy, ranges of which are respectively [0,100], [0,1] and [0,1] increasing from bottom to top (color from blue to red).}
    \label{appdx_fig_hpe_mnist}
\end{figure*}
\begin{figure*}[ht]%
    \centering
    \includegraphics[width=0.24\textwidth, height=0.16\textwidth]{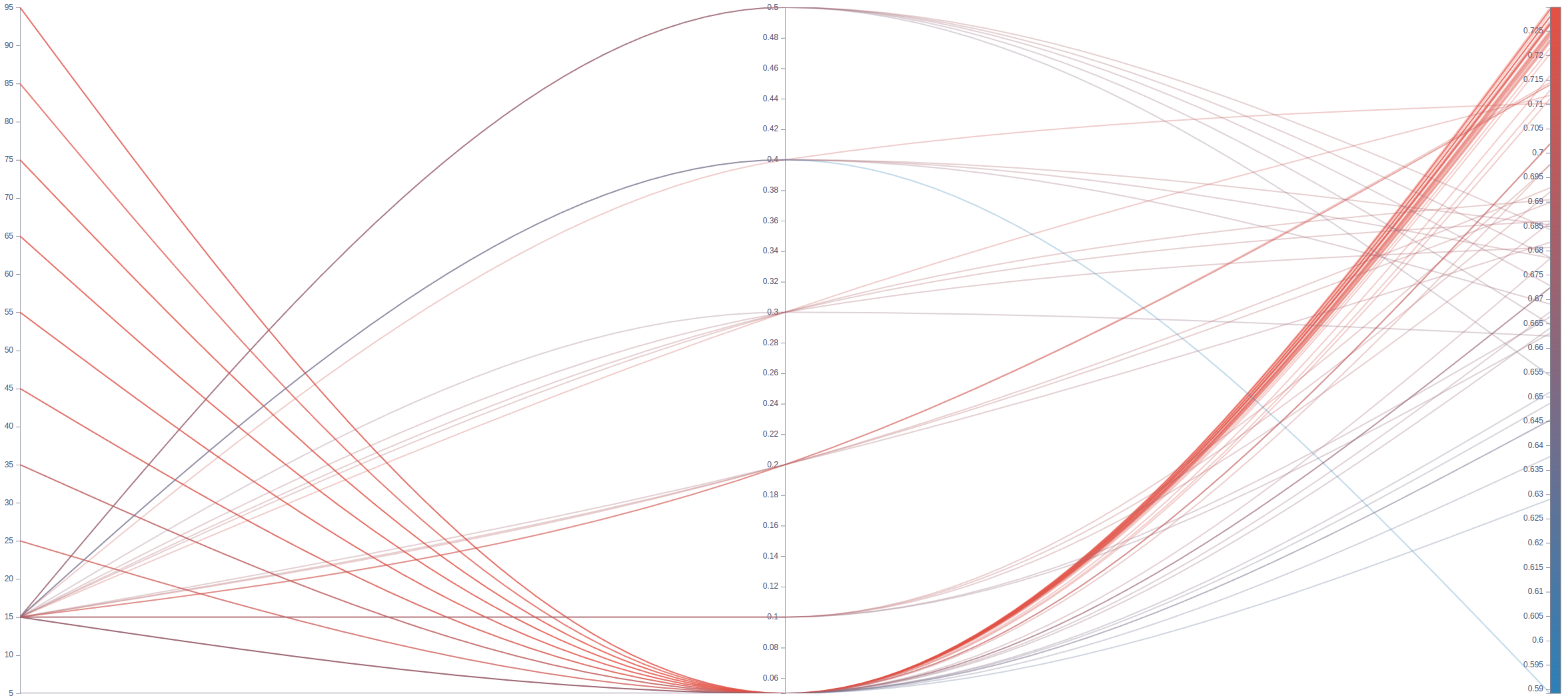}
    \includegraphics[width=0.24\textwidth, height=0.16\textwidth]{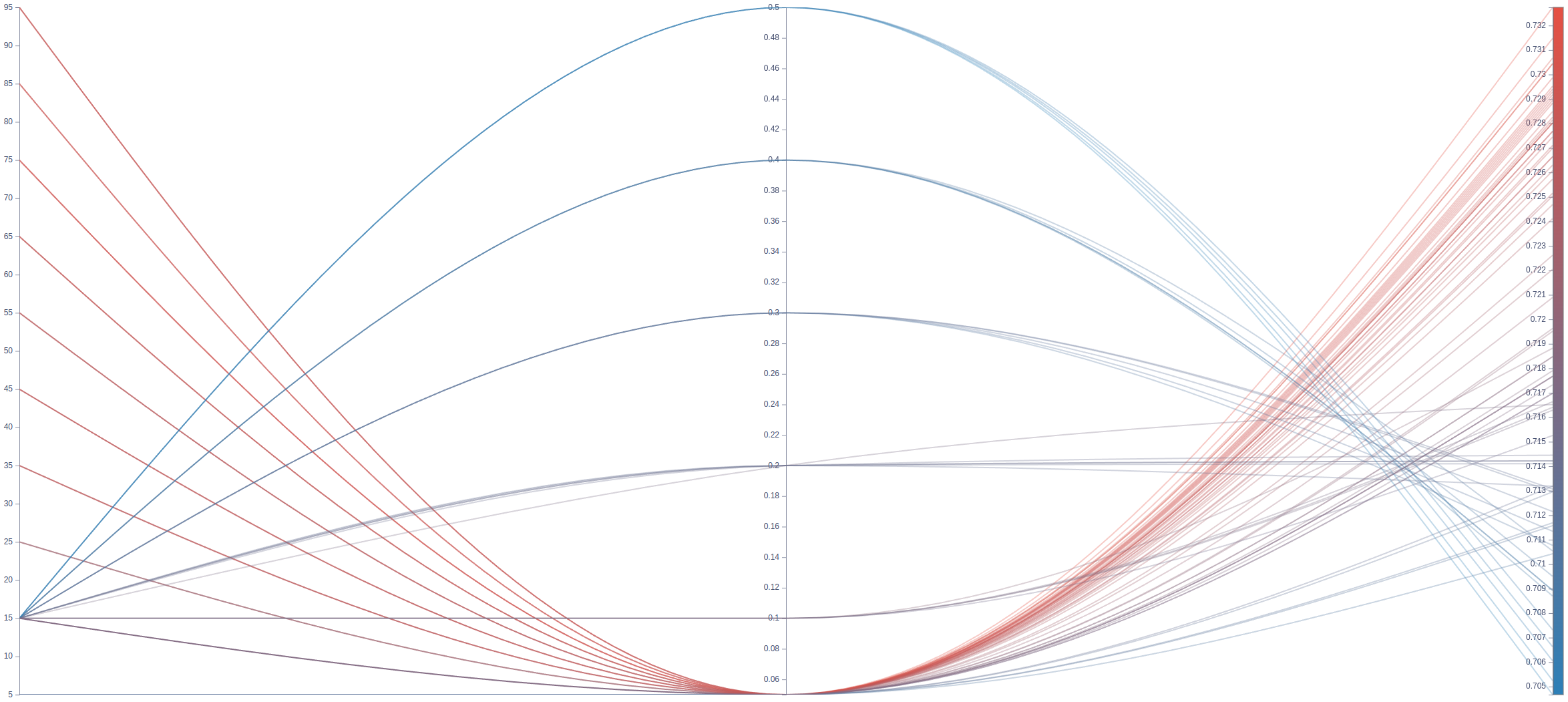}
    \includegraphics[width=0.24\textwidth, height=0.16\textwidth]{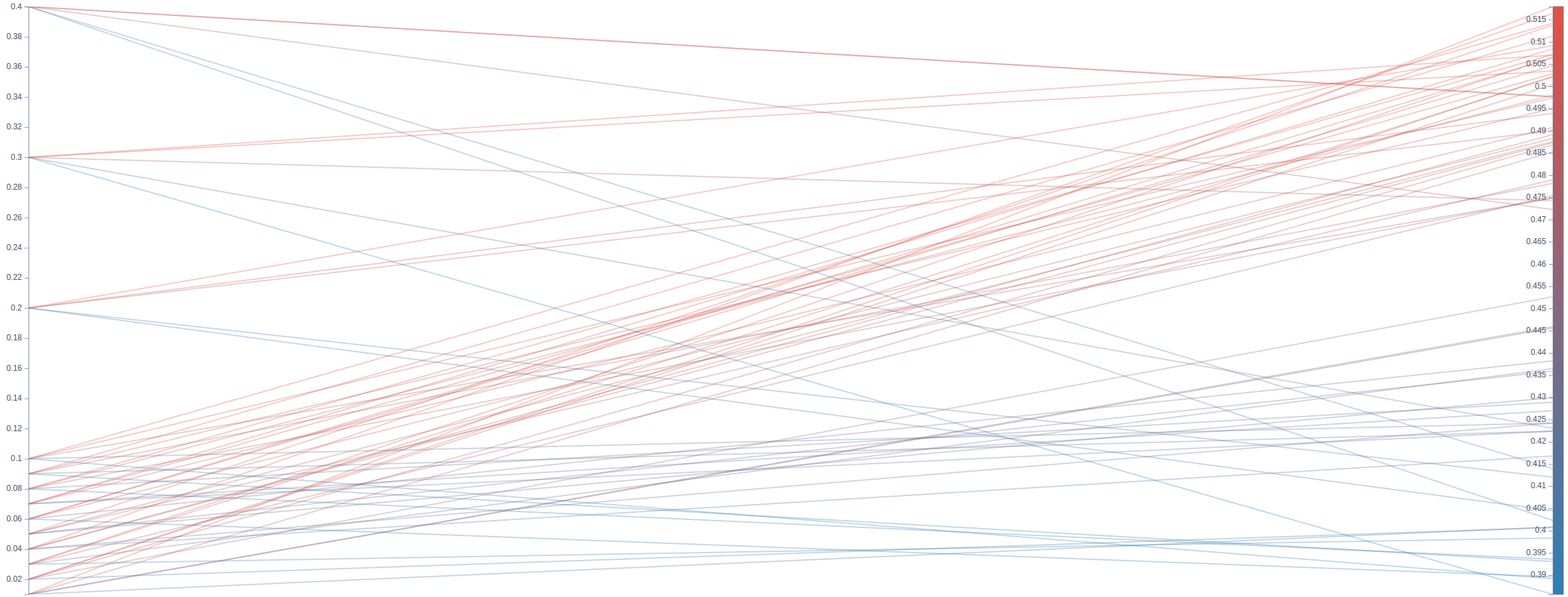}
    \includegraphics[width=0.24\textwidth, height=0.16\textwidth]{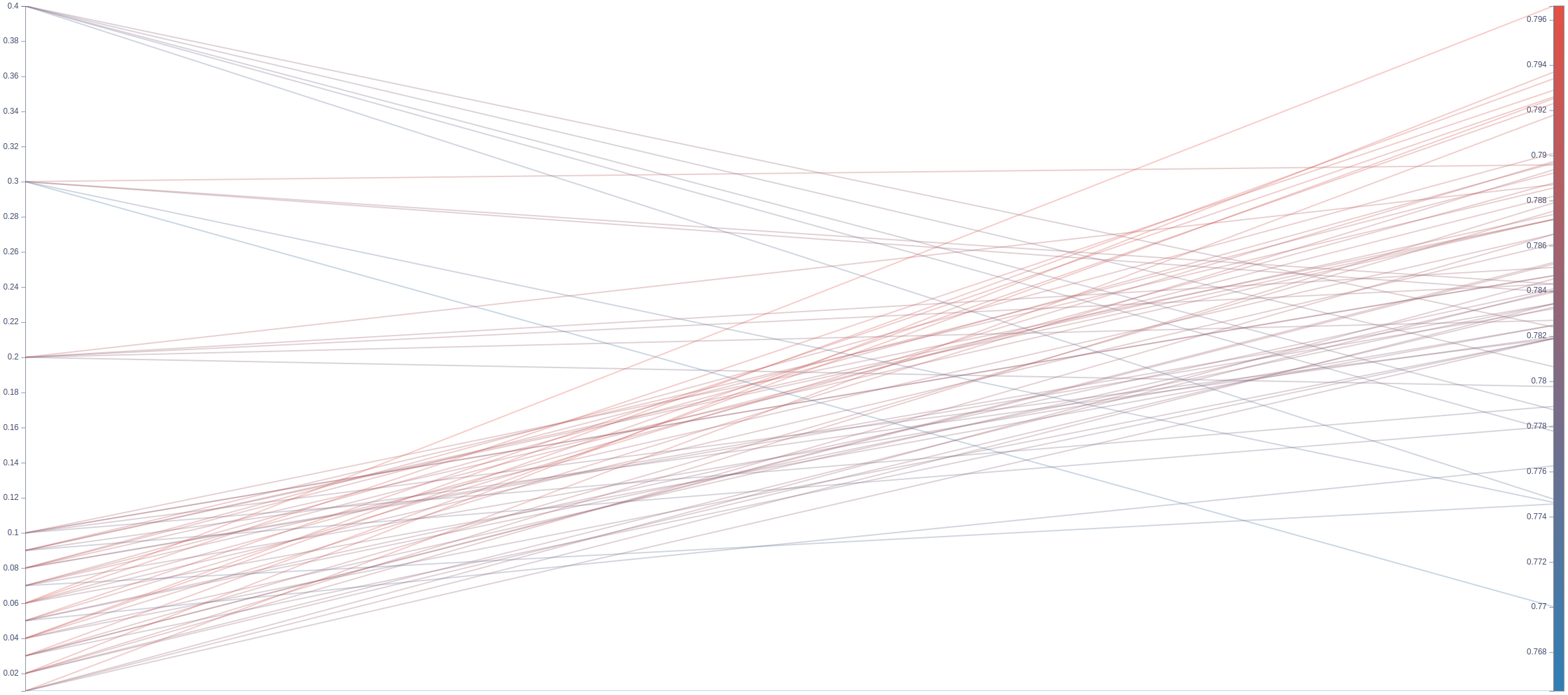}
    \begin{center}
        \footnotesize 
        Sent140-LSTM-G \qquad\qquad Sent140-LSTM-P
        \qquad\qquad
        CIFAR-10-CNN-G \qquad\qquad CIFAR-10-CNN-G
    \end{center}
    \caption{The left, middle and right bars in each figure respectively represent $\lambda$ and test accuracy, ranges of which are respectively [0,100] and [0,1] increasing from bottom to top (color from blue to red). The ranges of $\eta$ are respectively [0,0.5] and [0,0.4] in settings of CIFAR-10-CNN and Sent140.}
    \label{appdx_fig_hpe_sent140_cifar-10}
\end{figure*}

\subsection{More about Deviation Analysis}
\label{appdx_ssec_mDevAnalysis}

The deviations of the global and local test on each settings are shown in Figure~\ref{appdx_fig_psnlz_com} mentioned in Section~\ref{ssec_personalization} in the main paper.

\begin{figure*}[ht]
    \centering
    \includegraphics[width=0.24\textwidth, height=0.213\textwidth]{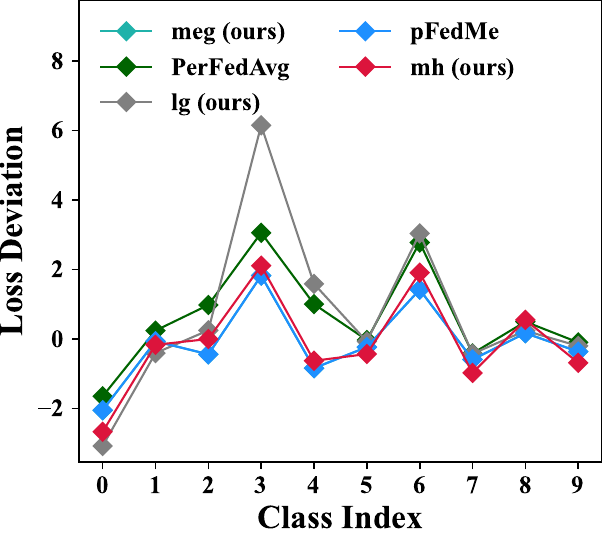}
    \includegraphics[width=0.24\textwidth, height=0.213\textwidth]{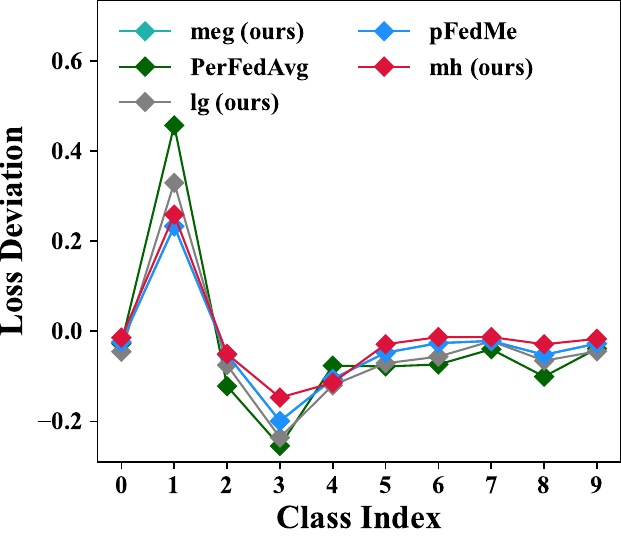}
    \includegraphics[width=0.24\textwidth, height=0.213\textwidth]{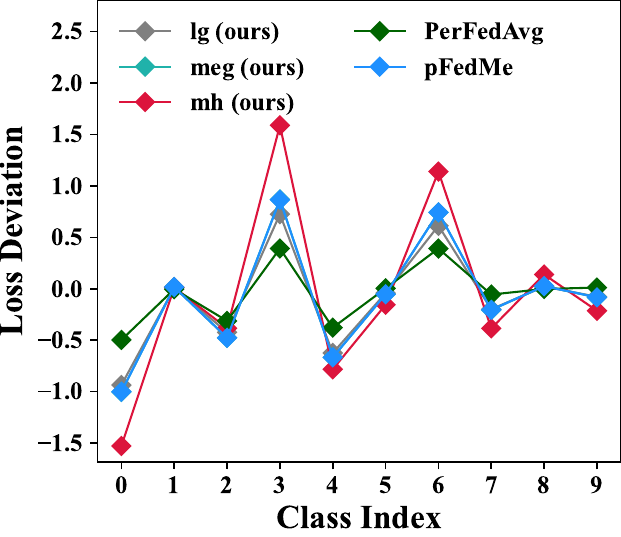}
    \includegraphics[width=0.24\textwidth, height=0.213\textwidth]{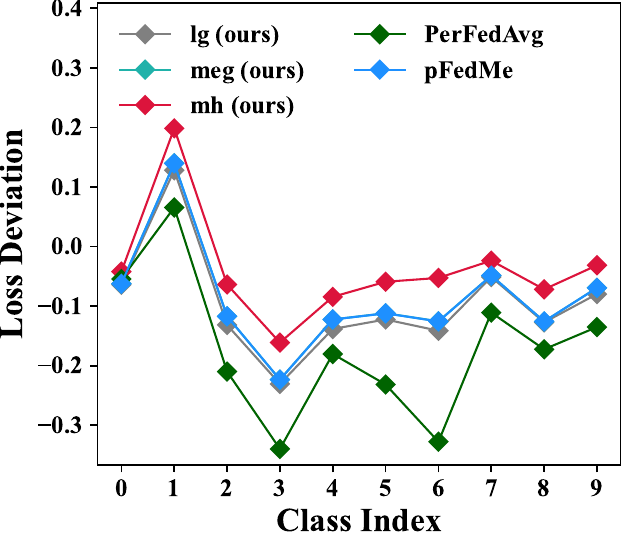}
    \begin{center}
        \footnotesize 
        FMNIST-DNN-G \quad\qquad FMNIST-DNN-L
        \qquad\qquad
        FMNIST-MCLR-G \qquad\qquad FMNIST-MCLR-L
    \end{center}
    \includegraphics[width=0.24\textwidth, height=0.213\textwidth]{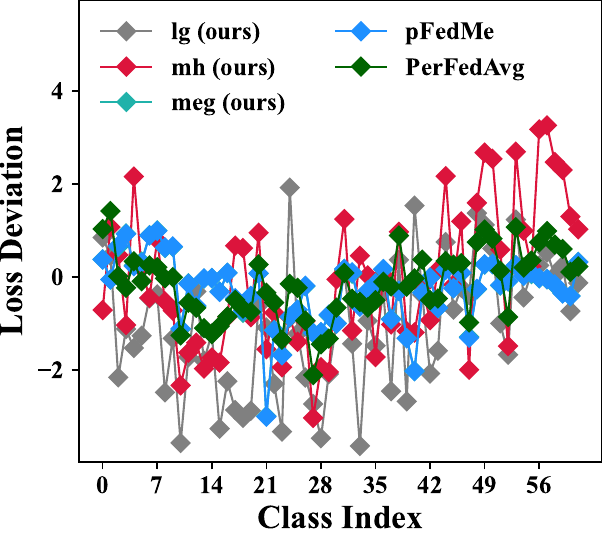}
    \includegraphics[width=0.24\textwidth, height=0.213\textwidth]{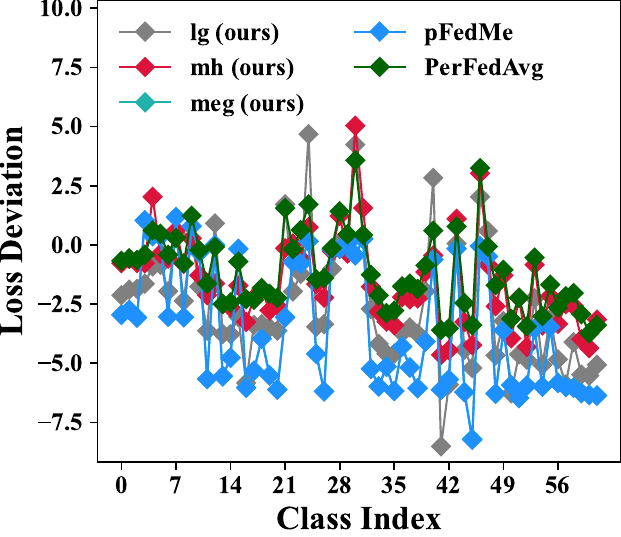}
    \includegraphics[width=0.24\textwidth, height=0.213\textwidth]{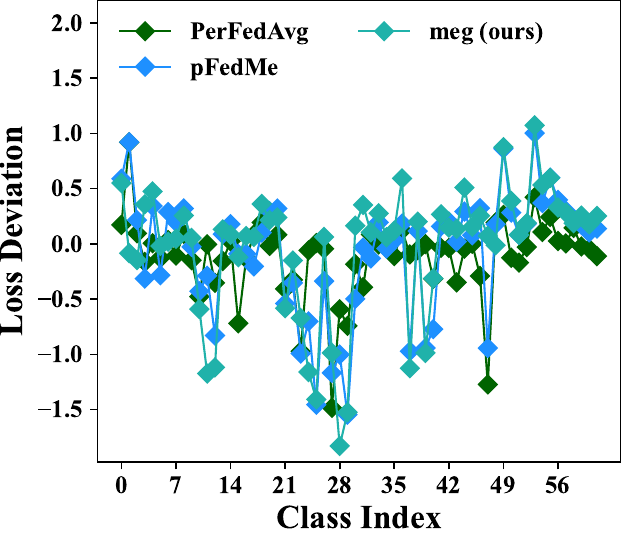}
    \includegraphics[width=0.24\textwidth, height=0.213\textwidth]{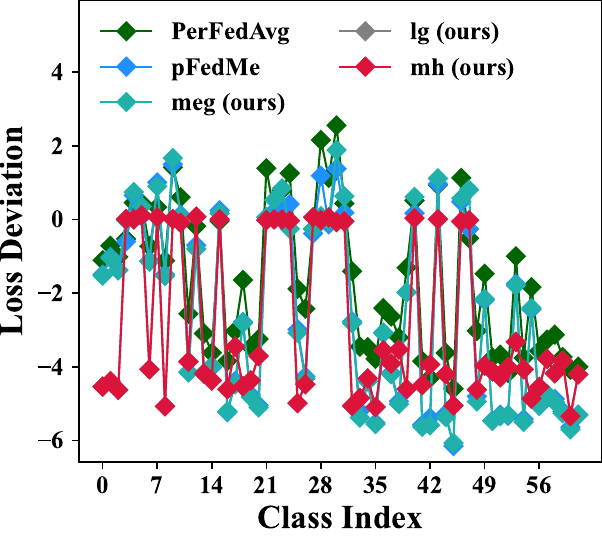}
    \begin{center}
        \footnotesize 
        FEMNIST-DNN-G \quad\qquad FEMNIST-DNN-L
        \qquad
        FEMNIST-MCLR-G \qquad\qquad FEMNIST-MCLR-L
    \end{center}
    \includegraphics[width=0.24\textwidth, height=0.213\textwidth]{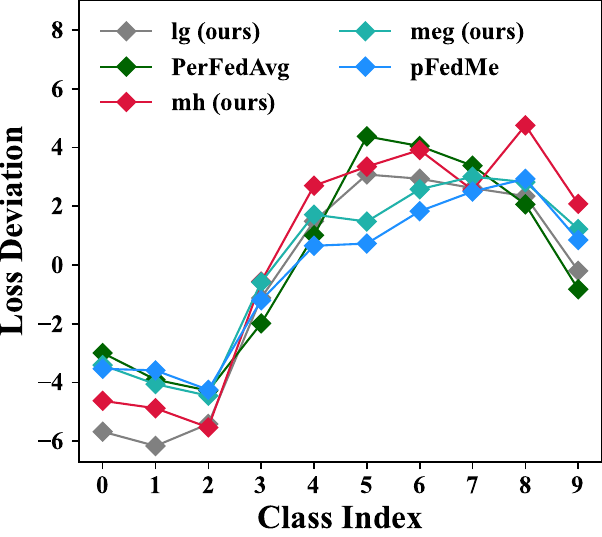}
    \includegraphics[width=0.24\textwidth, height=0.213\textwidth]{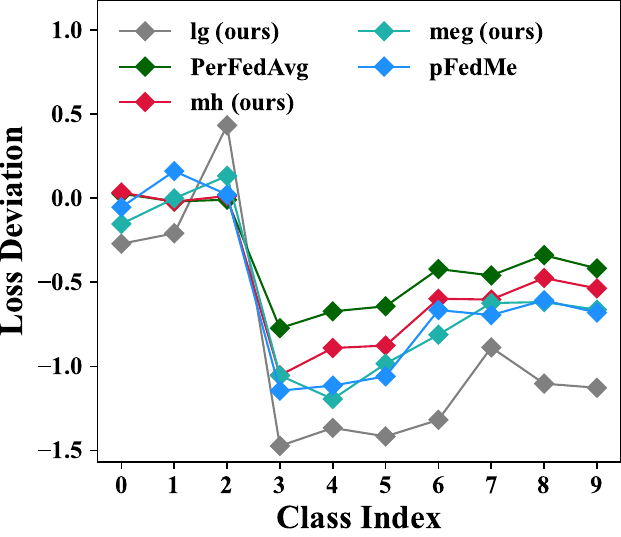}
    \includegraphics[width=0.24\textwidth, height=0.213\textwidth]{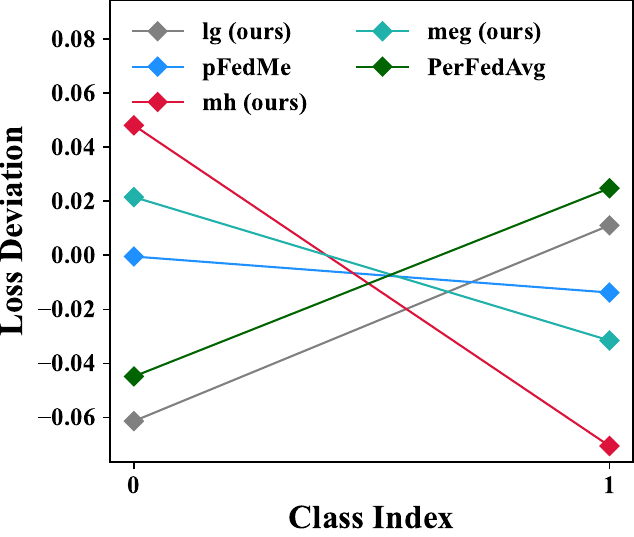}
    \includegraphics[width=0.24\textwidth, height=0.213\textwidth]{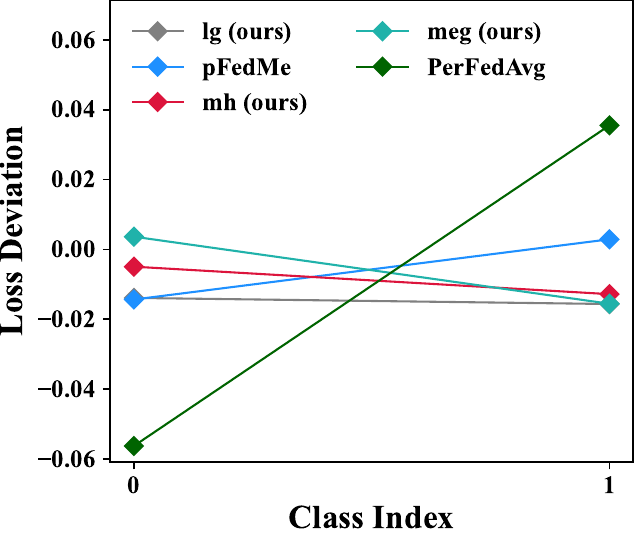}
    \begin{center}
        \footnotesize 
        CIFAR-10-CNN-G \quad\qquad CIFAR-10-CNN-L
        \qquad\qquad
        Sent140-LSTM-G \qquad\qquad Sent140-LSTM-L
    \end{center}
    \caption{The loss deviation of our experiments in Section~\ref{sec_exp} on the first client on settings: FEMNIST-DNN/MCLR, FMNIST-DNN/MCLR, CIFAR-10-CNN and Sent140-LSTM.}
    \label{appdx_fig_psnlz_com}
\end{figure*}

\subsection{Experiments about Instability and Robustness on Aggregation Noise and Data Heterogeneity}
\label{appdx_ssec_instability}

In this section, we experimentally demonstrate the instability of the global model in $\mathbf{mh}$ at small aggregation ratios by comparing the performances of clients with different aggregation numbers. Additionally, we also conduct experiments on different data heterogeneity settings.

The experimental settings in Section~\ref{ssec_exps} of the main paper have been utilized, with the exception of the client count for aggregation at the culmination of each global epoch. To ensure clarity, we present Table~\ref{tbl_aggregation} without well-tuning hyper-parameters (which are random selected in a narrow range with Gaussian variance of 0.01). Notably, supplementary experiments have been repeated 5 times to enhance the robustness of our analysis.

The results of experiments about different Non-IID settings are shown in Table~\ref{tbl_noniid}. The FMNIST in these experiments are equal number of total local data with different local data distribution the distribution are shown in Figure~\ref{appdx_fmnist_noniid}. All experiments employ full aggregation of 40 clients and only 1 local epoch to get rid of the effects from aggregation noise and client drift caused by multiple local update.

An interesting example is that if the local classes are only two classes in the case of an extremely unbalanced heterogeneous distribution, the underlying local test accuracy for a personalized model will be at least the probability of the maximum probability class being sampled, say 90\% of the first class and 10\% of the second class, then a learned knowledge model is at least 90\% accurate.

\begin{figure}
    \centering
    \includegraphics[width=0.16\textwidth]{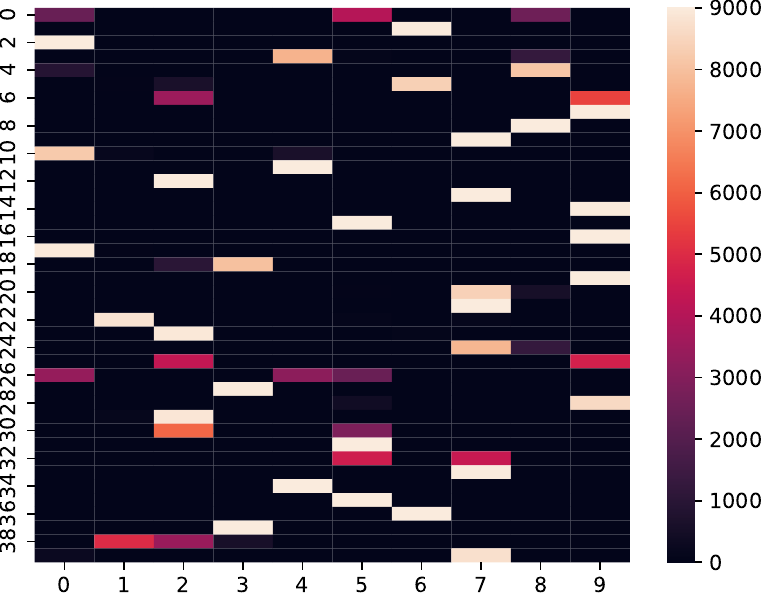}
    \includegraphics[width=0.16\textwidth]{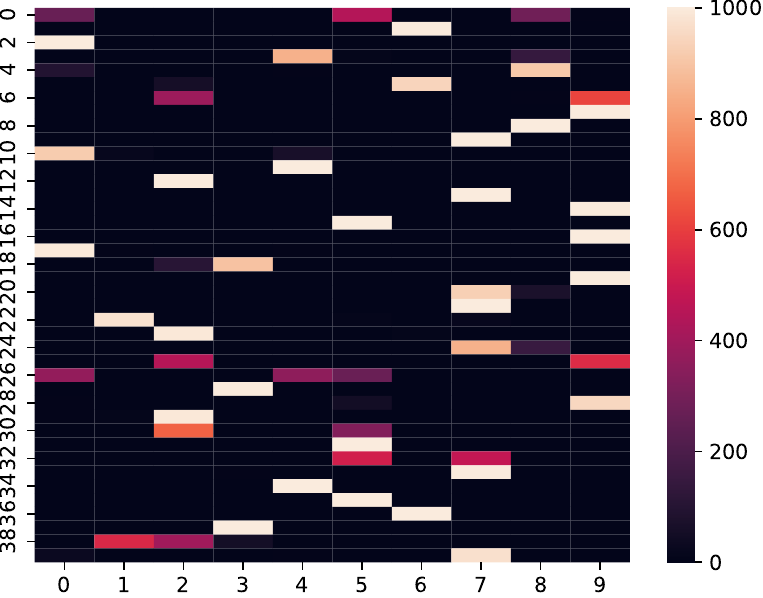}
    \includegraphics[width=0.16\textwidth]{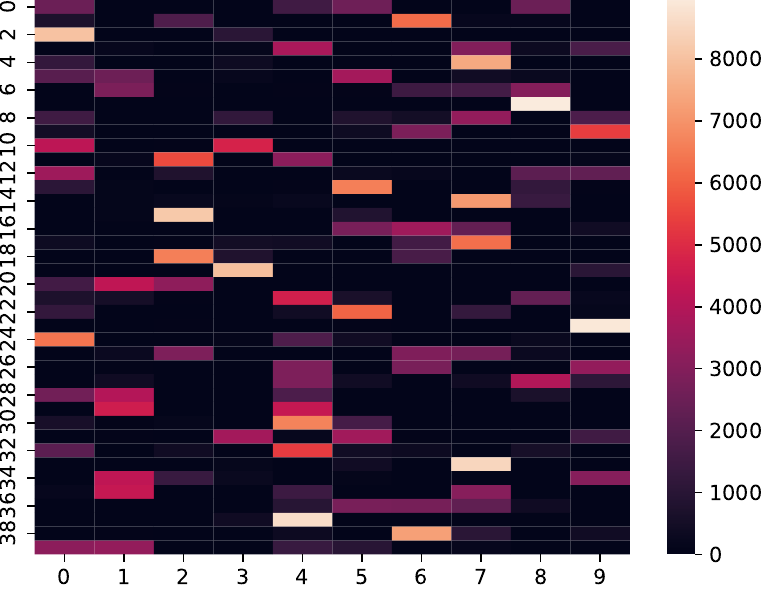}
    \includegraphics[width=0.16\textwidth]{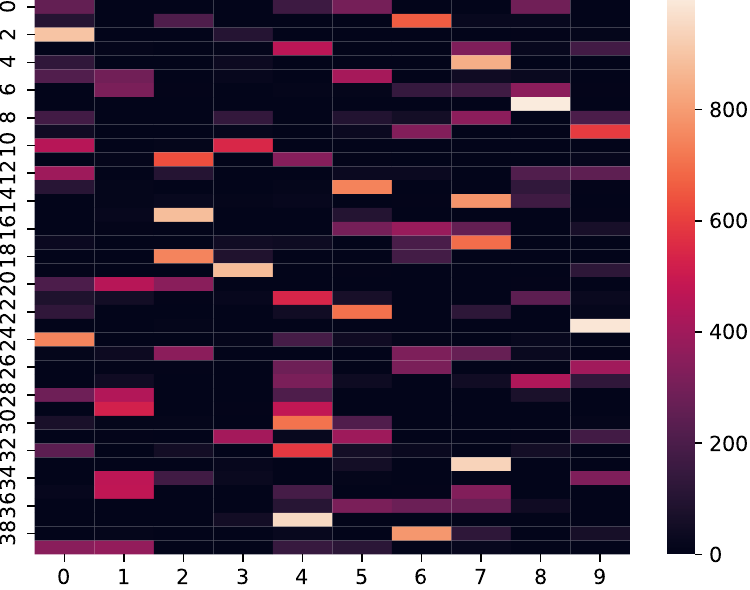}
    \includegraphics[width=0.16\textwidth]{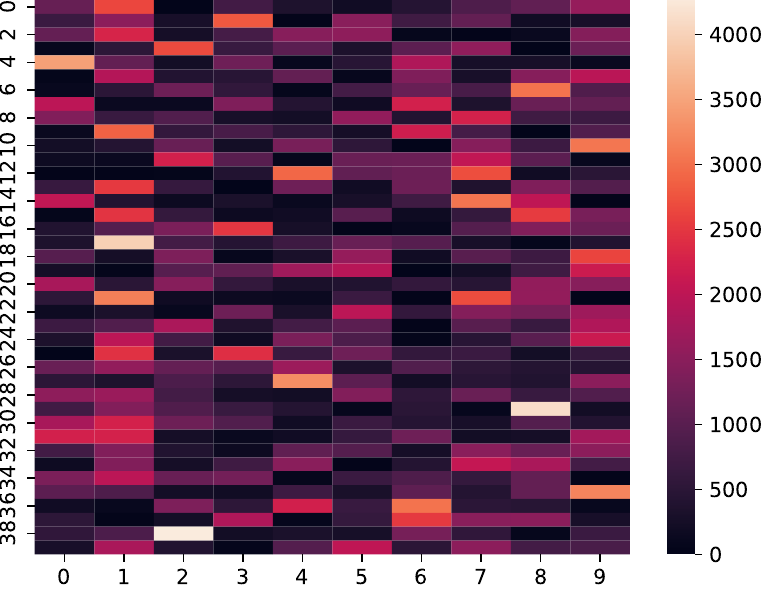}
    \includegraphics[width=0.16\textwidth]{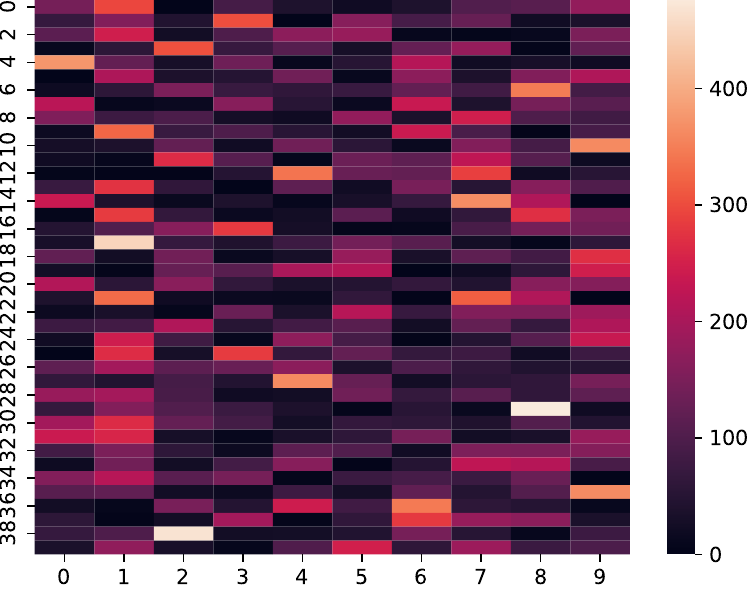}
    \begin{center}
        \footnotesize
        0.01-Test/Train \qquad\qquad\qquad
        0.1-Test/Train \qquad\qquad\qquad\quad
        1-Test/Train
    \end{center}
    \includegraphics[width=0.16\textwidth]{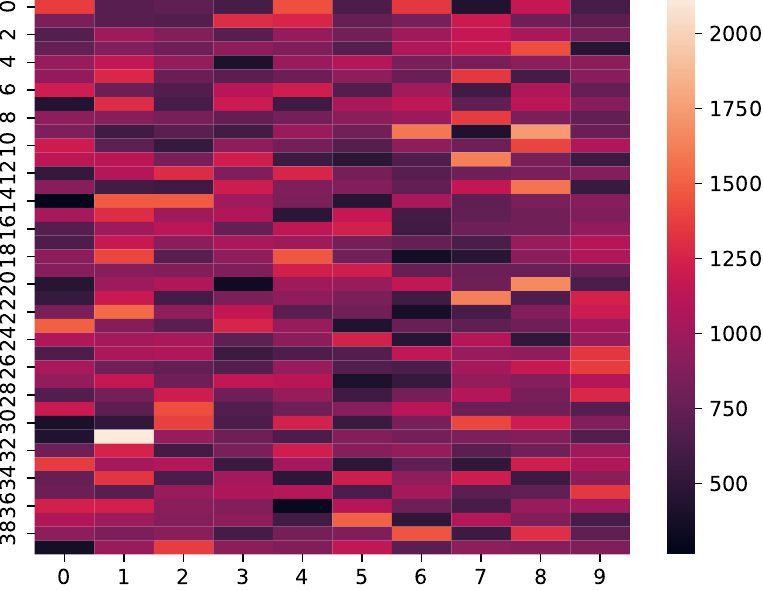}
    \includegraphics[width=0.16\textwidth]{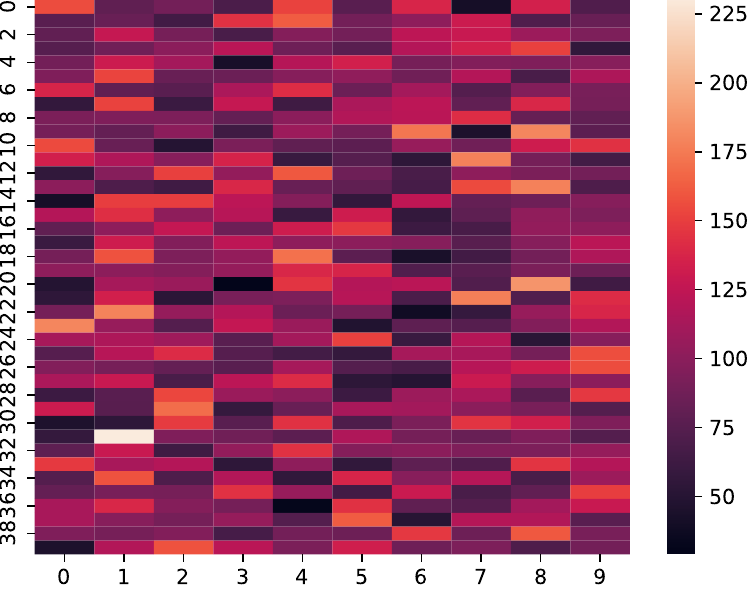}
    \includegraphics[width=0.16\textwidth]{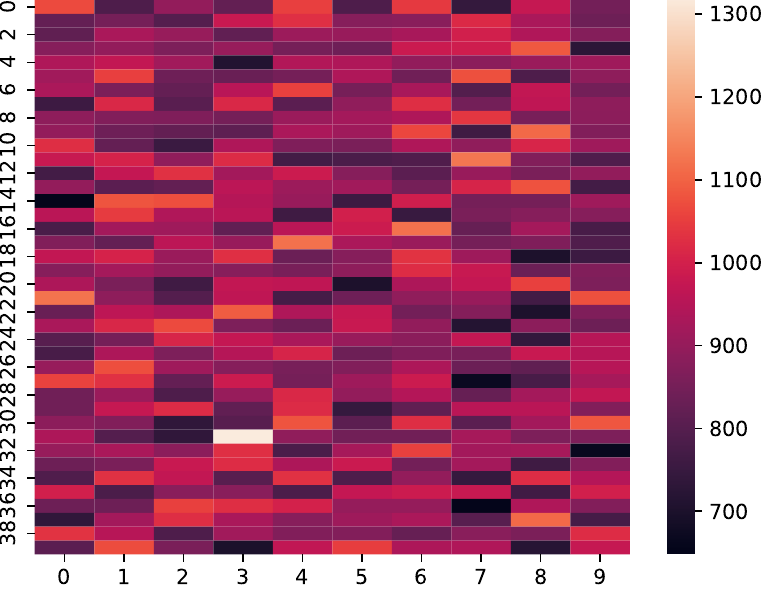}
    \includegraphics[width=0.16\textwidth]{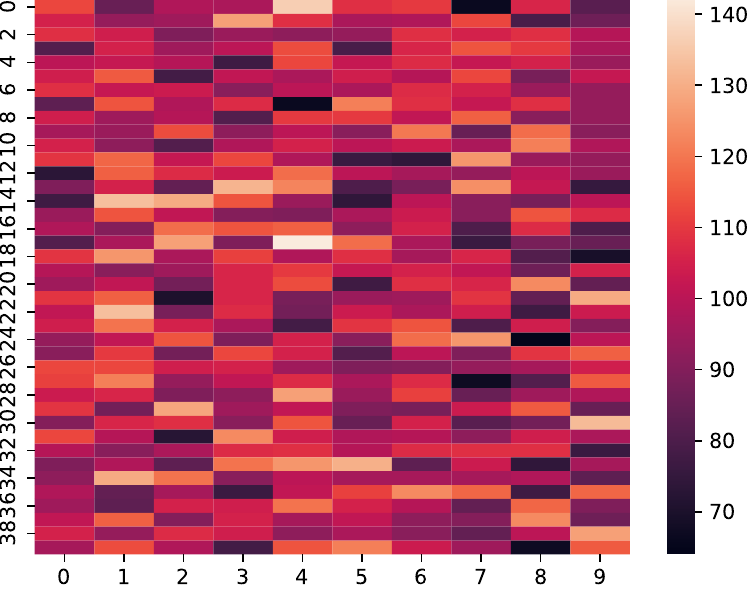}
    \includegraphics[width=0.16\textwidth]{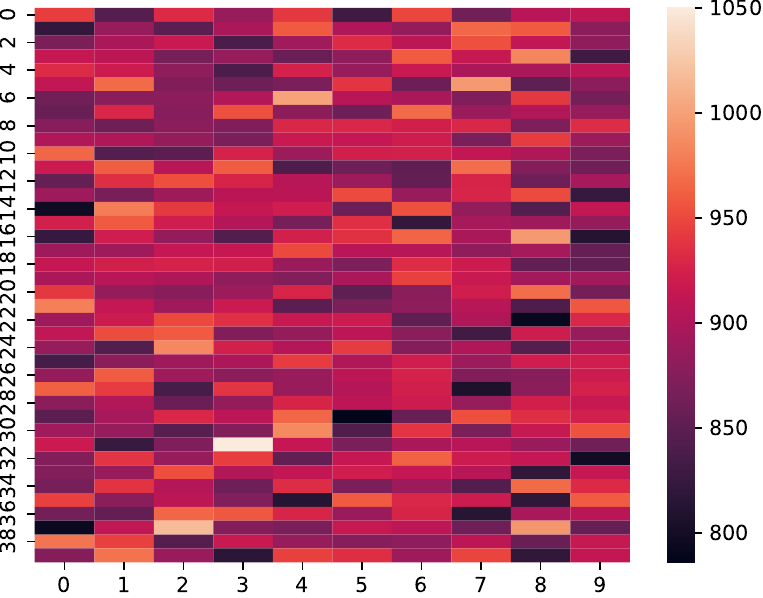}
    \includegraphics[width=0.16\textwidth]{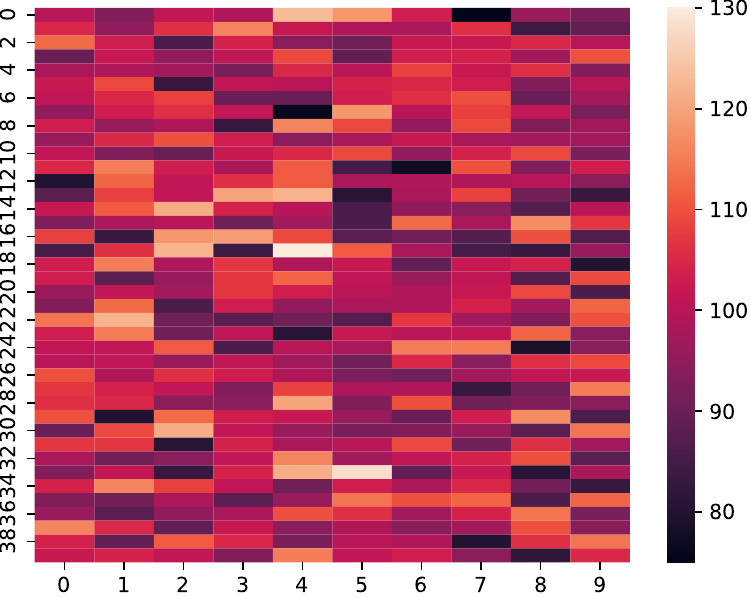}
    \begin{center}
        \footnotesize
        10-Test/Train \qquad\qquad\qquad
        100-Test/Train \qquad\qquad\qquad
        1000-Test/Train
    \end{center}
    \caption{Different heterogeneous distributions of FMNIST. The horizontal and vertical axes represent the different classes and clients respectively.}
    \label{appdx_fmnist_noniid}
\end{figure}

\subsection{Additional Experiments}
The additional experiments with more baselines are shown in Table~\ref{appdx_add_exps} with the same settings mentioned in the Table~\ref{tbl_res_lnlm}.
\begin{table}[t]
    \centering
    \caption{Additional experiments with new baselines. (accuracy)}
    \label{appdx_add_exps}
    \resizebox{.98\textwidth}{!}{
    \begin{tabular}{lccccccc}
        \hline
         Methods / Datasets \& Models & \multicolumn{2}{c}{FEMNIST / DNN} & \multicolumn{2}{c}{CIFAR-10 / CNN} & \multicolumn{2}{c}{Sent140 / LSTM} & Average Decrease by Noise \\
         Aggregation Ratio & \multicolumn{2}{c}{$10\% \rightarrow 5\%$}
         & \multicolumn{2}{c}{$20\% \rightarrow 10\%$}
         & \multicolumn{2}{c}{$40\% \rightarrow 20\%$} & -\\
         \hline
         FedPAC~\cite{xu2022personalized} & \multicolumn{2}{c}{$62.2\% \rightarrow 60.7\%$} & \multicolumn{2}{c}{$78.9\% \rightarrow 77.3\%$} & \multicolumn{2}{c}{$68.1\% \rightarrow 66.8\%$} & 1.5\% \\       FedHN~\cite{shamsian2021personalized}  & \multicolumn{2}{c}{$61.1 \% \rightarrow 59.6\%$} & \multicolumn{2}{c}{$77.5\% \rightarrow 76.9\%$} & \multicolumn{2}{c}{$71.2\% \rightarrow 70.1\%$} & 1.1\% \\
         Fedfomo~\cite{zhang2020personalized} & \multicolumn{2}{c}{$60.1 \% \rightarrow 58.9\%$} & \multicolumn{2}{c}{$71.4\% \rightarrow 70.6\%$} & \multicolumn{2}{c}{$70.1\% \rightarrow 68.9\%$} & 1.1\% \\
         Ditto~\cite{li2021ditto} & \multicolumn{2}{c}{$52.9 \% \rightarrow 52.2\%$} & \multicolumn{2}{c}{$72.4\% \rightarrow 72.1\%$} & \multicolumn{2}{c}{$71.0\% \rightarrow 70.3\%$} & 0.6\% \\
        mh(ours) & \multicolumn{2}{c}{$\textbf{64.9} \% \rightarrow \textbf{64.3}\%$} & \multicolumn{2}{c}{$\textbf{79.4}\% \rightarrow \textbf{79.1}\%$} & \multicolumn{2}{c}{$\textbf{72.0}\% \rightarrow \textbf{71.8}\%$} & \textbf{0.4}\% \\
        \hline
    \end{tabular}
    }
\end{table}

\section{Details of Theorems}
\label{appdx_sec_theo}

\subsection{Proof Sketch}
We prove the theorems primarily through two supporting lemmas. The first lemma provides the upper bound of the global iterative error, while the second lemma restricts the upper bound of the error between the actual local update and theoretical expectation.

\subsection{Related Notations}

$\cdot_{i,r}^{(t)}$ represents the $\cdot$ on $i^{th}$ client at $r^{th}$ local epoch of $t^{th}$ global epoch.

The Local Sampled Data $\tilde{d}_{i}\in d_{i}$

The Approximated Personalized Model $\tilde{\theta}_{i,r}^{(t)} := \tilde{\theta}(\mu_{i,r}^{(t)})$.

The Uniform Local Data Sampling Expectation $\mathbf{E}_{\tilde{d}_{i}}:=\frac{1}{|d_{i}|}\sum_{\tilde{d}_{i}\in d_{i}}$

The Unbiased Empirical First Moment $\mathbf{E}_{\tilde{d}_{i}}\nabla \tilde{f}_{i}(\theta;\tilde{d}_{i})=\nabla f_{i}(\theta)$

The Global Minimizer $w^{*}$.

The Local Minimizer $\theta_{i,r}^{*(t)}:=\mathcal{D}\mathbf{prox}_{g^{*},\lambda^{-1}}f_{i}(\mu_{i,r}^{(t)})$.

The Local Approximate Error $\Delta_{i,r}^{(t)}:=\tilde{\theta}_{i,r}^{(t)} -\theta_{i,r}^{*(t)}$.

The Global Approximate Squared Error $\mathbf{\Delta}^{(t)} := \mathbf{E}||w^{(t)}-w^{*}||^{2}$

The Approximated Global Gradient $\mathbf{g}_{i,r}^{(t)}=\lambda  \mathbf{D}\mu_{i,r}^{(t)} \nabla^{2}g^{*}(\mu_{i,r}^{(t)})[ \mu_{i,r}^{(t)}-\tilde{\theta}_{i,r}^{(t)}]$

The (first-order) Approximated Envelope Gradient: $\nabla \tilde{F}_{i}(w)$.

$||\cdot||_{m}$ is any matrix norm, with $||I||_{m} = \hat{u}_{m}$.

$\mathbb{I}_{E}$, indicator function on event $E$ .

The Virtual Global Gradient: $\mathbf{g}^{(t)}=\frac{1}{SR}\sum_{i\in\mathcal{S}^{(t)}}\sum_{r=1}^{R}\mathbf{g}_{i,r}^{(t)}$.

The Virtual Global Step-size: $\tilde{\alpha}_{m} = \alpha_{m} \beta R$.

The Expected Smooth~\cite{gower_sgd_2019} Coefficient of $F$ and $F_{i}$: $L_{F_{\cdot}}$ $L_{F_{i}}$.

Bounded Deviation Ratio of Strategy Disturbance Coefficient $\sigma_{\Phi}$.

\subsection{Basic Propositions}

\begin{proposition}[$\mu$-strongly convex]
\label{appdx_prop_sconv}
If $f$ is $\mu$-strongly convex, we have:
$$
\begin{aligned}
    \langle \nabla f(x) - \nabla f(y), x-y\rangle \ge \mu||x-y||^{2} \\
    ||\nabla f(x) - \nabla f(y)|| \ge \mu||x-y|| \\
\end{aligned}
$$
\end{proposition}

\begin{proposition}[$L$-smooth]
\label{appdx_prop_lsmth}
If $f$ is $L$-smooth, we have:
$$
\begin{aligned}
    \langle \nabla f(x) - \nabla f(y), x-y\rangle \le L||x-y||^{2} \\
    ||\nabla f(x) - \nabla f(y)|| \le L||x-y|| \\
    ||\nabla f(x) - \nabla f(y)||^{2} \le 2L \mathcal{D}_{f}(x,y)\\
\end{aligned}
$$
\end{proposition}

\begin{proposition}[Jensen's inequality]
\label{appdx_prop_jensen}
If $f$ is convex, we have:
$$
\begin{aligned}
    \mathbf{E}_{X}f(X) \ge f(\mathbf{E}_{X}X)
\end{aligned}
$$.
A variant of the general one shown above:
$$
\begin{aligned}
    ||\sum_{i=1}^{\mathcal{N}}x_{i}||^{2} \le 
 \mathcal{N}\sum_{i}^{\mathcal{N}}||x_{i}||^{2}
\end{aligned}
$$.
\end{proposition}
\begin{proposition}[triangle inequality]
\label{appdx_prop_tri}
    The triangle inequality:
    $$||A+B||\le ||A||+||B||$$
\end{proposition}

\begin{proposition}[matrix norm compatibility]
\label{appdx_prop_mnc}
     The matrix norm compatibility, $A\in \mathbf{R}^{a\times b}, B\in \mathbf{R}^{b\times c}, v\in \mathbf{R}^{b}$:
     $$
     \begin{aligned}
        ||AB||_{m}\le ||A||_{m}||B||_{m} \\
        ||Av||_{m}\le ||A||_{m}||v||
     \end{aligned}
     $$
\end{proposition}

\begin{proposition}[Peter Paul inequality]
\label{appdx_prop_ppi}
    $$ 2 \langle x, y \rangle \le \frac{1}{\epsilon}||x||^{2} + \epsilon ||y||^{2}$$
\end{proposition}

\subsection{General Assumptions for Analysis}
\label{appdx_ssec_GAss}
\begin{assumption}[Prior selection]
\label{assumption_ps}
    The given $g^{*}$ is $\hat{\mu}_{g^{*}}$-strongly convex and $\hat{L}_{g^{*}}$-smooth: $\hat{\mu}_{g^{*}}||x-y|| \le ||\nabla g^{*}(x) - \nabla g^{*}(y)||<\hat{L}_{g^{*}}||x-y||$. and $||\nabla^{2}g^{*}(\cdot)||_{m}\le \hat{L}_{g^{*}}$ (Examples are in Appendix~\ref{appdx_ssec_foMtd}).
\end{assumption}
\begin{assumption}[Smooth envelope assumption]
\label{assumption_se}
    For each local envelope $\mathcal{E}_{i}(\cdot)  = [F_{i}\circ\mu_{i}^{-1}](\cdot) = \mathcal{D}\mathbf{env}_{g^{*},\lambda^{-1}}(\cdot)$), we have $||\nabla\mathcal{E}_{i}(x)-\nabla\mathcal{E}_{i}(y)||^{2}\le 2\hat{L}_{\mathcal{E}_{i}}\mathcal{D}_{\mathcal{E}_{i}}(x,y)$, note that $\mathcal{E}_{i}$ is convex, $\mathcal{D}_{\mathcal{E}_{i}}(x,y):=\mathcal{E}_{i}(x)-\mathcal{E}_{i}(y)-\langle \nabla \mathcal{E}_{i}(y), x-y\rangle$. For simplification, we take $\hat{L}_{\mathcal{E}}:=\max{\hat{L}_{\mathcal{E}_{i}}}, \forall i$ and bounded difference on optimal point $0 \le \frac{\mathcal{D}_{\mathcal{E}_{i}}(\mu_{i}(w),\mu_{i}(w^{*}))}{\mathcal{D}_{F_{i}}(w,w^{*})} \le \tau$.
\end{assumption}
\begin{assumption}[Strongly convex envelope settings]
\label{assumption_sces}
    $f_{i}$ is $\hat{\mu}_{f_{i}}$-strongly convex: $\hat{\mu}_{f_{i}}||x-y|| \le ||\nabla f_{i}(x) - \nabla f_{i}(y)||$, $\hat{\mu}_{f}=\min_{i} \hat{\mu}_{f_{i}}, \forall i$;
    $f_{i}$ is $\hat{L}_{f_{i}}$-smooth and non-convex : $\hat{L}_{f_{i}}||x-y|| \ge ||\nabla f_{i}(x) - \nabla f_{i}(y)||$,, $\hat{L}_{f}=\max_{i} \hat{L}_{f_{i}}, \forall i$. Therefore, we have $F_{i}$ is $\hat{\mu}_{F_{sc}}:=\lambda\hat{\mu}_{g}+\hat{\mu}_{f}$-strongly convex or $\hat{\mu}_{F_{nc}}:=\lambda\hat{\mu}_{g}-\hat{L}_{f}$-strongly convex, by tuning $\lambda$ to make $\lambda\hat{\mu}_{g}-\hat{L}_{f} > 0$. We use $\hat{\mu}_{F_{\cdot}}$ as the unified notation for both, for simplification.
\end{assumption}

\begin{assumption}[Bounded local error]
\label{assumption_ble}
     Since classical gradient descent is used locally, we assume a unified local error bound, $\forall (i,r,t), ||\nabla {f}_{i}(\tilde{\theta}_{i,r}^{(t)};d_{i})+\lambda\nabla \mathcal{D}_{g^{*}}(\tilde{\theta}_{i,r}^{(t)},\mu_{i})|| \le c_{i,r}^{(t)} \le \hat{\epsilon}, \forall i$, and a local data sampling shift variance bound $\forall \theta, \mathbf{d}\in d_{i}, \mathbf{E}_{\mathbf{d}}||\nabla \tilde{f}_{i}(\theta;\mathbf{d})-\nabla f_{i}(\theta;d_{i})|| \le \gamma_{f_{i}} \le \hat{\gamma}_{f}:=\max\{\gamma_{f_{i}}\}, \forall i$.
\end{assumption}
\begin{assumption}[RMD meta-step function bound]
\label{appdx_assumption_rmd_msb}
    $\forall i, \Phi_{i}$ with limited gradient, $||\nabla\Phi_{i}(\cdot)|| \le \mathcal{G}_{\Phi}$, and Hessian$||\nabla^{2}\Phi_{i}(w)||_{m}\le \hat{\gamma}_{\Phi}$, therefore, $||\mathbf{D}\mu_{i}(w)||_{m}=||I-\eta\nabla^{2}\Phi_{i}(w)||_{m} \le \hat{u}_{m} + \eta\hat{\gamma}_{\Phi}$.
\end{assumption}

\begin{assumption}[Bounded deviation ratio of strategy disturbance]
\label{appdx_assumption_bdrsd}
We assume the local training is not affected too much by the personalized prior strategies, which means we don't want a large discrepancy between the results of local strategies formulation and the calculation of local envelope gradients given the prior on each client, which may cause a significant disturbance in the local optimization objective due to the haphazard formulation of prior strategies.
Given $\forall, w, w^{\prime}$, we have:
$$
\frac{||\mathbf{D}\mu_{i}(w)-\mathbf{D}\mu_{i}(w^{\prime})||_{m}}{||\nabla \mathcal{E}(\mu_{i}(w)) - \nabla \mathcal{E}(\mu_{i}(w^{\prime}))||} \le \sigma_{\Phi} \frac{\max\{||\mathbf{D}\mu_{i}(w)||_{m},||\mathbf{D}\mu_{i}(w^{\prime})||_{m}\}}{\max\{||\nabla \mathcal{E}(\mu_{i}(w))||, ||\nabla \mathcal{E}(\mu_{i}(w^{\prime}))||\}}
$$
    
\end{assumption}

\begin{assumption}[Optimal global gradient noise bound]
\label{appdx_assumption_oggnb}
$||\nabla F_{i}(w^{*})||^{2} \le \sigma_{F_{i},*}^{2}$, let $\sigma_{F,*}^{2} = \max_{i} \sigma_{F_{i},*}^{2},\forall i$.
\end{assumption}

\begin{assumption}[First-order approximate bound]
\label{appdx_assumption_fo}
$||\nabla F_{i}(w) - \nabla \tilde{F}_{i}(w)|| \le \epsilon_{1}$    
\end{assumption}

\subsection{General Lemmas}

\begin{lemma}[Local Samplng Proximal Bound]
\label{appdx_lemma_lspb}
Under settings and assumptions in Section~\ref{sec_alg} and Section~\ref{appdx_ssec_GAss}, if $f$ is $\hat{\mu}_{f}$-strongly convex, $\mathbf{E}_{\tilde{d}_{i}}||\Delta_{i,r}^{(t)}||^{2} \le \frac{2}{(\hat{\mu}_{f}+\lambda\hat{\mu}_{g^{*}})^{2}}[\frac{\hat{\gamma}_{f}^{2}}{|\tilde{d}_{i}|}+\hat{\epsilon}^2]$ holds; if $f$ is $\hat{L}_{f}$-smooth and non-convex, $\mathbf{E}_{\tilde{d}_{i}}||\Delta_{i,r}^{(t)}||^{2} \le \frac{2}{(\lambda\hat{\mu}_{g^{*}}-\hat{L}_{f})^{2}}[\frac{\hat{\gamma}_{f}^{2}}{|\tilde{d}_{i}|}+\hat{\epsilon}^2]$ holds, such that:
$$
\mathbf{E}_{\tilde{d}_{i}}||\Delta_{i,r}^{(t)}||^{2} \le \frac{2}{\hat{\mu}_{F_{\cdot}}^{2}}[\frac{\hat{\gamma}_{f}^{2}}{|\tilde{d}_{i}|}+\hat{\epsilon}^2]
$$
\end{lemma}
\begin{proof}
    With Proposition~\ref{appdx_prop_sconv}, Assumption~\ref{assumption_sces} and optimal condition of $F_{i}(\mu_{i,r}^{(t)})$ on $\theta_{i,r}^{*(t)}$, we have:
    $$
    \begin{aligned}
        ||\Delta_{i,r}^{(t)}||^{2} &=||\tilde{\theta}_{i,r}^{(t)} -\theta_{i,r}^{*(t)}||^{2}
        \le \frac{1}{\mu_{F_{\cdot}}^{2}}||\mathbf{g}_{i,r}^{(t)}||^{2} \\
    \end{aligned}
    $$
    Note that, $\mathbf{g}_{i,r}^{(t)} = \nabla \tilde{f}_{i}(\tilde{\theta}_{i,r}^{(t)};\tilde{d}_{i}) +\lambda \nabla \mathcal{D}_{g^{*}}(\tilde{\theta}_{i,r}^{(t)}, \mu_{i,r}^{(t)})$. With Proposition~\ref{appdx_prop_jensen} and Assumption~\ref{assumption_ble}, we have:
    
    $$
    \begin{aligned}
    ||\mathbf{g}_{i,r}^{(t)}||^{2}& = \nabla \tilde{f}_{i}(\tilde{\theta}_{i,r}^{(t)};\tilde{d}_{i})-\nabla f_{i}(\tilde{\theta}_{i,r}^{(t)};d_{i}) + \nabla {f}_{i}(\tilde{\theta}_{i,r}^{(t)};d_{i})+\lambda\nabla \mathcal{D}_{g^{*}}(\tilde{\theta}_{i,r}^{(t)},\mu_{i}) \\
    \le & 2\{||\nabla \tilde{f}_{i}(\tilde{\theta}_{i,r}^{(t)};\tilde{d}_{i})-\nabla f_{i}(\tilde{\theta}_{i,r}^{(t)};d_{i})||^{2} + ||\nabla {f}_{i}(\tilde{\theta}_{i,r}^{(t)};d_{i})+\lambda\nabla \mathcal{D}_{g^{*}}(\tilde{\theta}_{i,r}^{(t)},\mu_{i})||^{2}\} \\
    \le & 2\{||\nabla \tilde{f}_{i}(\tilde{\theta}_{i,r}^{(t)};\tilde{d}_{i})-\nabla f_{i}(\tilde{\theta}_{i,r}^{(t)};d_{i})||^{2}+\hat{\epsilon}^2\}
    \end{aligned}
    $$
    Taking expectation on both sides, 
    combining both inequalities above, we have:
    $$
    \begin{aligned}
        \mathbf{E}_{\tilde{d}_{i}}||\Delta_{i,r}^{(t)}||^{2} \le & 2\{\frac{1}{|\tilde{d}_{i}|^{2}}\mathbf{E}_{\mathbf{d}}||\sum_{\mathbf{d}\in \tilde{d}_{i}}\nabla \tilde{f}_{i}(\tilde{\theta}_{i,r}^{(t)};\mathbf{d})-\nabla f_{i}(\tilde{\theta}_{i,r}^{(t)};d_{i})||^{2} + \hat{\epsilon}^2\} \\
        \le & 2\{\frac{1}{|\tilde{d}_{i}|^{2}}\sum_{\mathbf{d}\in \tilde{d}_{i}}\mathbf{E}_{\mathbf{d}}||\nabla \tilde{f}_{i}(\tilde{\theta}_{i,r}^{(t)};\mathbf{d})-\nabla f_{i}(\tilde{\theta}_{i,r}^{(t)};d_{i})||^{2} + \hat{\epsilon}^2\} \\
        \le & 2[\frac{\hat{\gamma}_{f}^{2}}{|\tilde{d}_{i}|}+\hat{\epsilon}^2]
    \end{aligned}
    $$
\end{proof}

\begin{lemma}[Expected-Smooth Personalized Local Object]
\label{appdx_lemma_esplo}
Under settings and assumptions in Section~\ref{sec_alg} and Section~\ref{appdx_ssec_GAss}, the personalized local objective function is expected-smooth, such that:
$$
\begin{aligned}
    ||\nabla F_{i}(w)-\nabla F_{i}(w^{\prime})|| &\le (1+\sigma_{\Phi})(\hat{u}_{m}+\eta\hat{\gamma}_{\Phi})||\nabla \mathcal{E}_{i}(\mu_{i}(w))-\nabla \mathcal{E}_{i}(\mu_{i}(w^{\prime}))||, \forall w, w^{\prime}; \\
    ||\nabla F_{i}(w)-\nabla F_{i}(w^{*})||^{2} &\le 2(1+\sigma_{\Phi})^{2}(\hat{u}_{m}+\eta\hat{\gamma}_{\Phi})^{2}\hat{L}_{\mathcal{E}_{i}}\mathcal{D}_{\mathcal{E}}(\mu_{i}(w),\mu_{i}(w^{*})) \le 2\hat{L}_{F_{i}}\mathcal{D}_{F_{i}}(w,w^{*}); \\
    \mathbf{E}_{i}||\nabla F_{i}(w)-\nabla F_{i}(w^{*})||^{2} &\le 2\hat{L}_{F}[F(w)-F(w^{*})] = 2\hat{L}_{F}\mathcal{D}_{F}(w,w^{*}), \\
\end{aligned}
$$
where $\hat{L}_{F_{i}}:=\tau(1+\sigma_{\Phi})^{2}(\hat{u}_{m}+\eta\hat{\gamma}_{\Phi})^{2}\hat{L}_{\mathcal{E}_{i}}$ and $\hat{L}_{F}=\max{\hat{L}_{F_{i}}},\forall i$.
\end{lemma}
\begin{proof}
    With Assumption~\ref{appdx_assumption_rmd_msb} and Assumption~\ref{appdx_assumption_bdrsd}, we have:
    $$
    \begin{aligned}
        &||\nabla F_{i}(w)-\nabla F_{i}(w^{\prime})||^{2} =  ||\mathbf{D}\mu_{i}(w)\nabla \mathcal{E}(\mu_{i}(w)) - \mathbf{D}\mu_{i}(w^{\prime})\nabla\mathcal{E}(\mu_{i}(w^{\prime}))|| \\
        = & ||\mathbf{D}\mu_{i}(w)\nabla \mathcal{E}(\mu_{i}(w)) - \mathbf{D}\mu_{i}(w^{\prime})\nabla\mathcal{E}(\mu_{i}(w^{\prime})) + \mathbf{D}\mu_{i}(w)\nabla \mathcal{E}(\mu_{i}(w^{\prime})) - \mathbf{D}\mu_{i}(w)\nabla \mathcal{E}(\mu_{i}(w^{\prime}))|| \\
        \le & ||\mathbf{D}\mu_{i}(w)[\nabla \mathcal{E}(\mu_{i}(w)) -\nabla \mathcal{E}(\mu_{i}(w^{\prime}))]|| + ||[\mathbf{D}\mu_{i}(w^{\prime}) - \mathbf{D}\mu_{i}(w)]\nabla \mathcal{E}(\mu_{i}(w^{\prime}))||\\
        \le & ||\mathbf{D}\mu_{i}(w)||_{m}||[\nabla \mathcal{E}(\mu_{i}(w)) -\nabla \mathcal{E}(\mu_{i}(w^{\prime}))]|| + ||[\mathbf{D}\mu_{i}(w^{\prime}) - \mathbf{D}\mu_{i}(w)]||_{m}||\nabla \mathcal{E}(\mu_{i}(w^{\prime}))||\\
        \le & \max\{||\mathbf{D}\mu_{i}(w)||, ||\mathbf{D}\mu_{i}(w^{\prime})||\}||\nabla \mathcal{E}(\mu_{i}(w)) - \nabla\mathcal{E}(\mu_{i}(w^{\prime}))|| \\
        & + \max\{||\nabla \mathcal{E}(\mu_{i}(w))||, ||\nabla \mathcal{E}(\mu_{i}(w^{\prime}))||\}|[\mathbf{D}\mu_{i}(w^{\prime}) - \mathbf{D}\mu_{i}(w)]||_{m}\\\le & \max\{||\mathbf{D}\mu_{i}(w)||, ||\mathbf{D}\mu_{i}(w^{\prime})||\}||\nabla \mathcal{E}(\mu_{i}(w)) - \nabla\mathcal{E}(\mu_{i}(w^{\prime}))|| \\
        & + \sigma_{\Phi} \max\{||\mathbf{D}\mu_{i}(w)||_{m},||\mathbf{D}\mu_{i}(w^{\prime})||_{m}\}||\nabla \mathcal{E}(\mu_{i}(w)) - \nabla\mathcal{E}(\mu_{i}(w^{\prime}))||\\
        \le & (1+\sigma_{\Phi})(\hat{u}_{m}+\eta\hat{\gamma}_{\Phi})||\nabla \mathcal{E}(\mu_{i}(w)) - \nabla\mathcal{E}(\mu_{i}(w^{\prime}))||\\
    \end{aligned}
    $$
    where the first two inequalities is by Proposition~\ref{appdx_prop_tri} and Proposition~\ref{appdx_prop_mnc}.
    
    With the first inequality in our lemma is proven. With the proven one and Assumption~\ref{assumption_se}, we have:
    $$
    \begin{aligned}
        ||\nabla F_{i}(w)-\nabla F_{i}(w^{*})||^{2} \le & 2(1+\sigma_{\Phi})^{2}(\hat{u}_{m}+\eta\hat{\gamma}_{\Phi})^{2}\hat{L}_{\mathcal{E}_{i}}\mathcal{D}_{\mathcal{E}}(\mu_{i}(w),\mu_{i}(w^{*}))\\
        \le & 2(1+\sigma_{\Phi})^{2}(\hat{u}_{m}+\eta\hat{\gamma}_{\Phi})^{2}\hat{L}_{\mathcal{E}_{i}} \tau\mathcal{D}_{F_{i}}(w,w^{*})\\
        \le & 2\hat{L}_{F_{i}}\mathcal{D}_{F_{i}}(w,w^{*}); \\
    \mathbf{E}_{i}||\nabla F_{i}(w)-\nabla F_{i}(w^{*})||^{2} \le & 2\hat{L}_{F}[F(w)-F(w^{*})] = 2\hat{L}_{F}\mathcal{D}_{F}(w,w^{*})
    \end{aligned}
    $$
    where the client sampling expectation is taken in the final inequality.
\end{proof}

\begin{lemma}[RMD Personalized Prior Bound]
Under settings and assumptions in Section~\ref{sec_alg} and Section~\ref{appdx_ssec_GAss}, the relationship between $||\nabla F_{i}(w)||$ and $||\nabla \mathcal{E}_{i}(\mu_{i}(w))||$ is:
$$||\nabla F_{i}(w)|| \le (\hat{u}_{m} +\eta\hat{\gamma}_{\Phi})||\nabla \mathcal{E}_{i}(\mu_{i}(w))|| \le \lambda \hat{L}_{g^{*}}(\hat{u}_{m} +\eta\hat{\gamma}_{\Phi})||\mu_{i}(w)-\theta_{i}^{*}||$$
\end{lemma}
\begin{proof}
    Applying Proposition~\ref{appdx_prop_mnc} and Assumption~\ref{appdx_assumption_rmd_msb}, $\nabla F_{i}(w) = \mathbf{D}\mu_{i}(w)\nabla \mathcal{E}_{i}(\mu_{i}(w))$, it's easy to prove the first inequality.
    Rewriting $\nabla\mathcal{E}_{i}(\mu_{i}(w))$ in detail as shown following, applying Proposition~\ref{appdx_prop_mnc} and Assumption~\ref{assumption_ps}, the final inequality is proven:
    $$\nabla\mathcal{E}_{i}(\mu_{i}(w)) = \lambda \nabla^{2}g^{*}(\mu_{i}(w)) [\mu_{i}(w)-\theta_{i}^{*}]$$
\end{proof}

\begin{lemma}[Local Objective's Client Sampling Error Bound]
\label{appdx_lemma_locseb}
Under settings and assumptions in Section~\ref{sec_alg} and Section~\ref{appdx_ssec_GAss}, the upper bound of local sampling error is:
$$\mathbf{E}_{\mathcal{S}^{t}}||\frac{1}{S}\sum_{i\in \mathcal{S}^{(t)}}\nabla F_{i}(w^{{(t)}})-\nabla F(w^{(t)})||^{2} \le \frac{N/S-1}{N-1}\sum_{i}^{N}\frac{1}{N}||\nabla F_{i}(w^{(t)})-\nabla F(w^{(t)})||^{2}$$
, where $|\mathcal{S}^{(t)}| = S, \forall t$.
\end{lemma}
\begin{proof}
    This lemma is the same lemma in~\cite{li_convergence_2023, t2020personalized}.
    $$
    \begin{aligned}
        \mathbf{E}_{\mathcal{S}^{t}}&||\frac{1}{S}\sum_{i\in \mathcal{S}^{(t)}}\nabla F_{i}(w^{{(t)}})-\nabla F(w^{(t)})||^{2} = \frac{1}{S^{2}}\mathbf{E}_{\mathcal{S}^{(t)}}||\sum_{i\in [N]} \mathbb{I}_{i\in \mathcal{S}^{(t)}}\nabla F_{i}(w^{{(t)}})-\nabla F(w^{(t)})||^{2} \\
        = & \frac{1}{S^{2}}[\sum_{i\in [N]}\mathbf{E}_{\mathcal{S}^{(t)}}[\mathbb{I}_{i\in \mathcal{S}^{(t)}}]||\nabla F_{i}(w^{{(t)}})-\nabla F(w^{(t)})||^{2} \\
        & +\sum_{i\neq j}\mathbf{E}_{\mathcal{S}^{(t)}}[\mathbb{I}_{i\in \mathcal{S}^{(t)}},\mathbb{I}_{j\in \mathcal{S}^{(t)}}]\langle \nabla F_{i}(w^{(t)})-\nabla F(w^{(t)}), \nabla F_{j}(w^{(t)})-\nabla F(w^{(t)})\rangle] \\
        = & \frac{1}{SN}\sum_{i}^{N}||\nabla F_{i}(w^{{(t)}})-\nabla F(w^{(t)})||^{2} \\
        & + \sum_{i\neq j}\frac{S-1}{SN(N-1)}\langle \nabla F_{i}(w^{(t)})-\nabla F(w^{(t)}), \nabla F_{j}(w^{(t)})-\nabla F(w^{(t)})\rangle \\
        = & \frac{1}{SN}(1-\frac{S-1}{N-1})\sum_{i\in [N]}||\nabla F_{i}(w^{(t)})-\nabla F(w^{(t)})||^{2} \\
        = & \frac{N/S-1}{N-1}\sum_{i\in [N]}\frac{1}{N}||\nabla F_{i}(w^{(t)})-\nabla F(w^{(t)})||^{2}
    \end{aligned}
    $$
    where $\mathbb{I}_{\cdot}\in \{0,1\}$ is indicator function, $\mathbf{E}_{\mathcal{S}^{(t)}}[\mathbb{I}_{i\in \mathcal{S}^{(t)}}]=\frac{S}{N}$ and $\mathbf{E}_{\mathcal{S}^{(t)}}[\mathbb{I}_{i\in \mathcal{S}^{(t)}},\mathbb{I}_{j\in \mathcal{S}^{(t)}}] = \frac{S(S-1)}{N(N-1)},\forall i\neq j$. Note that:
    $$
    \sum_{i}^{N}||\nabla F_{i}(w^{{(t)}})-\nabla F(w^{(t)})||^{2} + \sum_{i\neq j}\langle \nabla F_{i}(w^{(t)})-\nabla F(w^{(t)}), \nabla F_{j}(w^{(t)})-\nabla F(w^{(t)})\rangle = 0.
    $$
\end{proof}

\begin{lemma}[Variance of Global Aggregation on Client Sampling Bound]
\label{appdx_lemma_vgacsb}
Under settings and assumptions in Section~\ref{sec_alg} and Section~\ref{appdx_ssec_GAss}, the upper bound of gradient aggregation variance is:
$$
\begin{aligned}
\mathbf{E}_{i}||\nabla F_{i}(w)-\nabla F(w)||^{2} \le \mathbf{E}_{i}||\nabla F_{i}(w)||^{2} \le  4\hat{L}_{F}\mathcal{D}_{F}(w,w^{*}) + 2{\sigma_{F,*}}^{2}
\end{aligned}
$$
\end{lemma}
\begin{proof}
    $$
    \begin{aligned}
    \mathbf{E}_{i}||\nabla F_{i}(w)-\nabla F(w)||^{2} \le & \mathbf{E}_{i}||\nabla F_{i}(w)||^{2} \le 2\mathbf{E}_{i}[||\nabla F_{i}(w) - \nabla F_{i}(w^{*})||^{2} + ||\nabla F_{i}(w^{*})||^{2}] \\
    \le & 4\hat{L}_{F}\mathcal{D}_{F}(w,w^{*}) + 2{\sigma_{F,*}}^{2}
    \end{aligned}
    $$
    where the first inequality is by $\mathbf{E}[||X||^{2}]=\mathbf{E}[||X-\mathbf{E}[X]||^{2}] + \mathbf{E}[||X||]^{2}$, the second one is by Proposition~\ref{appdx_prop_jensen} and the final one is by Lemma~\ref{appdx_lemma_esplo} and Assumption~\ref{appdx_assumption_oggnb}.
\end{proof}

\subsection{Supporting Lemmas}

\begin{lemma}[Global Iteration Bound]
\label{appdx_lemma_gib}
Under settings and assumptions in Section~\ref{sec_alg} and Section~\ref{appdx_ssec_GAss}, the upper bound of global iteration error is:
$$
\begin{aligned}
\mathbf{E}_{\cdot|t}||w^{(t+1)}-w^{*}||^{2}\le & (1-\frac{\tilde{\alpha}_{m}\hat{\mu}_{F_{\cdot}}}{2})||w^{(t)}-w^{*}||^{2} + \frac{3\tilde{\alpha}_{m}^{2}+2\tilde{\alpha}_{m}/\hat{\mu}_{F_{\cdot}}}{NR}\sum_{i,r}^{N,R}||\mathbf{g}_{i,r}^{(t)}-\nabla F_{i}(w^{(t)})|| \\
& + 3\tilde{\alpha}_{m}^{2}\mathbf{E}_{\cdot|t}||\frac{1}{S}\sum_{i\in \mathcal{S}^{(t)}}\nabla F_{i}(w^{(t)})-\nabla F(w^{(t)})||^{2} \\
& + (6\tilde{\alpha}_{m}^{2}\hat{L}_{F} - 2\tilde{\alpha}_{m})\mathbf{E}\mathcal{D}_{F}(w^{(t)},w^{*}) \\
\end{aligned}
$$
\end{lemma}
\begin{proof}
To separate the norm, we have:
$$\begin{aligned}
\mathbf{E}_{\cdot|t}||w^{(t+1)}-w^{*}||^{2} &= \mathbf{E}_{ \cdot | t}[||w^{(t)}-\tilde{\alpha}_{m}\mathbf{g}^{(t)}-w^{*}||^{2}] \\
&= ||w^{(t)}-w^{*}||^{2} - 2\tilde{\alpha}_{m}\mathbf{E}_{ \cdot | t}[\langle \mathbf{g}^{(t)}, w^{(t)}-w^{*} \rangle] + \tilde{\alpha}_{m}^{2}\mathbf{E}_{\cdot | t}[||\mathbf{g}^{(t)}||^{2}]
\end{aligned}$$

The second term:
$$
\begin{aligned}
-2\tilde{\alpha}_{m}\mathbf{E}_{\cdot|t}&[\langle \mathbf{g}^{((t)}, w^{(t)}-w^{*}\rangle] = -2\tilde{\alpha}_{m}\langle \mathbf{E}_{\cdot|t}\mathbf{g}^{((t)}, w^{(t)}-w^{*}\rangle \\
&= -2\tilde{\alpha}_{m}\frac{1}{NR}\sum_{i,r}^{N,R}[ \langle \mathbf{g}_{i,r}^{(t)} - \nabla F_{i}(w^{(t)}), w^{(t)} - w^{*} \rangle + \langle \nabla F_{i}(w^{(t)}), w^{(t)} - w^{*} \rangle ] \\
&= \frac{\tilde{\alpha}_{m}}{NR}\sum_{i,r}^{N,R}[ -2\langle \mathbf{g}_{i,r}^{(t)} - \nabla F_{i}(w^{(t)}), w^{(t)} - w^{*} \rangle] - 2\tilde{\alpha}_{m} \mathbf{E}_{i}\langle \nabla F_{i}(w^{(t)}), w^{(t)} - w^{*} \rangle \\
\end{aligned}
$$

Each of the two factors of the second term is bounded (note that $\mathbf{E}_{i}=\frac{1}{N}\sum_{i=1}^{N}$ is discussed):

$-\mathbf{E}_{i}\langle \nabla F_{i}(w^{(t)}), w^{(t)} - w^{*} \rangle \le -\mathbf{E}\mathcal{D}_{F}(w^{(t)},w^{*}) - \mathbf{E}\frac{\hat{\mu}_{F_{\cdot}}}{2}||w^{(t)}-w^{*}||^{2}$

$-2\langle \mathbf{g}_{i,r}^{(t)} - \nabla F_{i}(w^{(t)}), w^{(t)} - w^{*} \rangle \le \frac{2}{\hat{\mu}_{F_{\cdot}}}||\mathbf{g}_{i,r}^{(t)}-\nabla F_{i}(w^{(t)})||+\frac{\hat{\mu}_{F_{\cdot}}}{2}||w^{(t)}-w^{*}||^{2}$
where the first inequality is by Proposition~\ref{appdx_prop_sconv} and the second one is by Proposition~\ref{appdx_prop_ppi}.

The third term:
$$
\begin{aligned}
\mathbf{E}_{\cdot|t}||\mathbf{g}^{(t)}||^{2} &= \mathbf{E}_{\cdot|t}||\frac{1}{SR}\sum_{i,r}^{\mathcal{S}^{(t)},R}\mathbf{g}_{i,r}^{(t)}||^{2} \le 3\mathbf{E}_{\cdot|t}[||\frac{1}{SR}\sum_{i,r}^{\mathcal{S}^{(t)},R}\mathbf{g}_{i,r}^{(t)}-\nabla F_{i}(w^{(t)})||^{2} \\
& +||\frac{1}{S}\sum_{i\in \mathcal{S}^{(t)}}\nabla F_{i}(w^{(t)})-\nabla F(w^{(t)})||^{2} + ||\nabla F(w^{(t)})||^{2} ] \\
&\le 3[\frac{1}{NR}\sum_{i,r}^{N,R}||\mathbf{g}_{i,r}^{(t)}-\nabla F_{i}(w^{(t)})||^{2} \\
& +\mathbf{E}_{\cdot|t}||\frac{1}{S}\sum_{i\in \mathcal{S}^{(t)}}\nabla F_{i}(w^{(t)})-\nabla F(w^{(t)})||^{2} + 2\hat{L}_{F}\mathbf{E}\mathcal{D}_{F}(w^{(t)},w^{*}) ]
\end{aligned}
$$
where the first inequality is by Proposition~\ref{appdx_prop_jensen} the second one is by $\nabla F(w^{(t)})=\nabla F(w^{(t)}) - \nabla F(w^{*})$ and Lemma~\ref{appdx_lemma_esplo}.

Thus, if we combine each term back into the separation at the very beginning of this proof, the lemma is proven.
\end{proof}

\begin{lemma}[Local-Global Client Drift Bound]
\label{appdx_lemma_lgcdb}
Under settings and assumptions in Section~\ref{sec_alg} and Section~\ref{appdx_ssec_GAss}, by choosing a proper $\tilde{\alpha}_{m}\le \frac{\beta}{\sqrt{2\dot{c}}}$, the client drift bound is:
$$
\begin{aligned}
    \frac{1}{NR}\sum_{i,r}^{N,R}\mathbf{E}_{\cdot|t,i} ||\mathbf{g}_{i,r}^{(t)}-\nabla F_{i}(w^{(t)})||^{2} \le \dot{\delta} + e\dot{c}\alpha_{m}^{2}2^{R+1}\{(1+2R)\mathbf{E}_{\cdot|t,i}[||\nabla F_{i}(w^{(t)})||^{2}] +\dot{\delta}\}
\end{aligned}
$$
where $\dot{\delta} = 4[\lambda\frac{\hat{L}_{g^{*}}}{\hat{\mu}_{F_{\cdot}}}(\hat{u}_{m} + \eta\hat{\gamma}_{\Phi})]^{2}(\frac{\hat{\gamma}_{f}^{2}}{|\tilde{d}_{i}|}+\hat{\epsilon}^2) + 16[(1+\sigma_{\Phi})\hat{L}_{\mathcal{E}}(\hat{u}_{m}+\eta\hat{\gamma}_{\Phi})\hat{\gamma}_{\Phi}]^{2}$ and $\dot{c}=4[(1+\sigma_{\Phi})\hat{L}_{\mathcal{E}}(\hat{u}_{m}+\eta\hat{\gamma}_{\Phi})]^{2}$.
\end{lemma}
\begin{proof}
\begin{equation}
\begin{aligned}
&||\mathbf{g}_{i,r}^{(t)}-\nabla F_{i}(w^{(t)})||^{2} \le  2[||\mathbf{g}_{i,r}^{(t)}-\nabla F_{i}(w_{i,r}^{(t)})||^{2} + ||\nabla F_{i}(w_{i,r}^{(t)})-\nabla F_{i}(w^{(t)})||^{2}] \\
\le & 2 \{[\lambda\hat{L}_{g^{*}}(\hat{u}_{m} + \eta\hat{\gamma}_{\Phi})]^{2}|| \Delta_{i,r}^{(t)}||^{2} \\ 
& + (1+\sigma_{\Phi})^{2}(\hat{u}_{m}+\eta\hat{\gamma}_{\Phi})^{2}||\nabla \mathcal{E}_{i}(\mu_{i}(w_{i,r}^{(t)})) - \nabla \mathcal{E}_{i}(\mu_{i}(w^{(t)}))||^{2}]\} \\
\le & 2 \{[\lambda\hat{L}_{g^{*}}(\hat{u}_{m} + \eta\hat{\gamma}_{\Phi})]^{2}|| \Delta_{i,r}^{(t)}||^{2}+ [(1+\sigma_{\Phi})\hat{L}_{\mathcal{E}_{i}}(\hat{u}_{m}\\ 
& +\eta\hat{\gamma}_{\Phi})]^{2}||\mu_{i}(w_{i,r}^{(t)}) -\mu_{i}(w^{(t)})||^{2}]\} \\
\le & 2 \{[\lambda\hat{L}_{g^{*}}(\hat{u}_{m} + \eta\hat{\gamma}_{\Phi})]^{2}|| \Delta_{i,r}^{(t)}||^{2}+ [(1+\sigma_{\Phi})\hat{L}_{\mathcal{E}_{i}}(\hat{u}_{m}\\ 
& +\eta\hat{\gamma}_{\Phi})]^{2}[2||w_{i,r}^{(t)}-w^{(t)}||^{2} + 2||\nabla^{2}\Phi(w_{i,r}^{(t)}) -\nabla^{2}\Phi(w^{(t)})||^{2}]\} \\
\le & 2 \{[\lambda\hat{L}_{g^{*}}(\hat{u}_{m} + \eta\hat{\gamma}_{\Phi})]^{2}|| \Delta_{i,r}^{(t)}||^{2}+ 2[(1+\sigma_{\Phi})\hat{L}_{\mathcal{E}_{i}}(\hat{u}_{m}+\eta\hat{\gamma}_{\Phi})]^{2}||w_{i,r}^{(t)}-w^{(t)}||^{2} \\
& + 8[(1+\sigma_{\Phi})\hat{L}_{\mathcal{E}_{i}}(\hat{u}_{m}+\eta\hat{\gamma}_{\Phi})\hat{\gamma}_{\Phi}]^{2}\} \\
\end{aligned}
\end{equation}
where the first inequality is by Proposition~\ref{appdx_prop_jensen}, the second one is by Lemma~\ref{appdx_lemma_esplo}, the third one is by Assumption~\ref{assumption_se} and Proposition~\ref{appdx_prop_lsmth}, the fourth one is by Proposition~\ref{appdx_prop_jensen} and bringing in Equation (~\ref{equ_gfl_mdlb}) and the final one is by Assumption~\ref{appdx_assumption_rmd_msb}.

With Lemma~\ref{appdx_lemma_lspb}, we have:

$$
\begin{aligned}
\mathbf{E}_{\cdot|t,i}&||\mathbf{g}_{i,r}^{(t)}-\nabla F_{i}(w^{(t)})||^{2} \le 4[\lambda\frac{\hat{L}_{g^{*}}}{\hat{\mu}_{F_{\cdot}}}(\hat{u}_{m} + \eta\hat{\gamma}_{\Phi})]^{2}(\frac{\hat{\gamma}_{f}^{2}}{|\tilde{d}_{i}|}+\hat{\epsilon}^2)\\
&+ 16[(1+\sigma_{\Phi})\hat{L}_{\mathcal{E}_{i}}(\hat{u}_{m}+\eta\hat{\gamma}_{\Phi})\hat{\gamma}_{\Phi}]^{2} + 4[(1+\sigma_{\Phi})\hat{L}_{\mathcal{E}_{i}}(\hat{u}_{m}+\eta\hat{\gamma}_{\Phi})]^{2} \mathbf{E}_{\cdot|t,i} ||w_{i,r}^{(t)}-w^{(t)}||^{2}    
\end{aligned}
$$

For simplification: $\mathbf{E}_{\cdot|t,i} ||\mathbf{g}_{i,r}^{(t)}-\nabla F_{i}(w^{(t)})||^{2} \le \dot{\delta} + \dot{c} \mathbf{E}_{\cdot|t,i}||w_{i,r}^{(t)}-w^{(t)}||^{2}$

The second term:
$$
\begin{aligned}
\mathbf{E}_{\cdot|t,i}||w_{i,r}^{(t)}-w^{(t)}||^{2} = & \mathbf{E}_{\cdot|t,i}[||w_{i,r-1}^{(t)}-w^{(t)}-\alpha_{m} \mathbf{g}_{i,r-1}^{(t)}||^{2}] \\
\le & 2\mathbf{E}_{\cdot|t,i}[||w_{i,r-1}^{(t)}-w^{(t)}-\alpha_{m} \nabla F_{i}(w^{(t)})||^{2} \\
& + \alpha_{m}^{2}||\mathbf{g}_{i,r-1}^{(t)}-\nabla F_{i}(w^{(t)})||^{2})] \\
\le & 2(1+\frac{1}{2R})\mathbf{E}_{\cdot|t,i}[||w_{i,r-1}^{(t)}-w^{(t)}||^{2}] + 2(1+2R)\alpha_{m}^{2}\mathbf{E}_{\cdot|t,i}[||\nabla F_{i}(w^{(t)})||^{2}] \\
& +2\alpha_{m}^{2}[ \dot{\delta} + \dot{c} \mathbf{E}_{\cdot|t,i} ||w_{i,r-1}^{(t)}-w^{(t)}||^{2}] \\
\le & 2(1+\frac{1}{2R}+\alpha_{m}^{2}\dot{c})\mathbf{E}_{\cdot|t,i}[||w_{i,r-1}^{(t)}-w^{(t)}||^{2}] \\
& + 2(1+2R)\alpha_{m}^{2}\mathbf{E}_{\cdot|t,i}[||\nabla F_{i}(w^{(t)})||^{2}] +2\alpha_{m}^{2}\dot{\delta} \\
\le & 2(1+\frac{1}{R})\mathbf{E}_{\cdot|t,i}[||w_{i,r-1}^{(t)}-w^{(t)}||^{2}] + 2(1+2R)\alpha_{m}^{2}\mathbf{E}_{\cdot|t,i}[||\nabla F_{i}(w^{(t)})||^{2}] \\
& +2\alpha_{m}^{2}\dot{\delta} \\
\end{aligned}
$$
where the first inequality is by Proposition~\ref{appdx_prop_jensen}, the second one is by Proposition~\ref{appdx_prop_ppi} and the simplified inequality and the final one is by choose $\tilde{\alpha}_{m}^{2}\le \frac{\beta^{2}}{2\dot{c}}$, and $\alpha_{m}^{2}\dot{c}\le\frac{1}{2R^{2}}\le\frac{1}{2R}$.

To recursively unroll: 
\begin{equation}
    \begin{aligned}
    \mathbf{E}_{\cdot|t,i}&||w_{i,r}^{(t)}-w^{(t)}||^{2} \\
        \le & \{(1+2R)\alpha_{m}^{2}\mathbf{E}_{\cdot|t,i}[||\nabla F_{i}(w^{(t)})||^{2}] +\alpha_{m}^{2}\dot{\delta}\}\sum_{\tilde{r}=0}^{r}2^{\tilde{r}+1}(1+\frac{1}{R})^{\tilde{r}} \\
        \le & \{(1+2R)\alpha_{m}^{2}\mathbf{E}_{\cdot|t,i}[||\nabla F_{i}(w^{(t)})||^{2}] +\alpha_{m}^{2}\dot{\delta}\}\sum_{\tilde{r}=0}^{R-1}2^{\tilde{r}+1}(1+\frac{1}{R})^{\tilde{r}} \\
        \le & \alpha_{m}^{2}e2^{R+1}\{(1+2R)\mathbf{E}_{\cdot|t,i}[||\nabla F_{i}(w^{(t)})||^{2}] +\dot{\delta}\} \\
    \end{aligned}
\end{equation}
Thus, bringing in the recursively unrolled inequality back into the simplified one, the lemma's proven.
\end{proof}

\subsection{Proof of Theorems}
\subsubsection{Proof of Theorem~\ref{theo_gm}}
The proof of Theorem~\ref{theo_gm} is shown as followings:
\begin{proof}
    With Lemma~\ref{appdx_lemma_gib}, we have:
    $$
    \begin{aligned}
    \mathbf{\Delta}^{(t+1)} &:= \mathbf{E}||w^{(t+1)}-w^{*}||^{2} \\
    & \le  (1-\frac{\tilde{\alpha}_{m}\hat{\mu}_{F_{\cdot}}}{2})\mathbf{\Delta}^{(t)} + \frac{3\tilde{\alpha}_{m}^{2}+2\tilde{\alpha}_{m}/\hat{\mu}_{F_{\cdot}}}{NR}\sum_{i,r}^{N,R}||\mathbf{g}_{i,r}^{(t)}-\nabla F_{i}(w^{(t)})||^{2} \\
    & + 3\tilde{\alpha}_{m}^{2}\mathbf{E}_{\cdot|t}||\frac{1}{S}\sum_{i\in \mathcal{S}^{(t)}}\nabla F_{i}(w^{(t)})-\nabla F(w^{(t)})||^{2} + (6\tilde{\alpha}_{m}^{2}\hat{L}_{F} - 2\tilde{\alpha}_{m})\mathbf{E}\mathcal{D}_{F}(w^{(t)},w^{*}) \\
    \end{aligned}
    $$
    With Lemma~\ref{appdx_lemma_locseb}, we have:
    $$
    \begin{aligned}
    \mathbf{\Delta}^{(t+1)} &\le  (1-\frac{\tilde{\alpha}_{m}\hat{\mu}_{F_{\cdot}}}{2})\mathbf{\Delta}^{(t)} + \frac{3\tilde{\alpha}_{m}^{2}+2\tilde{\alpha}_{m}/\hat{\mu}_{F_{\cdot}}}{NR}\sum_{i,r}^{N,R}||\mathbf{g}_{i,r}^{(t)}-\nabla F_{i}(w^{(t)})||^{2} \\
    & + 3\tilde{\alpha}_{m}^{2}  \frac{N/S-1}{N-1} \mathbf{E}_{i}||\nabla F_{i}(w^{(t)})-\nabla F(w^{(t)})||^{2} + (6\tilde{\alpha}_{m}^{2}\hat{L}_{F} - 2\tilde{\alpha}_{m})\mathbf{E}\mathcal{D}_{F}(w^{(t)},w^{*}) \\
    \end{aligned}
    $$
    With Lemma~\ref{appdx_lemma_vgacsb}, we have:
    $$
    \begin{aligned}
    \mathbf{E}_{i}||\nabla F_{i}(w)-\nabla F(w)||^{2} \le & 4\hat{L}_{F}\mathcal{D}_{F}(w,w^{*}) + 2{\sigma_{F,*}}^{2}
    \end{aligned}
    $$
    Thus, the inequality is:
    \begin{equation}
    \label{appdx_equ_25}
    \begin{aligned}
    \mathbf{\Delta}^{(t+1)} &\le  (1-\frac{\tilde{\alpha}_{m}\hat{\mu}_{F_{\cdot}}}{2})\mathbf{\Delta}^{(t)} + \frac{3\tilde{\alpha}_{m}^{2}+2\tilde{\alpha}_{m}/\hat{\mu}_{F_{\cdot}}}{NR}\sum_{i,r}^{N,R}||\mathbf{g}_{i,r}^{(t)}-\nabla F_{i}(w^{(t)})||^{2} \\
    & + 3\tilde{\alpha}_{m}^{2}  \frac{N/S-1}{N-1}[4\hat{L}_{F}\mathbf{E}\mathcal{D}_{F}(w^{(t)},w^{*}) + 2{\sigma_{F,*}}^{2}] + (6\tilde{\alpha}_{m}^{2}\hat{L}_{F} - 2\tilde{\alpha}_{m})\mathbf{E}\mathcal{D}_{F}(w^{(t)},w^{*}) \\
    \end{aligned}    
    \end{equation}
    With Lemma~\ref{appdx_lemma_lgcdb} and $\tilde{\alpha}_{m}\le \frac{\beta}{\sqrt{2\dot{c}}}$, by taking full expectation of all variables noted by $\mathbf{E}$, we have:
    $$
    \begin{aligned}
    \frac{1}{NR}\sum_{i,r}^{N,R}\mathbf{E} ||\mathbf{g}_{i,r}^{(t)}-\nabla F_{i}(w^{(t)})||^{2} \le & e\dot{c}\alpha_{m}^{2}2^{R+1}(1+2R)\mathbf{E}_{i}[||\nabla F_{i}(w^{(t)})||^{2}] +(e\dot{c}\alpha_{m}^{2}2^{R+1}+1)\dot{\delta} \\
    \le & e\dot{c}\alpha_{m}^{2}2^{R+1}(1+2R)\mathbf{E}_{i}[2||\nabla F_{i}(w^{(t)})-\nabla F_{i}(w^{*}))||^{2} \\
    & + 2||\nabla F_{i}(w^{*}))||^{2}] +(e\dot{c}\alpha_{m}^{2}2^{R+1}+1)\dot{\delta}  \\
    \le & e\dot{c}\alpha_{m}^{2}2^{R+3}(1+2R)\hat{L}_{F}\mathbf{E}\mathcal{D}_{F}(w^{(t)},w^{*}) \\
    &  + e\dot{c}\alpha_{m}^{2}2^{R+2}(1+2R)\sigma_{F,*}^{2} + (e\dot{c}\alpha_{m}^{2}2^{R+1}+1)\dot{\delta}  \\
    \end{aligned}
    $$
    where the second inequality is by Proposition~\ref{appdx_prop_jensen} and the final one is using Lemma~\ref{appdx_lemma_esplo} and Assumption~\ref{appdx_assumption_oggnb}. With this inequality, Equation (~\ref{appdx_equ_25}) turns into:
    $$
    \begin{aligned}
        \mathbf{\Delta}^{(t+1)}
         \le & (1-\frac{\tilde{\alpha}_{m}\hat{\mu}_{F_{\cdot}}}{2})\mathbf{\Delta}^{(t)} + 6\tilde{\alpha}_{m}^{2}  \frac{N/S-1}{N-1}{\sigma_{F,*}}^{2}\\
        & + (3\tilde{\alpha}_{m}^{2}+2\tilde{\alpha}_{m}/\hat{\mu}_{F_{\cdot}})[e\dot{c}\tilde{\alpha}_{m}^{2}\frac{2^{R+2}(1+2R)}{\beta^{2}R^{2}}\sigma_{F,*}^{2} + (e\dot{c}\tilde{\alpha}_{m}^{2}\frac{2^{R+1}}{\beta^{2}R^{2}}+1)\dot{\delta}] \\
        & + \{
        e\dot{c}(3\tilde{\alpha}_{m}+2/\hat{\mu}_{F_{\cdot}})
        \tilde{\alpha}_{m}^{3}\frac{2^{R+3}(1+2R)}{\beta^{2}R^{2}}\hat{L}_{F} + 12\tilde{\alpha}_{m}^{2}  \frac{N/S-1}{N-1}\hat{L}_{F} \\
        & + 6\tilde{\alpha}_{m}^{2}\hat{L}_{F} - 2\tilde{\alpha}_{m}
        \}\mathbf{E}\mathcal{D}_{F}(w^{(t)},w^{*})\\
    \end{aligned}
    $$

    To simplify this inequality with condition $\tilde{\alpha}_{m}\le\min\{\frac{\beta}{\sqrt{2\dot{c}}},\frac{2}{\hat{\mu}_{F_{\cdot}}}\}$, we have:
    $$
    \begin{aligned}
        \mathbf{\Delta}^{(t+1)} &\le (1-\frac{\tilde{\alpha}_{m}\hat{\mu}_{F_{\cdot}}}{2})\mathbf{\Delta}^{(t)} + 6\tilde{\alpha}_{m}^{2}  \frac{N/S-1}{N-1}{\sigma_{F,*}}^{2}\\
        & + \frac{2^{R+4}e\dot{c}}{\hat{\mu}_{F_{\cdot}}\beta^{2}R^{2}}[2(1+2R)\sigma_{F,*}^{2} + \dot{\delta}]\tilde{\alpha}_{m}^{3} 
        +\frac{8\dot{\delta}}{\hat{\mu}_{F_{\cdot}}}\tilde{\alpha}_{m}
        \\
        & - \{2-\tilde{\alpha}_{m}[
        \frac{e(1+\sigma_{\Phi})\hat{L}_{\mathcal{E}}(\hat{u}_{m}+\eta\hat{\gamma}_{\Phi}) 2^{R+6\frac{1}{2}}(\frac{1}{R}+2)}{\hat{\mu}_{F_{\cdot}}\beta R} \\
        & + 12  \frac{N/S-1}{N-1} + 3]\hat{L}_{F}
        \}\tilde{\alpha}_{m}\mathbf{E}\mathcal{D}_{F}(w^{(t)},w^{*})\\
    \end{aligned}
    $$
    where we use $3\tilde{\alpha}_{m}^{2} \le \frac{6}{\hat{\mu}_{F_{\cdot}}}$ and $\dot{c}\tilde{\alpha}_{m} \le \sqrt{2}\beta (1+\sigma_{\Phi}) \hat{L}_{\mathcal{E}}(\hat{u}_{m}+\eta\hat{\gamma}_{\Phi})$

    Let $\dot{c}_{1} := 2-\tilde{\alpha}_{m}[
        \frac{e(1+\sigma_{\Phi})\hat{L}_{\mathcal{E}}(\hat{u}_{m}+\eta\hat{\gamma}_{\Phi}) 2^{R+6\frac{1}{2}}(\frac{1}{R}+2)}{\hat{\mu}_{F_{\cdot}}\beta R} + 12  \frac{N/S-1}{N-1} + 6]\hat{L}_{F}$, and we have $\dot{c}_{1} \ge 1$, when $\tilde{\alpha}_{m}$ satisfies:
        \begin{equation}
        \label{appdx_equ_c1}
        \begin{aligned}
            \tilde{\alpha}_{m} \le \hat{\alpha}_{m}&:= \frac{\hat{\mu}_{F_{\cdot}}\beta R}{
            e(1+\sigma_{\Phi})\hat{L}_{\mathcal{E}}(\hat{u}_{m}+\eta\hat{\gamma}_{\Phi}) 2^{R+6\frac{1}{2}}(\frac{1}{R}+2) + 18(\hat{\mu}_{F_{\cdot}}\beta R)\hat{L}_{F}} \\
            &\le \frac{1}{[
            \frac{e(1+\sigma_{\Phi})\hat{L}_{\mathcal{E}}(\hat{u}_{m}+\eta\hat{\gamma}_{\Phi}) 2^{R+6\frac{1}{2}}(\frac{1}{R}+2)}{\hat{\mu}_{F_{\cdot}}\beta R} + 12  \frac{N/S-1}{N-1} + 6]\hat{L}_{F}}
        \end{aligned}
        \end{equation}
    By setting $\tilde{\alpha}_{m}$ with Equation (~\ref{appdx_equ_c1}), then let $\xi^{(t)}=(1-\frac{\tilde{\alpha}\hat{\mu}_{F_{\cdot}}}{2})^{-t-1}$ and $\mathcal{X}^{(T)}=\sum_{t=0}^{T-1}\xi^{(t)}$, $\tilde{\alpha}T\ge \frac{2}{\hat{\mu}_{F_{\cdot}}}$, $\tilde{\alpha}_{m}\le\min\{\frac{\beta}{\sqrt{2\dot{c}}},\frac{2}{\hat{\mu}_{F_{\cdot}}}\}$, we have:    
    $$
    \mathbf{\Delta}^{(t+1)}
    \le  (1-\frac{\tilde{\alpha}_{m}\hat{\mu}_{F_{\cdot}}}{2})\mathbf{\Delta}^{(t)} - \tilde{\alpha}_{m}\mathbf{E}\mathcal{D}_{F}(w^{(t)},w^{*}) + \sum_{j=1}^{3}\dot{\delta}_{j}\tilde{\alpha}_{m}^{j}
    $$
    where $\dot{\delta}_{1}:= \frac{8\dot{\delta}}{\hat{\mu}_{F_{\cdot}}}$, $\dot{\delta}_{2}:= 6\frac{N/S-1}{N-1}{\sigma_{F,*}}^{2}$ and $\dot{\delta}_{3}:= \frac{2^{R+4}e\dot{c}}{\hat{\mu}_{F_{\cdot}}\beta^{2}R^{2}}[2(1+2R)\sigma_{F,*}^{2} + \dot{\delta}]$.

    Reformulate it as following:
    $$
    \mathbf{E}\mathcal{D}_{F}(w^{(t)},w^{*})
    \le  \frac{1}{\tilde{\alpha}_{m}}[(1-\frac{\tilde{\alpha}_{m}\hat{\mu}_{F_{\cdot}}}{2})\mathbf{\Delta}^{(t)} - \mathbf{\Delta}^{(t+1)}] + \sum_{j=1}^{3}\dot{\delta}_{j}\tilde{\alpha}_{m}^{j-1} \\
    $$

    Multiply both sides with $\xi^{(t)}$ and accumulate over $t$:
    $$
    \begin{aligned}
    \mathbf{E}\mathcal{D}_{F}(\frac{\sum_{t=0}^{T-1}\xi^{(t)}w^{(t)}}{\mathcal{X}^{(T)}},w^{*}) \le & \frac{\sum_{t=0}^{T-1}\xi^{(t)}}{\mathcal{X}^{(T)}}\mathbf{E}\mathcal{D}_{F}(w^{(t)},w^{*}) \\
    \le &  \frac{1}{\tilde{\alpha}_{m} \mathcal{X}^{(T)}}\sum_{t=0}^{T-1}[(1-\frac{\tilde{\alpha}_{m}\hat{\mu}_{F_{\cdot}}}{2})\xi^{(t)}\mathbf{\Delta}^{(t)} - \xi^{(t)}\mathbf{\Delta}^{(t+1)}] + \sum_{j=1}^{3}\dot{\delta}_{j}\tilde{\alpha}_{m}^{j-1} \\
    = & \frac{1}{\tilde{\alpha}_{m}\mathcal{X}^{(T)}}\mathbf{\Delta}^{(0)}-\frac{\xi^{(T-1)}}{\tilde{\alpha}_{m}\mathcal{X}^{(T)}}\mathbf{\Delta}^{(T)}+\sum_{j=1}^{3}\dot{\delta}_{j}\tilde{\alpha}_{m}^{j-1} \\
    = & \frac{\hat{\mu}_{F_{\cdot}}}{2\xi^{(T-1)}[1-(1-\tilde{\alpha}_{m}\hat{\mu}_{F_{\cdot}}/2)^{T}]}\mathbf{\Delta}^{(0)}-\frac{\xi^{(T-1)}}{\tilde{\alpha}_{m}\mathcal{X}^{(T)}}\mathbf{\Delta}^{(T)}+\sum_{j=1}^{3}\dot{\delta}_{j}\tilde{\alpha}_{m}^{j-1} \\
    \le & \hat{\mu}_{F_{\cdot}}e^{-\tilde{\alpha}_{m}\hat{\mu}_{F_{\cdot}}T/2} \mathbf{\Delta}^{(0)}-\frac{\hat{\mu}_{F_{\cdot}}}{2}\mathbf{\Delta}^{(T)}+\sum_{j=1}^{3}\dot{\delta}_{j}\tilde{\alpha}_{m}^{j-1} \\
    \le & \hat{\mu}_{F_{\cdot}}e^{-\tilde{\alpha}_{m}\hat{\mu}_{F_{\cdot}}T/2} \mathbf{\Delta}^{(0)}+\sum_{j=1}^{3}\dot{\delta}_{j}\tilde{\alpha}_{m}^{j-1} \\
    \le & \mathcal{O}[\mathcal{D}_{F}(\bar{w}^{(T)},w^{*})] \\
    \end{aligned}
    $$
    where $\bar{w}^{(T)} := \frac{\sum_{t=0}^{T-1}\xi^{(t)}}{\mathcal{X}^{(T)}}w^{(t)}$, we use convexity of $\mathcal{D}_{F}$ and $F$ for the first inequality, the second one is by the reformulated inequality and the third one is by setting $\tilde{\alpha}_{m}T \ge \frac{2}{\hat{\mu}_{{F}_{\cdot}}}$ and the fact $\frac{2\xi^{(T-1)}}{\tilde{\alpha}_{m}\hat{\mu}_{F_{\cdot}}} \ge \mathcal{X}^{(T)}=\frac{2\xi^{(T-1)}[1-(1-\frac{\tilde{\alpha}_{m}\hat{\mu}_{F_{\cdot}}}{2})^{T}]}{\tilde{\alpha}_{m}\hat{\mu}_{F_{\cdot}}} \ge \frac{\xi^{(T-1)}}{\tilde{\alpha}_{m}\hat{\mu}_{F_{\cdot}}}$ and $0 \le (1-\frac{\tilde{\alpha}_{m}\hat{\mu}_{F_{\cdot}}}{2})^{T}\le e^{-\frac{1}{2}\tilde{\alpha}_{m}\hat{\mu}_{F_{\cdot}}T}\le e^{-1}\le \frac{1}{2}$.

    To tighten this bound, we recommend~\cite{pmlr-v117-arjevani20a}, which discusses the range and strategy of step sizes in detail rather than our unified bound.

    With $\tilde{\alpha}_{m}\ge \frac{2}{\hat{\mu}_{F_{\cdot}}T}$, we have:
    $$
    \sum_{j=1}^{3}\dot{\delta}_{j}\tilde{\alpha}_{m}^{j-1} \le \mathcal{O}(\dot{\delta}_{1}) + \mathcal{O}(\frac{\dot{\delta}_{2}}{T\hat{\mu}_{F_{\cdot}}}) + \mathcal{O}(\frac{\dot{\delta}_{3}}{T^{2}\hat{\mu}_{F_{\cdot}}^{2}})
    $$
    Thus, $$
    \begin{aligned}
        \mathcal{O}[\mathcal{D}_{F}(\bar{w}^{(T)},w^{*})] =& \mathcal{O}(\hat{\mu}_{F_{\cdot}}e^{-\tilde{\alpha}_{m}\hat{\mu}_{F_{\cdot}}T/2} \mathbf{\Delta}^{(0)})+\mathcal{O}(\frac{\dot{\delta}}{\hat{\mu}_{F_{\cdot}}}) \\
        &+ \mathcal{O}(\frac{(N/S-1){\sigma_{F,*}}^{2}}{N T\hat{\mu}_{F_{\cdot}}}) + \mathcal{O}(\frac{2^{R+4}e\dot{c}}{T^{2}\hat{\mu}_{F_{\cdot}}^{3}\beta^{2}R^{2}}[2(1+2R)\sigma_{F,*}^{2} + \dot{\delta}])
    \end{aligned}
    $$
    where, $\dot{\delta} = 4[\lambda\frac{\hat{L}_{g^{*}}}{\hat{\mu}_{F_{\cdot}}}(\hat{u}_{m} + \eta\hat{\gamma}_{\Phi})]^{2}(\frac{\hat{\gamma}_{f}^{2}}{|\tilde{d}_{i}|}+\hat{\epsilon}^2) + 16[(1+\sigma_{\Phi})\hat{L}_{\mathcal{E}}(\hat{u}_{m}+\eta\hat{\gamma}_{\Phi})\hat{\gamma}_{\Phi}]^{2}$ and $\dot{c}=4[(1+\sigma_{\Phi})\hat{L}_{\mathcal{E}}(\hat{u}_{m}+\eta\hat{\gamma}_{\Phi})]^{2}$.

    For simplification, letting $A = [\frac{\hat{L}_{g^{*}}}{\hat{\mu}_{F_{\cdot}}}(\hat{u}_{m} + \eta\hat{\gamma}_{\Phi})]^{2}(\frac{\hat{\gamma}_{f}^{2}}{|\tilde{d}_{i}|}+\hat{\epsilon}^2)$, $B = [(1+\sigma_{\Phi})\hat{L}_{\mathcal{E}}(\hat{u}_{m}+\eta\hat{\gamma}_{\Phi})\hat{\gamma}_{\Phi}]^{2}$ and $C = \frac{\sigma_{\Phi}^{2}\hat{L}_{\mathcal{E}}^{2}(\hat{u}_{m}+\eta\hat{\gamma}_{\Phi})^{2}}{\hat{\mu}_{F_{\cdot}}^{3}}$, we have:
    $$
    \begin{aligned}
        \mathcal{O}[\mathcal{D}_{F}(\bar{w}^{(T)},w^{*})] = & \mathcal{O}(\hat{\mu}_{F_{\cdot}}e^{-\tilde{\alpha}_{m}\hat{\mu}_{F_{\cdot}}T/2} \mathbf{\Delta}^{(0)})+\mathcal{O}(\frac{A\lambda^{2}+B}{\hat{\mu}_{F_{\cdot}}}) \\
        &+ \mathcal{O}(\frac{(N/S-1){\sigma_{F,*}}^{2}}{N T\hat{\mu}_{F_{\cdot}}}) + \mathcal{O}(\frac{2^{R}C}{T^{2}\beta^{2}R^{2}}[R\sigma_{F,*}^{2} + A\lambda^{2}+B]).
    \end{aligned}
    $$
    
\end{proof}

\subsubsection{Proof of Theorem~\ref{theo_pg}}
The proof of Theorem~\ref{theo_pg} is shown as followings:
\begin{proof}
    With Gaussian prior and first-order methods, we have the bound between personalized model and optimal global model,
    with $\dot{\delta}_{p}= \frac{2}{\hat{\mu}_{F_{i,\cdot}}^{2}}(\frac{\hat{\gamma}_{f}^{2}}{|\tilde{d}_{i}|}+\hat{\epsilon}^2)+\frac{2}{\lambda^{2}}\epsilon_{1}^{2} + \frac{4}{\lambda^{2}}\sigma_{F,*}^{2} +\frac{1}{2}\eta^{2}\mathcal{G}_{\Phi}^{2}$, and $\dot{c}_{p}=(\frac{32}{\lambda^{2}}\hat{L}_{F} + \frac{8}{\hat{\mu}_{F_{\cdot}}})$:
    
    $$
    \begin{aligned}
    \mathbf{E}||\tilde{\theta}_{i}(\bar{w}^{T})-w^{*}||^{2} \le & 4[\mathbf{E}||\tilde{\theta}_{i}(\bar{w}^{T})-\theta^{*}_{i}(\bar{w}^{T})||^{2} \\
    & + \mathbf{E}||\theta_{i}^{*}(\bar{w}^{T})-\mu_{i}(\bar{w}^{T})||^{2} + \mathbf{E}||\mu_{i}(\bar{w}^{T})-w^{(T)}||^{2}+\mathbf{E}||w^{(T)}-w^{*}||^{2}]  \\
    \le & 4[ \frac{2}{\hat{\mu}_{F_{i,\cdot}}^{2}}(\frac{\hat{\gamma}_{f}^{2}}{|\tilde{d}_{i}|}+\hat{\epsilon}^2) +\frac{1}{\lambda^{2}}\mathbf{E}2[||\nabla \tilde{F}_{i}(w^{(T)})-\nabla F_{i}(w^{(T)})||^{2} \\
    & + ||\nabla F_{i}(w^{(T)})||^{2}] +\frac{1}{2}\mathbf{E}||\eta\nabla \Phi_{i}(w^{(T)})||^{2} +\mathbf{E}||w^{(T)}-w^{*}||^{2}]\\
    \le & 4[ \frac{2}{\hat{\mu}_{F_{i,\cdot}}^{2}}(\frac{\hat{\gamma}_{f}^{2}}{|\tilde{d}_{i}|}+\hat{\epsilon}^2) +\frac{2}{\lambda^{2}}\mathbf{E}\{\epsilon_{1}^{2} + 2[||\nabla F_{i}(w^{(T)}) - \nabla F_{i}(w^{*})||^{2} + ||\nabla F_{i}(w^{*})||^{2}] \} \\
    & +\frac{1}{2}\eta^{2}\mathcal{G}_{\Phi}^{2} +\mathbf{E}||w^{(T)}-w^{*}||^{2} ]\\
    \le & 4[ \frac{2}{\hat{\mu}_{F_{i,\cdot}}^{2}}(\frac{\hat{\gamma}_{f}^{2}}{|\tilde{d}_{i}|}+\hat{\epsilon}^2)+\frac{2}{\lambda^{2}}\epsilon_{1}^{2} +\frac{8}{\lambda^{2}}\hat{L}_{F}\mathcal{D}_{F}(w^{(T)},w^{*}) + \frac{4}{\lambda^{2}}\sigma_{F,*}^{2} \\
    & +\frac{1}{2}\eta^{2}\mathcal{G}_{\Phi}^{2} +\mathbf{E}||w^{(T)}-w^{*}||^{2} ]\\
    \le & 4[ \frac{2}{\hat{\mu}_{F_{i,\cdot}}^{2}}(\frac{\hat{\gamma}_{f}^{2}}{|\tilde{d}_{i}|}+\hat{\epsilon}^2)+\frac{2}{\lambda^{2}}\epsilon_{1}^{2} + \frac{4}{\lambda^{2}}\sigma_{F,*}^{2} +\frac{1}{2}\eta^{2}\mathcal{G}_{\Phi}^{2} +(\frac{8}{\lambda^{2}}\hat{L}_{F} + \frac{2}{\hat{\mu}_{F_{\cdot}}})\mathcal{D}_{F}(w^{(T)},w^{*})] \\
    \le & \mathcal{O}(\dot{\delta}_{p}) + \mathcal{O}[\dot{c}_{p}\mathcal{D}_{F}(\bar{w}^{(T)},w^{*})] \\
    \end{aligned}
    $$
    where the first inequality is by Proposition~\ref{appdx_prop_jensen}, the second one is by Lemma~\ref{appdx_lemma_lspb} and Proposition~\ref{appdx_prop_jensen}, the third one is by Assumption~\ref{appdx_assumption_fo} and Lemma~\ref{appdx_assumption_rmd_msb}, the fourth one is by Lemma~\ref{appdx_lemma_esplo} and Assumption~\ref{appdx_assumption_oggnb} and the final one is by Theorem~\ref{theo_gm}.
\end{proof}

\end{document}